\documentclass{article}

\usepackage{microtype}
\usepackage{graphicx}
\usepackage{subfig}
\usepackage{hyperref}

\usepackage{subcaption}
\usepackage{caption}
\usepackage{tikz}

\usetikzlibrary{shapes.geometric, arrows}

\tikzset{
    latent/.style={circle, draw=white, thick, minimum size=1cm, fill={rgb,255:red,235; green,243; blue,251}},
    observed/.style={circle, draw=black, thick, minimum size=1cm},
    data/.style={circle, draw=white, thick, minimum size=1cm, fill={rgb,255:red,225; green,225; blue,225}},
    dashed_arrow/.style={-stealth, dashed, thick},
    solid_arrow/.style={-stealth, thick},
}
\usepackage{booktabs} 

\newcommand{\norm}[1]{\left\lVert #1 \right\rVert_2}

\usepackage[accepted]{style/icml2025}

\usepackage{hyperref}
\usepackage{amsmath}
\usepackage{amssymb}
\usepackage{bbm}
\usepackage{mathtools}
\usepackage{amsthm}
\usepackage{amsfonts}
\usepackage{amstext}
\usepackage{bm}
\usepackage{multirow}
\usepackage{relsize}
\usepackage{makecell}
\usepackage[flushleft]{threeparttable}
\usepackage{pdflscape}

\usepackage[capitalize,noabbrev]{cleveref}

\theoremstyle{plain}
\newtheorem{theorem}{Theorem}[section]
\newtheorem{proposition}[theorem]{Proposition}
\newtheorem{lemma}[theorem]{Lemma}
\newtheorem{corollary}[theorem]{Corollary}
\theoremstyle{definition}
\newtheorem{definition}[theorem]{Definition}
\newtheorem{assumption}[theorem]{Assumption}
\theoremstyle{remark}
\newtheorem{remark}[theorem]{Remark}

\usepackage[textsize=tiny]{todonotes}

\begin{document}

\twocolumn[
\icmltitle{Diffusion Models for Inverse Problems in the Exponential Family}

\icmlsetsymbol{equal}{*}

\begin{icmlauthorlist}
\icmlauthor{Alessandro Micheli}{equal,yyy}
\icmlauthor{Mélodie Monod}{equal,yyy}
\icmlauthor{Samir Bhatt}{yyy,comp}
\end{icmlauthorlist}

\icmlaffiliation{yyy}{Imperial College London}
\icmlaffiliation{comp}{University of Copenhagen}

\icmlcorrespondingauthor{Alessandro Micheli}{a.micheli19@imperial.ac.uk}
\icmlcorrespondingauthor{Mélodie Monod}{melodie.monod18@imperial.ac.uk}

\icmlkeywords{Machine Learning, ICML}

\vskip 0.3in
]

\printAffiliationsAndNotice{\icmlEqualContribution} 

\begin{abstract}
    Diffusion models have emerged as powerful tools for solving inverse problems, yet prior work has primarily focused on observations with Gaussian measurement noise, restricting their use in real-world scenarios.
This limitation persists due to the intractability of the likelihood score, which until now has only been approximated in the simpler case of Gaussian likelihoods.
In this work, we extend diffusion models to handle inverse problems where the observations follow a distribution from the exponential family, such as a Poisson or a Binomial distribution. 
By leveraging the conjugacy properties of exponential family distributions, we introduce the \textit{evidence trick}, a method that provides a tractable approximation to the likelihood score.
In our experiments, we demonstrate that our methodology effectively performs Bayesian inference on spatially inhomogeneous Poisson processes with intensities as intricate as ImageNet images. Furthermore, we demonstrate the real-world impact of our methodology by showing that it performs competitively with the current state-of-the-art in predicting malaria prevalence estimates in Sub-Saharan Africa.


\end{abstract}





\section{Introduction} \label{sec:introduction}

\begin{figure*}[ht!]
\centering
\includegraphics[width=\textwidth]{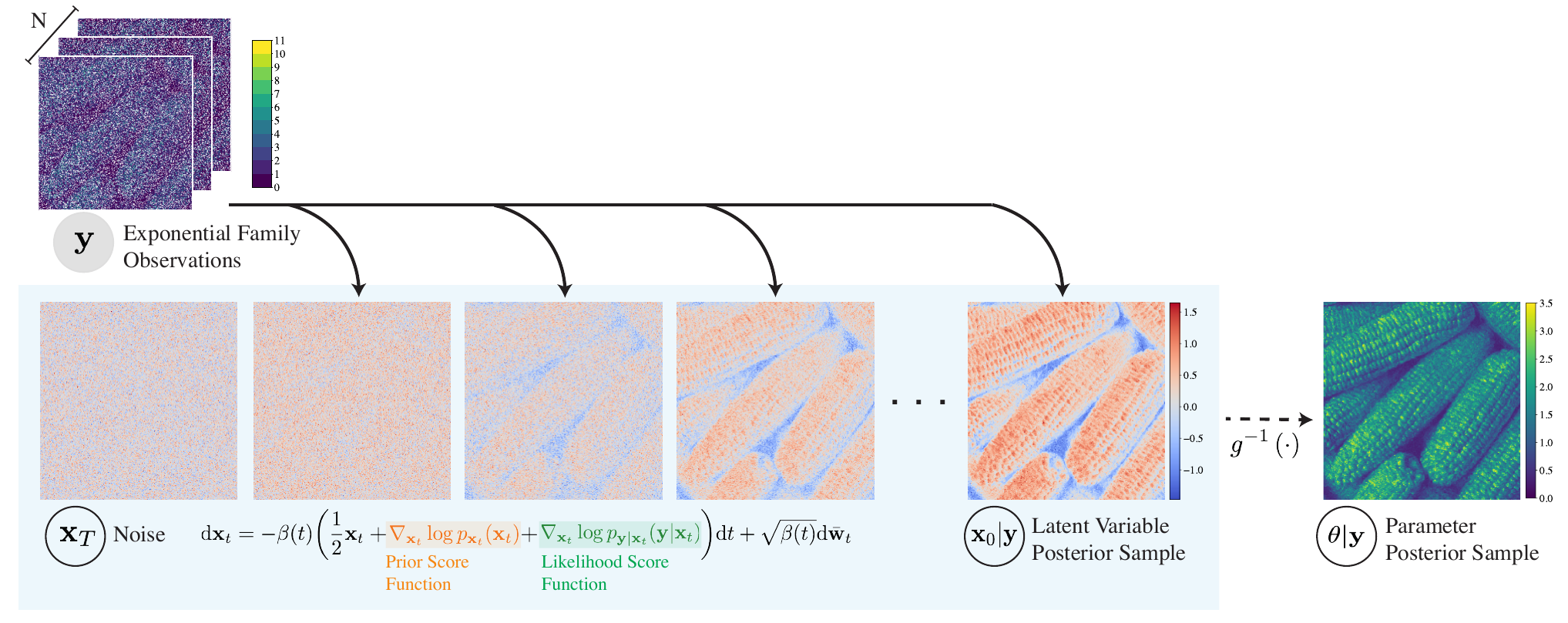}
\caption{\textbf{Illustration of the approach using Diffusion Models for Inverse Problems in the Exponential Family.} By leveraging the posterior score $\nabla_{\mathbf{x}_t} p_{\mathbf{x}_t|\mathbf{y}}(\mathbf{x}_t|\mathbf{y})$, a reverse stochastic differential equation (SDE) can be solved to generate posterior samples of the latent variable $\mathbf{x}_0$ from noise. Posterior samples of the parameter $\boldsymbol{\theta}$ are obtained by applying a deterministic inverse link function. The prior score function, $\nabla_{\mathbf{x}_t} p_{\mathbf{x}_t}(\mathbf{x}_t)$, is estimated using a neural network, following established approaches. A novel method is introduced to estimate the likelihood score function, $\nabla_{\mathbf{x}_t} p_{\mathbf{y}|\mathbf{x}_t}(\mathbf{y}|\mathbf{x}_t)$, leveraging the \textit{evidence trick} in combination with amortized variational inference.  The Figure illustrates the inference of a spatially inhomogeneous Poisson process where the intensity is as intricate as an ImageNet image.}
\label{fig-intro-summary-figure}
\end{figure*}

Score-based diffusion models offer a powerful framework for generating new samples from complex data distributions through a two-step process~\citep{song2021scorebased}. First, the score of the data distribution is estimated by learning to denoise corrupted samples. Second, leveraging this learned score, the noisy inputs are iteratively refined to produce new samples that align with the data distribution.

This capability is particularly advantageous in solving inverse problems. Given access to noisy observations $\mathbf{y} \in \mathbb{R}^{d_y}$, we are interested in inferring a latent signal $\mathbf{x}_0 \in \mathbb{R}^{d_x}$ that generated the observations by sampling from the posterior distribution $p_{\mathbf{x}_0|\mathbf{y}}(\mathbf{x}_0|\mathbf{y})$.
Diffusion models do not require the prior distribution $p_{\mathbf{x}_0}(\mathbf{x}_0)$ to be analytically specified or explicitly parameterized, they rely only on the ability to sample from it.
Therefore, they can be trained to learn the score of highly complex prior distributions, for instance where $\mathbf{x}_0$ are images from ImageNet. 
This distinguishes them from traditional methods like Markov Chain Monte Carlo (MCMC), which necessitates evaluating the prior density.

Despite these advantages, current diffusion-based methodologies are predominantly confined to Gaussian likelihoods (e.g.,~\citet{kadkhodaie2021, kawar2021, kawar2022, chung2023, song2023pseudoinverseguided, boys2024, rozet2024}), limiting their applicability to scenarios such as image deblurring or denoising. 
However, many scientific applications involve likelihoods that deviate significantly from Gaussian distributions. For instance, event data (e.g., COVID-19 case counts) is naturally modeled using a Poisson distribution, while proportion data (e.g., prevalence rates) aligns with a Binomial distribution. 

One major obstacle to extending diffusion models to non-Gaussian likelihoods lies in the intractability of the posterior distribution score at diffusion time $t$, $ p_{\mathbf{x}_t | \mathbf{y}}(\mathbf{x}_t | \mathbf{y})$. Specifically, the posterior distribution score incorporates both the the prior score at time $t$, $ p_{\mathbf{x}_t}(\mathbf{x}_t)$, which can be trained, and the likelihood score at time $t$, $p_{\mathbf{y} | \mathbf{x}_t }(\mathbf{y}| \mathbf{x}_t )$. 
The latter is generally intractable because it depends on an integral over the reverse diffusion process density, which is itself difficult to compute. 
A common approach, when the observations follow a Gaussian distribution, is to approximate the reverse diffusion process and derive a closed-form expression for the likelihood score.
This approach is used by methods like Diffusion Posterior Sampling (DPS)~\citep{chung2023}, which proposes a delta function approximation centered at the Tweedie's posterior first moment, and Tweedie Moment Projected Diffusions~\citep{boys2024}, which employed a Multivariate anisotropic Gaussian distribution approximation.

These methods have significant limitations. They often struggle to accurately quantify uncertainty in the reverse diffusion process and they rely on Tweedie’s formula~\citep{Efron2011}, which exhibits high variance at high noise levels (see Section 1.2 of~\citet{target_score_matching}). 
Additionally, they cannot accommodate non-Gaussian observations. While~\citet{chung2023} proposed using Gaussian approximations for non-Gaussian distributions --- such as approximating a Poisson distribution with a Gaussian --- this approach is highly unstable for small values and entirely inapplicable to certain distributions.
These limitations underscore the need for robust methodologies that extend the applicability of diffusion models to non-Gaussian settings, ensuring both stability and accurate representation of diverse likelihoods.

To address these limitations, we introduce an approach that we term, the \textit{evidence trick}. By leveraging the properties of the exponential family and employing an amortized variational approach, we extend the applicability of diffusion models for inverse problems to any likelihood distribution within the one-parameter exponential family, which includes the Poisson and Binomial distributions. 
To approximate the likelihood score, $p_{\mathbf{y}|\mathbf{x}_t}(\mathbf{y}|\mathbf{x}_t)$, we use the conjugate prior distribution as a variational approximation for the reverse diffusion process. This reformulation makes the integral over the reverse process tractable and accounts for parameter uncertainty. In contrast to previous approaches, we derive an objective to optimize the variational distribution that is independent of the reverse process expectation, bypassing the need for Tweedie's formula. Figure~\ref{fig-intro-summary-figure} graphically summarises our methodology.

We leverage our methodology to introduce a \textit{``Score-Based Cox process''}, a discrete Cox process where the intensity is modeled using a score-based diffusion process. We demonstrate that our model can effectively capture rough and intricate intensity patterns, including those as complex as image samples from the ImageNet database. Furthermore, we show that our methodology can address real-world scientific challenges by performing competitively with the current state-of-the-art in predicting malaria prevalence estimates in Sub-Saharan Africa.

\section{Background} \label{sec:background}
Throughout this work, vectors and matrices will be represented using boldface notation, denoted as $\mathbf{x}$, while scalars will be expressed in standard font as $x$. Furthermore, we will use $\norm{\mathbf{x}}$ to denote the $\ell^{2}$-norm of a vector $\mathbf{x}$.
\subsection{Score-Based Diffusion Models} 
\label{sec-score-based-diffusions}
Score-based diffusion models aim to generate samples from a target distribution $ p_{\mathbf{x}_0}(\mathbf{x}_0) $ by progressively perturbing data with increasing noise levels and then learning to reverse this perturbation. This reversal defines a generative model capable of approximating the original data distribution. In this work, we adopt the framework introduced by \citet{song2021scorebased}, who define the forward noising process via an Itô stochastic differential equation (SDE). Specifically, we focus on the Variance-Preserving (VP) formulation of the SDE presented by~\citet{song2021scorebased}, which corresponds to the Denoising Diffusion Probabilistic Models (DDPM) introduced by~\citet{ho_denoising}.

We construct a forward diffusion process $ (\mathbf{x}_t)_{t \in [0,T]} $, with $ \mathbf{x}_t \in \mathbb{R}^{d_x} $, governed by the following equation:
\begin{equation*}
\mathrm{d}\mathbf{x}_t = - \frac{1}{2} \beta(t) \mathbf{x}_t \, \mathrm{d}t + \sqrt{\beta(t)} \, \mathrm{d}\mathbf{w}_t, \quad \mathbf{x}_0 \sim p_{\mathbf{x}_0},
\end{equation*}
where $ \mathbf{w}_t $ denotes a standard Wiener process, and $ \beta(t) : \mathbb{R} \to \mathbb{R}_+$ is a noise schedule. A commonly chosen parametrization for the noise schedule is the linear schedule $ \beta(t) = \beta_0 + t \, (\beta_1 - \beta_0) $, as discussed in \citet[Appendix C]{song2021scorebased}. The forward process is associated with the following transition kernel:
\begin{equation}
\label{eq-forward-transition-kernel}
p_{\mathbf{x}_t | \mathbf{x}_0}(\mathbf{x}_t | \mathbf{x}_0) = \mathcal{N}(\mathbf{x}_t; \sqrt{\alpha_t} \mathbf{x}_0, v_t \mathbf{I}_{d_x}),
\end{equation}
where $ \alpha_t := \exp\left( - \int_0^t \beta(s) \, \mathrm{d}s \right) $ and $ v_t := 1 - \alpha_t $. 

To recover the data-generating distribution, we reverse the noising process by solving the reverse SDE, derived from the forward process \cite{anderson_1982}:
\begin{multline*}
\mathrm{d}\mathbf{x}_t = - \beta(t) \left( \frac{1}{2} \mathbf{x}_t + \nabla_{\mathbf{x}_t} \log p_{\mathbf{x}_{t}}(\mathbf{x}_t) \right) \, \mathrm{d}t   \\ + \sqrt{\beta(t)} \, \mathrm{d}\bar{\mathbf{w}}_t, \quad \mathbf{x}_T \sim p_{\mathbf{x}_T},
\end{multline*}
where $\mathrm{d}t$ corresponds to time running backward, $\mathrm{d}\bar{\mathbf{w}}_t$ to the standard Wiener process running backward. Importantly, the term $ \nabla_{\mathbf{x}_t} \log p_{\mathbf{x}_{t}}(\mathbf{x}_t) $ is the score function, which guides the reverse process and it is typically approximated using a neural network $ \mathbf{s}_{\boldsymbol{\phi}}$, with learnable parameters $\boldsymbol{\phi}$, trained via Denoising Score Matching (DSM)~\cite{vincent2011} using the objective
\begin{multline}\label{eq:phi_star}
    \mathcal{J}_{\text{DSM}}(\boldsymbol{\phi}) =\\  \mathbb{E}_{t\sim U(\epsilon, 1)}\Big[ \lambda(t)  \mathbb{E}_{ \mathbf{x}_0\sim  p_{\mathbf{x}_0}, \mathbf{x}_t \sim p_{\mathbf{x}_t\vert\mathbf{x}_0}}\Big[ \mathcal{L}_{\text{DSM}}(\boldsymbol{\phi},  \mathbf{x}_0, \mathbf{x}_t, t)\Big]\Big], 
\end{multline}
with 
\begin{multline*}
\mathcal{L}_{\text{DSM}}(\boldsymbol{\phi}, \mathbf{x}_0, \mathbf{x}_t, t) = \\  \norm{\mathbf{s}_{\boldsymbol{\phi}}(\mathbf{x}_t, t) 
    -  \nabla_{\mathbf{x}_t} \log p_{\mathbf{x}_t\vert\mathbf{x}_0}(\mathbf{x}_t\vert\mathbf{x}_0)}^2,
\end{multline*}
and where $\epsilon \approx 0$ is a small positive constant, $\lambda(t) :[0,T]\to\mathbb{R}_+$ is a positive weighting function typically set to $\lambda(t) = 1/ \mathbb{E}\left[\left\vert\left\vert\nabla_{\mathbf{x}_t} \log p_{\mathbf{x}_t\vert\mathbf{x}_0}(\mathbf{x}_t|\mathbf{x}_0) \right\vert\right\vert_2^2\right]$ (see~\citet[Section 3.3]{song2021scorebased}). Once $\boldsymbol{\phi}^{*}$ is acquired by minimizing~\eqref{eq:phi_star}, one can use the approximation $\nabla_{\mathbf{x}_t} \log p_{\mathbf{x}_t}(\mathbf{x}_t) \simeq \mathbf{s}_{\boldsymbol{\phi}^*}(\mathbf{x}_t, t)$. 

\subsection{Inverse Problems with Diffusion Models}
\label{sec-diffusion-posterior-sampling}

Inverse problems across various scientific domains share a unified mathematical framework. The objective in these problems is to infer unknown parameters $\mathbf{x}_{0}$ given a set of measurements 
$\mathbf{y} \in \mathbb{R}^{d_{y}}$.
To solve such problems in a Bayesian framework, one adopts a prior distribution $p_{\mathbf{x}_0}(\mathbf{x}_0)$ and a likelihood distribution  $p_{\mathbf{y}|\mathbf{x}_0}(\mathbf{y}|\mathbf{x}_0)$, and seeks to sample from the posterior distribution $p_{\mathbf{x}_0|\mathbf{y}}(\mathbf{x}_0|\mathbf{y})$. 
Using Bayes’ rule, the posterior is given by:
\begin{equation*}
p_{\mathbf{x}_0|\mathbf{y}}(\mathbf{x}_0|\mathbf{y}) = \frac{ p_{\mathbf{y}|\mathbf{x}_0}(\mathbf{y}|\mathbf{x}_0)p_{\mathbf{x}_0}(\mathbf{x}_0)}{p_{\mathbf{y}}(\mathbf{y})},
\end{equation*}
where the \textit{evidence} is $p_{\mathbf{y}}(\mathbf{y}) = \int p_{\mathbf{y}|\mathbf{x}_0}(\mathbf{y}|\mathbf{x}_0)p_{\mathbf{x}_0}(\mathbf{x}_0) d\mathbf{x}_0$. The diffusion-based approaches of Section~\ref{sec-score-based-diffusions} can be adapted to sample from the posterior by adopting the following reverse process
\begin{multline}
\label{eq:reverse_SDE_posterior}
\mathrm{d}\mathbf{x}_t = - \beta(t) \left(\frac{1}{2}\mathbf{x}_t+ \nabla_{\mathbf{x}_t}\log p_{\mathbf{x}_{t}\vert \mathbf{y}}(\mathbf{x}_t|\mathbf{y})  
\right)  \mathrm{d}t \\ + \sqrt{\beta(t)} \mathrm{d}\bar{\mathbf{w}}_t, \quad \mathbf{x}_T \sim p_{\mathbf{x}_T|\mathbf{y}}. 
\end{multline}
It follows from Bayes' rule that the score of the posterior is 
\begin{multline} 
\label{eq:score_posterior}
    \nabla_{\mathbf{x}_t }\log p_{\mathbf{x}_t | \mathbf{y}}(\mathbf{x}_t | \mathbf{y}) = \\ \nabla_{\mathbf{x}_t } \log p_{\mathbf{x}_t}(\mathbf{x}_t) + \nabla_{\mathbf{x}_t } \log p_{\mathbf{y}|\mathbf{x}_t}(\mathbf{y}|\mathbf{x}_t).
\end{multline}
Hence, computing the score of the posterior distribution can be reduced to evaluating two terms: the prior score function, $\nabla_{\mathbf{x}_t }\log p_{\mathbf{x}_t}(\mathbf{x}_t)$, and the likelihood score function,$\nabla_{\mathbf{x}_t} \log p_{\mathbf{y}|\mathbf{x}_t}(\mathbf{y}|\mathbf{x}_t)$. The former can be directly obtained using the trained prior score function $\mathbf{s}_{\phi^*}(\mathbf{x}_t, t)$. However, computing the latter is challenging in closed form due to its dependence on time $t$, as there is only an explicit dependence between $\mathbf{y}$ and $\mathbf{x}_0$. To address this, \citet{chung2023} propose to factorize
$p_{\mathbf{y}|\mathbf{x}_t}(\mathbf{y}|\mathbf{x}_t )$ as:
\begin{equation}
\label{eq-likelihood-y-xt}
    p_{\mathbf{y}|\mathbf{x}_t}(\mathbf{y}|\mathbf{x}_t ) 
    = \int p_{\mathbf{y}|\mathbf{x}_0}(\mathbf{y}|\mathbf{x}_0) p_{\mathbf{x}_0|\mathbf{x}_t}(\mathbf{x}_0|\mathbf{x}_t) \mathrm{d}\mathbf{x}_0,
\end{equation} 
which follows from the fact that $\mathbf{y}$ and $\mathbf{x}_{t}$ are conditionally independent given $\mathbf{x}_{0}$. The density $p_{\mathbf{x}_0|\mathbf{x}_t}(\mathbf{x}_0|\mathbf{x}_t)$ is generally intractable, making the approximation of the integral in~\eqref{eq-likelihood-y-xt} a challenging task.

\subsection{Sampling for Linear Inverse Problems}
\label{sec-sampling-linear-inverse-problem}

The existing literature has predominantly focused on applications where observations follow a Gaussian likelihood:
\begin{equation*} 
    \mathbf{y} = \mathcal{H}( \mathbf{x}_0) + \mathbf{u}, \quad \text{where} \quad \mathbf{u} \sim \mathcal{N}(0, \sigma_y^2 \mathbf{I}_{d_y}),
\end{equation*}
where $\mathcal{H}:\mathbb{R}^{d_{x}}\to\mathbb{R}^{d_{y}}$ is the forward measurement operator and $\mathbf{u}$ is the measurement noise.
Practical applications relevant to this work often involve a potentially non-invertible linear setting, where $ \mathcal{H}(\mathbf{x}_0) = \mathbf{H}\mathbf{x}_0 $ for an $ d_{y} \times d_{x} $ real matrix $ \mathbf{H} $ with $ d_{x} \leq d_{y} $. In this context, existing studies~\cite{chung2023,song2023pseudoinverseguided, boys2024,rozet2024} approximate the integral in~\eqref{eq-likelihood-y-xt} by employing a Gaussian approximation, $ q_{\mathbf{x}_0|\mathbf{x}_t}(\mathbf{x}_0|\mathbf{x}_t) $, for the true posterior distribution $ p_{\mathbf{x}_0|\mathbf{x}_t}(\mathbf{x}_0|\mathbf{x}_t) $. This Gaussian approximation is defined as:  
\begin{equation*} 
q_{\mathbf{x}_0|\mathbf{x}_t}(\mathbf{x}_0|\mathbf{x}_t) = \mathcal{N}_{d_x}(\mathbf{x}_0; \mathbf{m}_0(\mathbf{x}_t), C_0(\mathbf{x}_t)),  
\end{equation*}  
where $ \mathbf{m}_0(\mathbf{x}_t) $ and $ C_0(\mathbf{x}_t) $ are the mean and covariance of the approximation, respectively. This approach enables the computation of closed-form expressions for $ p_{\mathbf{y}|\mathbf{x}_t}(\mathbf{y}|\mathbf{x}_t) $, as the integral in Equation~\eqref{eq-likelihood-y-xt} becomes analytically tractable\footnote{We informally interpret the delta function approximation in \citet{chung2023} as a degenerate Gaussian distribution where the variance approaches zero.}. 

Relevant to our work is the approach adopted by \citet{boys2024} who proposed approximating $ p_{\mathbf{x}_0|\mathbf{x}_t}(\mathbf{x}_0|\mathbf{x}_t) $ by projecting it onto the closest Gaussian distribution $ q_{\mathbf{x}_0|\mathbf{x}_t}(\mathbf{x}_0|\mathbf{x}_t) $ in terms of the Kullback-Leibler (KL) divergence. The closest Gaussian in this sense is the one that matches the first two moments, $\mathbb{E}_{p_{\mathbf{x}_0|\mathbf{x}_t}}[\mathbf{x}_0]$ and $\mathbb{E}_{p_{\mathbf{x}_0|\mathbf{x}_t}}[\mathbf{x}_0 \mathbf{x}_0^\top]$,  of the true posterior distribution. They estimated these moments using Tweedie's formula~\cite{Efron2011}.

\section{Sampling for Diffusion Models with Conjugacy Structure} \label{sec:method}

The results presented in this section are derived using the theoretical framework and properties of exponential family distributions, which are thoroughly reviewed in Appendix~\ref{app-exponential-family}. 

\subsection{Setup}
The dataset $\mathbf{y} = \{ \boldsymbol{y}_i \}_{i=1}^N$ is assumed to consist of $N$ independent and identically distributed (i.i.d.) observations. Each observation $\boldsymbol{y}_i \in \mathcal{Y}^d \subseteq \mathbb{R}^d$ is derived from a parameter vector $\boldsymbol{\theta} \in \Theta^d \subseteq \mathbb{R}^d$ through the conditional distribution $p_{\boldsymbol{y} \vert \boldsymbol{\theta}}(\boldsymbol{y}_i \vert \boldsymbol{\theta})$ for $i = 1, \ldots, N$. 
The components of $\boldsymbol{y}_i$ and $\boldsymbol{\theta}$ are denoted by $\boldsymbol{y}_i = (y_{i,1}, y_{i,2}, \ldots, y_{i,d})$ and $\boldsymbol{\theta} = (\theta_1, \theta_2, \ldots, \theta_d)$, respectively. 
Henceforth we will work under the following assumptions. 
\begin{assumption}[Conditional Independence of Variables]
\label{ass-independence-y}
The variable $y_{i,j}|\boldsymbol{\theta}$ is independent of $y_{i,k}|\boldsymbol{\theta}$ for all $j \neq k$ and for all $i = 1, \ldots, N$. Furthermore, we assume that $y_{i,j}| \theta_j$ is independent of $\theta_{k}$ for all $j \neq k$.
\end{assumption}
\begin{assumption}[Exponential Family Distribution]
\label{ass-exponential-distribution}
The distribution $p_{{y} \vert {\theta}}(y_{i,j} \vert {\theta}_j)$ belongs to the univariate one-parameter exponential family with natural parameter $\eta(\theta_j)$, base measure $h_{y}(y_{i,j})$, sufficient statistics $T_{y}(y_{i,j})$ and log-partition function $A_y(\eta(\theta_j))$ for $j = 1, \ldots, d$ and $i = 1, \ldots, N$.
\end{assumption}
Given Assumptions~\ref{ass-independence-y} and~\ref{ass-exponential-distribution}, it follows that the distribution $p_{\boldsymbol{y} \vert \boldsymbol{\theta}}(\boldsymbol{y}_i \vert \boldsymbol{\theta})$ belongs to the multivariate exponential family with the form
\begin{multline*}
p_{\boldsymbol{y}\vert\boldsymbol{\theta}}(\boldsymbol{y}_i \vert \boldsymbol{\theta}) =\\  h_{\boldsymbol{y}}(\boldsymbol{y}_i)   \exp \left( \boldsymbol{\eta}(\boldsymbol{\theta})^{\top} \mathbf{T}_{\boldsymbol{y}}(\boldsymbol{y}_i) - \mathbf{1}_d^\top \mathbf{A}_{\boldsymbol{y}}(\boldsymbol{\eta}(\boldsymbol{\theta})) \right),
\end{multline*}
for $i = 1, \ldots, N$ and where $\mathbf{1}_d$ is a vector of ones of dimension $d$ and
\begin{equation*}
\begin{aligned}
&h_{\boldsymbol{y}}(\boldsymbol{y}_i) = \prod_{j = 1}^d h_{y}\left(y_{i,j}\right), \\ 
&\boldsymbol{\eta}(\boldsymbol{\theta}) = \left(\eta(\theta_1), \dots, \eta(\theta_d)\right), \\
&\mathbf{T}_{\boldsymbol{y}}(\boldsymbol{y}_i) = \left(T_y(y_{i,1}), \ldots, T_y(y_{i,d})\right), \\
&\mathbf{A}_{\boldsymbol{y}}\left(\boldsymbol{\eta}\left(\boldsymbol{\theta}\right)\right) = \left(A_y\left(\eta\left(\theta_1\right)\right), \ldots, A_y\left(\eta\left(\theta_d\right)\right)\right).
\end{aligned}
\end{equation*}
Furthermore, since 
$\mathbf{y}$ consists of $N$ i.i.d observations, then the distribution $p_{\mathbf{y}|\boldsymbol{\theta}}(\mathbf{y}|\boldsymbol{\theta})$ can be written as, 
\begin{multline}
    \label{eq:likelihood_y_phi}
    p_{\mathbf{y}|\boldsymbol{\theta}}(\mathbf{y}|\boldsymbol{\theta}) = \\
    h_{\mathbf{y}}(\mathbf{y}) \exp \left(\boldsymbol{\eta}(\boldsymbol{\theta})^{\top}\mathbf{T}_{\mathbf{y}}(\mathbf{y}) - N  \mathbf{1}_d^\top \mathbf{A}_{\boldsymbol{y}}(\boldsymbol{\eta}(\boldsymbol{\theta})) \right),
\end{multline}
where $h_{\mathbf{y}}(\mathbf{y}) = \prod_{i=1}^N h_{\boldsymbol{y} }(\boldsymbol{y}_i)$  and $\mathbf{T}_{\mathbf{y} }(\mathbf{y}) = \sum_{i = 1}^N \mathbf{T}_{\boldsymbol{y} }(\boldsymbol{y}_i)$.

\subsection{Sampling with a Link Function}
\label{sec-sampling-link-function}
We introduce a deterministic \textit{link function}, denoted as $g(\cdot)$. The following assumption is imposed on the link function:
\begin{assumption}
\label{ass-link-function}
    The link function $g: \Theta \to \mathbb{R}$ is assumed to be continuously differentiable, one-to-one and with ${\mathrm{d}g}/{\mathrm{d}\theta}\neq~0$ for all $\theta\in\Theta$.
\end{assumption}
These properties are standard assumptions and are consistent with those typically used in the context of the change-of-variable technique in probability and statistics (see Theorem 17.2 in \citet{billingsley_prob}).
We write $g(\boldsymbol{\theta})$ to denote the entry-wise application of $g(\cdot)$ to $\boldsymbol{\theta}$. 
The link function maps each parameter $\boldsymbol{\theta}\in \Theta^{d}$ to a transformed variable $\mathbf{x}_{0}= (x_{0,1}, x_{0,2}, \dots, x_{0,d}) \in \mathbb{R}^d$ satisfying the relation 
\begin{equation} \label{eq:transformation_parameter}
    \mathbf{x}_0   = g(\boldsymbol{\theta}).
\end{equation}
Our goal is to generate samples from the posterior distribution $p_{\boldsymbol{\theta} \vert \mathbf{y}}(\boldsymbol{\theta} \vert \mathbf{y})$, or a suitable approximation thereof. Instead of sampling directly from $p_{\boldsymbol{\theta} \vert \mathbf{y}}(\boldsymbol{\theta} \vert \mathbf{y})$, this can be achieved by sampling $\mathbf{x}_0$ from the transformed posterior $p_{\mathbf{x}_0 \vert \mathbf{y}}(\mathbf{x}_0 \vert \mathbf{y})$ using diffusion models as per the methodology described in Section~\ref{sec-diffusion-posterior-sampling} and applying the inverse link function 
\begin{equation*}
    \boldsymbol{\theta} = g^{-1}(\mathbf{x}_0).
\end{equation*}
Figure~\ref{fig-graphical-model} illustrates our approach as a hierarchical probabilistic model.
To streamline the presentation of our results, we defer the discussion in the presence of a linear measurement operator $\mathbf{H} \in \mathbb{R}^{d_y \times d_x}$ to Appendix~\ref{app-observation-operator-H}.
Proposed link functions for mapping the likelihood parameters $\boldsymbol{\theta}$ to the latent variable  $\mathbf{x}_0 $ are provided in Appendix~\ref{app-proposed_link_distributions}.
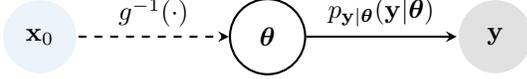
\begin{figure}[t!]
\centering
\begin{tikzpicture}[node distance=3cm and 2cm]

\node[latent] (x0) {$\mathbf{x}_{0}$};
\node[observed, right=of x0] (theta) {$\boldsymbol{\theta}$};
\node[data, right=of theta] (y) {$\mathbf{y}$};

\draw[dashed_arrow] (x0) -- node[midway, above] {$g^{-1}(\cdot)$} (theta);
\draw[solid_arrow] (theta) -- node[midway, above] {$p_{\mathbf{y}\vert\boldsymbol{\theta}}(\mathbf{y} \vert \boldsymbol{\theta})$} (y);

\end{tikzpicture}
\caption{\textbf{Hierarchical Probabilistic Model.} The dotted arrow represents a deterministic relationship, while the solid arrow indicates a probabilistic relationship.}
\label{fig-graphical-model}
\end{figure}

\subsection{The Evidence Trick}
To approximate the likelihood $p_{\mathbf{y} \vert \mathbf{x}_t}(\mathbf{y} \vert \mathbf{x}_t)$ as defined in~\eqref{eq-likelihood-y-xt}, we propose a simple yet effective approach that we call the \textit{evidence trick}. Given the assumption that the likelihood $p_{\mathbf{y}|\boldsymbol{\theta}}(\mathbf{y}|\boldsymbol{\theta})$ belongs to the exponential family, there always exists a natural conjugate prior distribution $q_{\boldsymbol{\theta}|\boldsymbol{\zeta}}(\boldsymbol{\theta}|\boldsymbol{\zeta})$ with hyperparameters $\boldsymbol{\zeta}$ for which the integral
\begin{equation*}
\int p_{\mathbf{y}|\boldsymbol{\theta}}(\mathbf{y}|\boldsymbol{\theta}) q_{\boldsymbol{\theta}|\boldsymbol{\zeta}}(\boldsymbol{\theta}|\boldsymbol{\zeta}) \mathrm{d}\boldsymbol{\theta} 
\end{equation*}
can be computed in closed-form and corresponds to the \textit{evidence} of $p_{\mathbf{y}|\boldsymbol{\theta}}(\mathbf{y}|\boldsymbol{\theta})$.
As shown in Proposition~\ref{prop:expfam_form_independent_parameters}, the natural conjugate prior distribution also belongs to the exponential family and takes the form:
\begin{equation}
\label{eq-prior-q-theta}
q_{\boldsymbol{\theta}|\boldsymbol{\zeta}}(\boldsymbol{\theta}|\boldsymbol{\zeta}) = h_{\boldsymbol{\theta}}(\boldsymbol{\theta}) \exp \left(\boldsymbol{\zeta}^T \mathbf{T}_{\boldsymbol{\theta}}(\boldsymbol{\theta}) - A_{\boldsymbol{\theta}}(\boldsymbol{\nu},  \boldsymbol{\tau}) \right),
\end{equation}
with hyperparameters $\boldsymbol{\zeta} = \left(\boldsymbol{\nu} ,\boldsymbol{\tau}\right)$, $\boldsymbol{\nu}, \boldsymbol{\tau}  \in \mathbb{R}^d$, base measure $h_{\boldsymbol{\theta}}(\boldsymbol{\theta})$, sufficient statistics $\mathbf{T}_{\boldsymbol{\theta}}(\boldsymbol{\theta}) = (\boldsymbol{\eta}(\boldsymbol{\theta}), -\mathbf{A}_{\boldsymbol{y}}(\boldsymbol{\eta}(\boldsymbol{\theta})))$ and log-partition function $A_{\boldsymbol{\theta}}(\boldsymbol{\nu},  \boldsymbol{\tau})$. The specific form of the natural conjugate prior distribution's base measure and log-partition function is provided in Appendix~\ref{app-table_distributions}.

On this basis, we propose approximating $p_{\boldsymbol{\theta}|\mathbf{x}_t}(\boldsymbol{\theta}|\mathbf{x}_t)$ using the variational distribution $q_{\boldsymbol{\theta}|\boldsymbol{\zeta}(\mathbf{x}_t)}(\boldsymbol{\theta}|\boldsymbol{\zeta}(\mathbf{x}_t))$, as expressed below:
\begin{equation*}
p_{\boldsymbol{\theta}|\mathbf{x}_t}(\boldsymbol{\theta}|\mathbf{x}_t) \approx q_{\boldsymbol{\theta}|\boldsymbol{\zeta}(\mathbf{x}_t)}(\boldsymbol{\theta}|\boldsymbol{\zeta}(\mathbf{x}_t)),
\end{equation*}
where the dependence of the hyperparameters $\boldsymbol{\zeta}(\mathbf{x}_t)  = \left(\boldsymbol{\nu}(\mathbf{x}_t) ,\boldsymbol{\tau}(\mathbf{x}_t)\right)$ on the input $\mathbf{x}_t$ is explicitly indicated.
This allows us to treat the density $p_{\mathbf{y}|\mathbf{x}_t}(\mathbf{y}|\mathbf{x}_{t})$ as the \textit{evidence} and approximate it with:
\begin{equation}
\label{eq:likelihood_t}
     p_{\mathbf{y}|\mathbf{x}_t}(\mathbf{y}|\mathbf{x}_{t})
     \approx\int p_{\mathbf{y}|\boldsymbol{\theta}}(\mathbf{y}|\boldsymbol{\theta}) q_{\boldsymbol{\theta}|\boldsymbol{\zeta}(\mathbf{x}_t)}(\boldsymbol{\theta}|\boldsymbol{\zeta}(\mathbf{x}_t)) \mathrm{d}\boldsymbol{\theta}.
\end{equation}
As shown in Proposition~\ref{prop:conjugacy}, the integral in \eqref{eq:likelihood_t} has a closed form expression which is given by 
\begin{multline}
\label{eq:likelihood_t-closed-form}
     p_{\mathbf{y}|\mathbf{x}_t}(\mathbf{y}|\mathbf{x}_{t}) \approx \\ h_{\mathbf{y}}(\mathbf{y}) \frac{\exp\left(-A_{\boldsymbol{\theta}}(\boldsymbol{\nu}(\mathbf{x}_t), \boldsymbol{\tau}(\mathbf{x}_t))\right)}{\exp\left(-A_{\boldsymbol{\theta}}(\mathbf{T}_{\mathbf{y}}(\mathbf{y}) + \boldsymbol{\nu}(\mathbf{x}_t), \boldsymbol{\tau}(\mathbf{x}_t) + N \mathbf{1}_d)\right)}.
\end{multline}

\subsection{Approximate Inference of  
$p_{\protect\boldsymbol{\theta}|\mathbf{x}_t}(\protect\boldsymbol{\theta}|\mathbf{x}_t)$}

In this section, we outline the process of finding the optimal approximation to $p_{\boldsymbol{\theta}|\mathbf{x}_t}(\boldsymbol{\theta}|\mathbf{x}_{t})$ by minimizing the KL divergence relative to $q_{\boldsymbol{\theta}\vert \boldsymbol{\zeta}(\mathbf{x}_t)}(\boldsymbol{\theta}\vert \boldsymbol{\zeta}(\mathbf{x}_t))$. For the next result, it is convenient to denote the log-partition function of the conjugate prior defined in~\eqref{eq-prior-q-theta} with $A_{\boldsymbol{\theta}}(\boldsymbol{\zeta}(\mathbf{x}_t)) := A_{\boldsymbol{\theta}}(\boldsymbol{\nu}(\mathbf{x}_t), \boldsymbol{\tau}(\mathbf{x}_t))$.
\begin{lemma}[KL Divergence of $p_{\boldsymbol{\theta}\vert \mathbf{x}_t}$ from $q_{\boldsymbol{\theta}\vert \boldsymbol{\zeta}(\mathbf{x}_t)}$]
\label{lemma-KL-divergence}
Let~$\boldsymbol{\theta} = g^{-1}(\mathbf{x}_0)$.
Furthermore, let
$q_{\boldsymbol{\theta}\vert \boldsymbol{\zeta}(\mathbf{x}_t)}(\boldsymbol{\theta}\vert \boldsymbol{\zeta}(\mathbf{x}_t))$ be defined as in~\eqref{eq-prior-q-theta} and be part of the exponential family with
hyperparameters $\boldsymbol{\zeta}(\mathbf{x}_t)$, base measure $h_{\boldsymbol{\theta}}(\boldsymbol{\theta})$, sufficient statistics $\mathbf{T}_{\boldsymbol{\theta}}(\boldsymbol{\theta})$ and
log-partition function $A_{\boldsymbol{\theta}}(\boldsymbol{\zeta}(\mathbf{x}_t))$.  
The KL divergence of $p_{\boldsymbol{\theta}\vert \mathbf{x}_t}$ from $q_{\boldsymbol{\theta}\vert \boldsymbol{\zeta}(\mathbf{x}_t)}$ is given by
\begin{equation}
\label{eq-kl-divergence-exp}
    D_{\text{KL}}(p_{\boldsymbol{\theta}\vert \mathbf{x}_t} \vert\vert q_{\boldsymbol{\theta}\vert \boldsymbol{\zeta}(\mathbf{x}_t)}) =  C(\mathbf{x}_t)  +\mathcal{L}_{\text{AVI}}(\boldsymbol{\zeta}, \mathbf{x}_t)
\end{equation}
where 
\begin{multline*}
\mathcal{L}_{\text{AVI}}\left(\boldsymbol{\zeta},\mathbf{x}_t\right) =\\  A_{\boldsymbol{\theta}}(\boldsymbol{\zeta}(\mathbf{x}_t)) -\boldsymbol{\zeta}(\mathbf{x}_t)^{\top}\mathbb{E}_{p_{\tilde{\mathbf{x}}_{0}\vert\mathbf{x}_t}}[\mathbf{T}_{\boldsymbol{\theta}}(g^{-1}(\tilde{\mathbf{x}}_{0}))]
\end{multline*}
and for a function $C(\mathbf{x}_t)$ that does not depend on $\boldsymbol{\zeta}$.
\end{lemma}
The proof of Lemma \ref{lemma-KL-divergence} is postponed to Appendix \ref{sec:proof_lemma-KL-divergence}.
We define $\boldsymbol{\zeta}^{\star}(\mathbf{x}_t)$ as the set of hyperparameters that minimizes the KL divergence in~\eqref{eq-kl-divergence-exp} for a given input $\mathbf{x}_t$. Furthermore, let  $\boldsymbol{\zeta}^{\star}(\cdot)$ denote the function that minimizes the expected KL divergence, as specified by the objective
\begin{equation}
\label{eq-kl-expectation}
\mathcal{J}_{\text{AVI}}(\boldsymbol{\zeta}) = \mathbb{E}_{t\sim U(\epsilon, 1), \mathbf{x}_t \sim p_{\mathbf{x}_t}}\Big[ \mathcal{L}_{\text{AVI}}(\boldsymbol{\zeta},\mathbf{x}_t,t)\Big].
\end{equation}
where the function $C(\mathbf{x}_t)$ in~\eqref{eq-kl-divergence-exp} has been excluded from the optimization, as it does not depend on $\boldsymbol{\zeta}$. 
A significant challenge in optimizing~\eqref{eq-kl-expectation} arises from the term $\mathbb{E}_{p_{\tilde{\mathbf{x}}_{0}|\mathbf{x}_t}}[\mathbf{T}_{\boldsymbol{\theta}}(g^{-1}(\tilde{\mathbf{x}}_{0}))]$. This term requires computing expectations under the reverse process distribution $p_{\tilde{\mathbf{x}}_{0}|\mathbf{x}_t}$, which is generally intractable. To address this issue, our next result demonstrates that the objective in~\eqref{eq-kl-expectation} can be reformulated in a way that entirely avoids this explicit evaluation.
\begin{theorem}
\label{prop-new-objective}
Let $p_{\boldsymbol{\theta}}(\boldsymbol{\theta})$ be the marginal distribution of $\boldsymbol{\theta}$.  Moreover, assume that $\boldsymbol{\zeta}(\cdot)$ is a Lipschitz continuous function and that the following conditions hold:
\begin{equation*}
\begin{aligned}
\mathbb{E}_{\boldsymbol{\theta}\sim  p_{\boldsymbol{\theta}}}[\norm{\mathbf{T}_{\boldsymbol{\theta}}(\boldsymbol{\theta}) }] &< \infty,\\
\mathbb{E}_{\boldsymbol{\theta}\sim  p_{\boldsymbol{\theta}}}[\norm{g(\boldsymbol{\theta})}\norm{\mathbf{T}_{\boldsymbol{\theta}}(\boldsymbol{\theta}) }] &< \infty 
\end{aligned}
\end{equation*}
Then, the objective in~\eqref{eq-kl-expectation} can be equivalently expressed as:
\begin{multline}
\label{eq-amortized-objective}
\mathcal{J}_{\text{AVI}}(\boldsymbol{\zeta}) = \\  \mathbb{E}_{t\sim U(\epsilon, 1), \mathbf{x}_0\sim  p_{\mathbf{x}_0}, \mathbf{x}_t \sim p_{\mathbf{x}_t|\mathbf{x}_0}}\Big[ \tilde{\mathcal{L}}_{\text{AVI}}(\boldsymbol{\zeta}, \mathbf{x}_0,\mathbf{x}_t)\Big]
\end{multline}
where
\begin{multline*}
\tilde{\mathcal{L}}_{\text{AVI}}(\boldsymbol{\zeta}, \mathbf{x}_0,\mathbf{x}_t) =  A_{\boldsymbol{\theta}}(\boldsymbol{\zeta}(\mathbf{x}_t))  - \boldsymbol{\zeta}(\mathbf{x}_t)^{\top}\mathbf{T}_{\boldsymbol{\theta}}(g^{-1}(\mathbf{x}_{0})).
\end{multline*}
\end{theorem}
The proof of Theorem~\ref{prop-new-objective} is deferred to Appendix~\ref{proof-prop-new-objective}. 
To approximate $\boldsymbol{\zeta}^{\star}(\cdot)$, we adopt the framework of \textit{amortized variational inference} (AVI).
We use a neural network, denoted by $\boldsymbol{\zeta}_{\boldsymbol{\rho}}(\mathbf{x}_t, t)$ where $\boldsymbol{\rho}$ represents the trainable parameters of the network. We train the neural network such that the parameters $\boldsymbol{\rho}^{*}$ are a minimizer of the following amortized objective:
\begin{multline*}
\mathcal{J}_{\text{AVI}}(\boldsymbol{\rho}) = \\ \mathbb{E}_{t\sim U(\epsilon, 1), \mathbf{x}_0\sim  p_{\mathbf{x}_0},\mathbf{x}_t \sim p_{\mathbf{x}_t|\mathbf{x}_0}}\Big[ \tilde{\mathcal{L}}_{\text{AVI}}(\boldsymbol{\rho},\mathbf{x}_0,\mathbf{x}_t, t)\Big]
\end{multline*}
where
\begin{multline*}
\tilde{\mathcal{L}}_{\text{AVI}}(\boldsymbol{\rho},\mathbf{x}_0,\mathbf{x}_t, t) =  \\  A_{\boldsymbol{\theta}}(\boldsymbol{\zeta}_{\boldsymbol{\rho}}(\mathbf{x}_t, t))  -\boldsymbol{\zeta}_{\boldsymbol{\rho}}(\mathbf{x}_t, t)^{\top}\mathbf{T}_{\boldsymbol{\theta}}(g^{-1}(\mathbf{x}_{0})).
\end{multline*}
\begin{remark}[Inference Network]
The function $\boldsymbol{\zeta}_{\boldsymbol{\rho}}(\mathbf{x}_t,t)$ serves as an \textit{inference network} that infers a posterior distribution over the original (denoised) parameter vector $\boldsymbol{\theta}$, conditioned on its progressively noised counterpart 
$\mathbf{x}_t$ at diffusion time step $t$. Implemented via a neural network, $\boldsymbol{\zeta}_{\boldsymbol{\rho}}$  maps the noisy input $\mathbf{x}_t$ and timestep $t$ to the parameters of this posterior distribution, effectively approximating the inverse of the forward noising process.
\end{remark}


\subsection{Computing the Score of $p_{\mathbf{y}|\mathbf{x}_t}(\mathbf{y}|\mathbf{x}_{t})$}
To sample from the posterior using diffusion models, it is necessary to approximate the likelihood score function, $\nabla_{\mathbf{x}_t} \log p_{\mathbf{y} \vert \mathbf{x}_t}(\mathbf{y} \vert \mathbf{x}_t)$, as defined in~\eqref{eq:score_posterior}. 
From~\eqref{eq:likelihood_t-closed-form}, the log-density $\log p_{\mathbf{y}|\mathbf{x}_t}(\mathbf{y}|\mathbf{x}_t)$ can be directly approximated with
\begin{multline*}
    \log p_{\mathbf{y}|\mathbf{x}_t}(\mathbf{y}|\mathbf{x}_{t})
     \approx \log h_{\mathbf{y}}(\mathbf{y}) - A_{\boldsymbol{\theta}}(\boldsymbol{\nu}(\mathbf{x}_t), \boldsymbol{\tau}(\mathbf{x}_t)) \\+ A_{\boldsymbol{\theta}}(\mathbf{T}_{\mathbf{y}}(\mathbf{y}) + \boldsymbol{\nu}(\mathbf{x}_t), \boldsymbol{\tau}(\mathbf{x}_t) + N \mathbf{1}_d).
\end{multline*}
The gradient of the log-density,~$\nabla_{\mathbf{x}_t} \log  p_{\mathbf{y}|\mathbf{x}_t}(\mathbf{y}|\mathbf{x}_{t})$ with respect to $\mathbf{x}_t$ can be efficiently computed using automatic differentiation.

\section{Experiments} \label{sec:experiment}
The technical details of each experiment discussed in this section can be found in Appendix~\ref{app-experiment-set-up}. The first experiment demonstrates the effectiveness of our method in approximating a posterior distribution that closely aligns with the ground truth obtained via MCMC, using a hierarchical model where the prior can be evaluated. The final two experiments address scenarios where an empirical prior is used, making MCMC infeasible.

\begin{figure*}[t!]
\centering
\includegraphics[width=\textwidth]{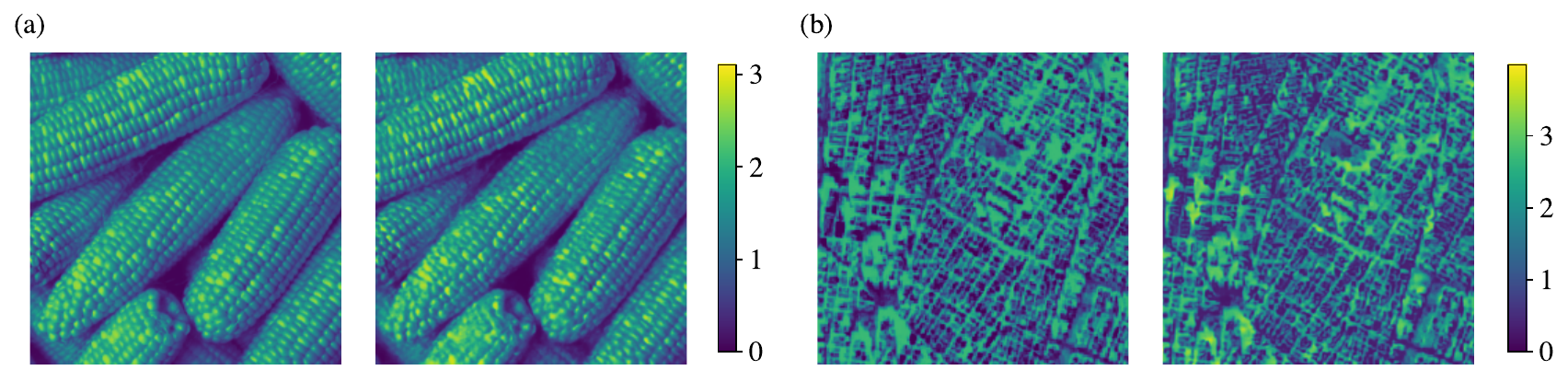}
\caption{\textbf{Score-Based Cox Process Results.} \textbf{(a)} (Left) True Cox Process intensity from the ImageNet validation set, transformed using an exponential link function. (Right) Median of the estimated Cox Process intensity posterior distribution using the Score-Based Cox Process method. \textbf{(b)} (Left) True Cox Process Intensity from Sentinel-2 Satellite Imagery of Manhattan, New York City (Right) Median of the estimated Cox Process intensity posterior distribution using the Score-Based Cox Process method.}
\label{fig:imagenet-cox-process}
\end{figure*} 

\subsection{One-dimensional Benchmark Analysis}
\label{sec-1d-synthetic-data}
We developed a simple one-dimensional experiment to illustrate the effectiveness of our method in approximating a posterior distribution that closely aligns with the ground truth obtained through MCMC.
This experiment also serves as a basis for comparing our posterior approximation with that generated by the DPS method. 
We consider the following hierarchical generative model:
\begin{equation*}
y_{i,j} \sim \text{Poisson}(\theta_j), \quad \boldsymbol{\theta} = \exp(\mathbf{x}_0), \quad \mathbf{x}_0 \sim \mathcal{GP}(0, \mathbf{K})
\end{equation*}
for $i = 1, \ldots, N$ and $j = 1, \dots d$, where $d = 30$ and $\mathbf{K}$ is the Gaussian Process (GP) covariance matrix defined by a radial basis function (RBF) kernel with variance 1 and length-scale 0.1. 
Using this model, we generated synthetic observations. We aim to solve the inverse problem of recovering the unknown Poisson intensity $\boldsymbol{\theta}$ from the generated synthetic observations. We used as a prior the true latent variable Gaussian Process distribution. 

We compared the posterior distribution of $\boldsymbol{\theta}$ estimated by our method against the ground-truth MCMC posterior as well as the DPS posterior approximation. 
The results of this comparison are provided in Appendix~\ref{app-results-synthetic-1d}.
Our approach demonstrates significantly better alignment with the ground-truth MCMC posterior. In contrast, DPS fails to accurately capture both the credible intervals and the point estimates.
To further assess robustness, we repeated this experiment using other distributions within the exponential family, for which DPS could not be used. Our method consistently aligned with the ground-truth MCMC posterior distribution. 

\subsection{Score-Based Cox Process}
\label{sec-experiment-cox-process}
A Cox process, also called a doubly stochastic Poisson process, is a point process that generalizes the Poisson process by allowing its intensity function to be governed by a stochastic process, varying across the underlying mathematical space. The space over which the intensity function is defined is discretized to be a $256 \times 256$ grid. Each grid cell's observation is a Poisson random variable, parameterized by the corresponding intensity value.

To generate synthetic Cox Process observations, we explored multiple intensities including samples from the ImageNet validation dataset, a satellite image, and a map of buildings' heights in London. For each choice of intensity, we drew $N = 50$ event samples according to a Cox Process and allocated 80\% of the grid cells to the training set and the remaining 20\% to the test set.

To address the inverse problem, we employed the ImageNet prior. This prior assumes that $\mathbf{x}_0$ are samples from the ImageNet train dataset. We use the exponential inverse link function. The hierarchical generative model was:
\begin{equation*}
y_{i,j} \sim \text{Poisson}(\theta_j), \quad
\boldsymbol{\theta} = \exp(\mathbf{x}_0), \quad
\mathbf{x}_0 \sim \text{ImageNet}
\end{equation*}
for $i = 1, \ldots, N$, $j = 1, \ldots d$, and where $d = 256\times 256$. We refer to this method as the \textit{``Score-Based Cox Process"}. It should be noted that MCMC inference cannot be used due to the intractability of the prior density.
Figure~\ref{fig:imagenet-cox-process} shows the results of the \textit{``Score-Based Cox Process"} on recovering the true intensity surface. Further experimental results given different values of $N$ and different intensities are provided in Appendix~\ref{app-further-experiment-cox-process}.

\subsection{Prevalence of Malaria Prevalence in Sub-Saharan Africa}
\label{sec-experiment-malaria}
The \emph{Plasmodium falciparum} parasite rate (PfPR) quantifies the proportion of individuals who have the malaria parasite. The data used to estimate the PfPR consist of the number of positive cases in location $j$, denoted as $y_j$ (detected using rapid diagnostic tests or PCR), out of the total number of individuals examined in the same location, $n_j$. Spatio-temporal mapping of PfPR is typically conducted using GPs~\cite{Bhatt2015-uk}. However, the growing volume of data has rendered full-rank Bayesian inference with GPs computationally impractical. Furthermore, the simple covariance functions commonly used in GPs may be inadequate, necessitating increasingly complex models to accurately predict PfPR across spatial and temporal dimensions \cite{Bhatt2017-tk}. 
Here, we reanalyzed a real-world dataset on PfPR from the Malaria Atlas Project, previously used to monitor malaria trends in Sub-Saharan Africa~\citep{Bhatt2015-uk,Pfeffer2018-cm, Weiss2019-au} --- the continent bearing the highest burden of the disease. 
We ignored temporal aspects, and only aimed to interpolate spatial data across all of Sub-Saharan Africa. We used a grid resolution of $256 \times 256$, equivalent to a $\sim 111 \text{ km}^2$ resolution, and aggregated positive cases and individuals examined to this resolution. Out of the grid, $7,048$ ($10.75$\%) entries had non-missing observations, which were then split into training and test sets in an 80/20 ratio. The hierarchical generative model was:
\begin{equation*}
y_{j} \sim \text{Binomial}(n_j,\theta_j), \; \boldsymbol{\theta} = \sigma(s\,\mathbf{x}_0), \;
\mathbf{x}_0 \sim \text{ImageNet}
\end{equation*}
for $j = 1, \ldots d$, $s = 5$ and where $d = 256\times 256$, and where $\sigma(\cdot)$ is the sigmoid (inverse logit) function.
Figure~\ref{fig:malaria-results} presents the PfPR posterior median and credible interval estimated using our approach. A benchmark analysis comparing our method to the Gaussian Markov Random Field (GMRF) --- considered the state-of-the-art for disease mapping~\citep{Rue2009-ty, Lindgren2011-fv, Heaton2017-vl} --- is provided in Appendix~\ref{app-further-experiment-malaria}. Our results show that our approach performs competitively with the GMRF model.

\begin{figure*}[ht!]
\centering
\includegraphics[width=\textwidth]{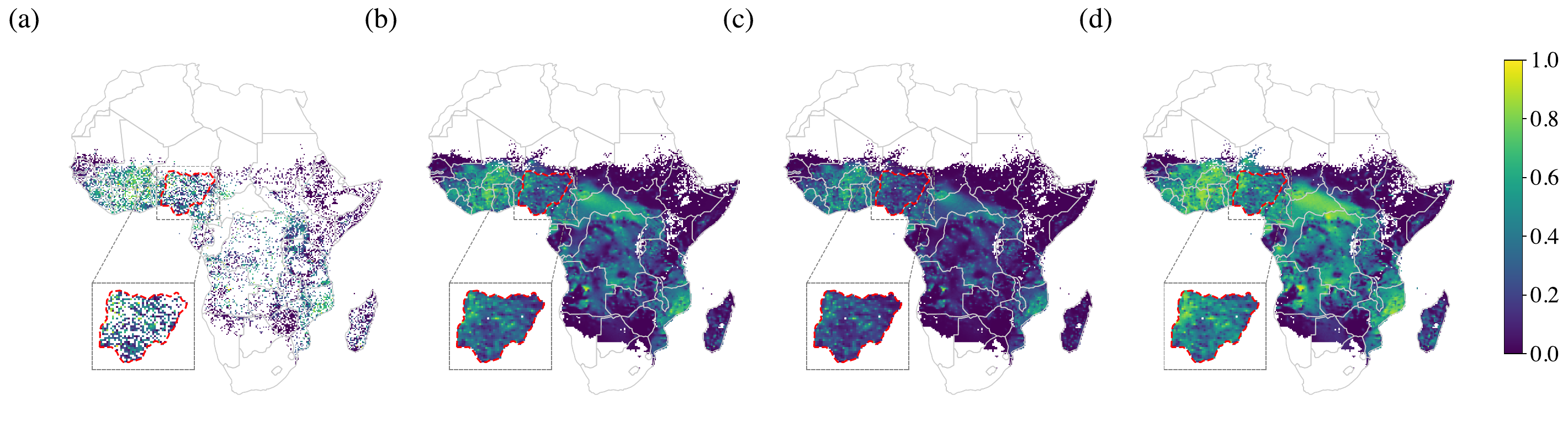}
\caption{\textbf{Prevalence of Malaria in Sub-Saharan Africa Results.} \textbf{(a)} Empirical PfPR. \textbf{(b)} Median of the estimated PfPR posterior distribution. \textbf{(c)} $25$\% quantile of the estimated PfPR posterior distribution. \textbf{(d)} $75$\% quantile of the estimated PfPR posterior distribution. 
The inset plots highlight Nigeria, one of the countries with the highest malaria burden worldwide.
The empty entries either correspond to locations outside Sub-Saharan Africa or the stable spatial limits of \emph{P. falciparum} transmission~\cite{Bhatt2015-uk} }
\label{fig:malaria-results}
\end{figure*}

\section{Related Work}\label{sec:related_work}
Since our work focuses on addressing inverse problems with non-Gaussian observations, we review several approaches that have also attempted to solve this problem.

\paragraph{Markov Chain Monte Carlo.} MCMC is the most commonly used posterior sampling method for performing Bayesian inference on inverse problems.  A key characteristic of MCMC methods is the need to evaluate the prior to compute the acceptance rate for candidate samples generated by the proposal distribution. This requirement presents a significant limitation compared to diffusion-based approaches, which only require the ability to sample from the prior distribution. Therefore, diffusion-based methods accommodate a much broader range of prior distributions.

\paragraph{Gaussian Process and Gaussian Markov Random Field.} GPs are widely used for modeling latent functions in tasks where capturing uncertainty is crucial but are computationally expensive due to the worst-case cubic complexity of inverting large covariance matrices \citep{Rasmussen_GP, Adams2009}. For non-Gaussian observations, Bayesian inference can only be performed via MCMC sampling, which suffers from autocorrelation and slow mixing, making it impractical for large-scale experiments. Given our dataset size and parameter dimensionality, MCMC-based GP inference was infeasible.

As an alternative to MCMC, the most popular approximate method is the integrated nested Laplace approximation (INLA) \cite{Rue2005}, combined with GMRFs. INLA enables sparse computations and avoids MCMC’s mixing issues through an optimization-based approach. However, INLA has limitations: it restricts covariance functions to stationary ones, struggles with high spectral frequencies \cite{Stein2014-hc}, lacks posterior accuracy guarantees, and makes obtaining posterior samples challenging.


\paragraph{Diffusion Posterior Sampling.} Among existing diffusion models-based methodologies, we mention DPS, the work of~\citet{chung2023}, who proposed to approximate $p_{\mathbf{x}_{0}\vert\mathbf{x}_t}(\mathbf{x}_0 \vert \mathbf{x}_t)$ as a Dirac delta distribution centered at the posterior mean $\mathbb{E}_{\mathbf{x}_{0}\sim p_{\mathbf{x}_{0}\vert\mathbf{x}_t}}[\mathbf{x}_0]$.
The latter is determined using Tweedie’s formula. Although their method was originally designed for linear inverse problems with Gaussian likelihoods, the authors also extended it to address inverse problems involving Poisson-distributed observations. 
This approach relies on the assumption that Gaussian distributions can effectively approximate Poisson-distributed data when the rate is sufficiently high. However, as noted in \citep[Appendix C.4]{chung2023} and illustrated in the experiment presented in Section~\ref{sec-1d-synthetic-data}, the method faces numerical instabilities and produces poor approximations when the observations come from a low-rate Poisson distribution. Furthermore, the method depends on Tweedie’s formula, which is known to exhibit high variance at high noise levels during the reverse diffusion process (see Section 1.2 of~\citet{target_score_matching}).

\paragraph{Simulation-Based Inference (SBI).} 
SBI avoids the need for a tractable likelihood by relying on simulated observations. Diffusion models enable SBI by approximating the likelihood score function through a Conditional Denoising Estimator (CDE) in the form of a neural network. The CDE directly estimates the likelihood score function by conditioning on three inputs: the observations $\mathbf{y}$, the noise-corrupted latent variable $\mathbf{x}_t$ and the diffusion timestep $t$~\citep{batzolis2021,simons2023}.

Existing approaches face three critical limitations. First, using observations 
$\mathbf{y}$ as network input requires retraining for each new dataset, incurring high computational costs.
Second, in a multiple samples regime ($N>1$), the likelihood score function depends on both the prior and the likelihood score networks~\citep{geffner23a}, causing errors from the prior network to propagate into the likelihood approximation. Third, these methods struggle to accommodate heterogeneous missing data patterns across observations. This limitation stems from their reliance on a fixed input structure for $\mathbf{y}$: missing observations can only be processed if they conform to the network's predefined input format, restricting their applicability to real-world datasets with variable or unanticipated missingness. In contrast, our approach trains a single network given a choice of likelihood, decoupling it from the specific missing-data pattern and the observations themselves. 

\section{Conclusion} \label{sec:conclusion}
In this work, we introduced a novel approach for solving inverse problems using diffusion models when observations follow distributions from the exponential family. 
Our posterior approximation closely aligns with MCMC methods while scaling to larger observational datasets and accommodating empirical priors. We demonstrate strong performance in image denoising under Poisson noise and further highlight the method’s effectiveness in real-world problems. Notably, our results suggest that an ImageNet prior can be a powerful tool for spatial statistics, enabling the recovery of latent patterns that extend beyond those present in ImageNet itself.

\section*{Impact Statement}
Most real-world phenomena, such as disease case counts, the time between occurrences of natural climate events, and the proportion of individuals with certain health conditions, exhibit non-Gaussian characteristics. A key challenge in modeling such data lies in solving inverse problems to estimate latent functions of scientific interest. Diffusion models have consistently demonstrated exceptional performance in addressing inverse problems and form the foundation of many recent breakthroughs in artificial intelligence research. By introducing an approach that leverages diffusion models within the framework of the exponential family, we significantly expand the scope of applications where these models can be effectively leveraged.

\bibliography{ref}

\begin{thebibliography}{36}
\providecommand{\natexlab}[1]{#1}
\providecommand{\url}[1]{\texttt{#1}}
\expandafter\ifx\csname urlstyle\endcsname\relax
  \providecommand{\doi}[1]{doi: #1}\else
  \providecommand{\doi}{doi: \begingroup \urlstyle{rm}\Url}\fi

\bibitem[Adams et~al.(2009)Adams, Murray, and MacKay]{Adams2009}
Adams, R.~P., Murray, I., and MacKay, D. J.~C.
\newblock Tractable {N}onparametric {B}ayesian {I}nference in {P}oisson {P}rocesses with {G}aussian {P}rocess {I}ntensities.
\newblock In \emph{Proceedings of the 26th Annual International Conference on Machine Learning}, ICML '09, pp.\  9–16, New York, NY, USA, 2009. Association for Computing Machinery.

\bibitem[Anderson(1982)]{anderson_1982}
Anderson, B.~D.
\newblock Reverse-time diffusion equation models.
\newblock \emph{Stochastic Processes and their Applications}, 12\penalty0 (3):\penalty0 313--326, 1982.

\bibitem[Batzolis et~al.(2021)Batzolis, Stanczuk, Schönlieb, and Etmann]{batzolis2021}
Batzolis, G., Stanczuk, J., Schönlieb, C.-B., and Etmann, C.
\newblock Conditional {I}mage {G}eneration with {S}core-{B}ased {D}iffusion {M}odels, 2021.
\newblock \textit{ar{X}iv preprint ar{X}iv:2111.13606}.

\bibitem[Bhatt et~al.(2015)Bhatt, Weiss, Cameron, Bisanzio, Mappin, Dalrymple, Battle, Moyes, Henry, Eckhoff, Wenger, Briët, Penny, Smith, Bennett, Yukich, Eisele, Griffin, Fergus, Lynch, Lindgren, Cohen, Murray, Smith, Hay, Cibulskis, and Gething]{Bhatt2015-uk}
Bhatt, S., Weiss, D.~J., Cameron, E., Bisanzio, D., Mappin, B., Dalrymple, U., Battle, K.~E., Moyes, C.~L., Henry, A., Eckhoff, P.~A., Wenger, E.~A., Briët, O., Penny, M.~A., Smith, T.~A., Bennett, A., Yukich, J., Eisele, T.~P., Griffin, J.~T., Fergus, C.~A., Lynch, M., Lindgren, F., Cohen, J.~M., Murray, C. L.~J., Smith, D.~L., Hay, S.~I., Cibulskis, R.~E., and Gething, P.~W.
\newblock The effect of malaria control on \textit{Plasmodium falciparum} in {A}frica between 2000 and 2015.
\newblock \emph{Nature}, 526\penalty0 (7572):\penalty0 207--211, 2015.

\bibitem[Bhatt et~al.(2017)Bhatt, Cameron, Flaxman, Weiss, Smith, and Gething]{Bhatt2017-tk}
Bhatt, S., Cameron, E., Flaxman, S.~R., Weiss, D.~J., Smith, D.~L., and Gething, P.~W.
\newblock Improved prediction accuracy for disease risk mapping using {G}aussian process stacked generalization.
\newblock \emph{Journal of the Royal Society Interface}, 14\penalty0 (134), 2017.

\bibitem[Billingsley(2012)]{billingsley_prob}
Billingsley, P.
\newblock \emph{Probability and Measure}.
\newblock Wiley Series in Probability and Statistics. Wiley-Blackwell, 2012.

\bibitem[Bishop(2016)]{bishop_pattern_recognition}
Bishop, C.
\newblock \emph{Pattern Recognition and Machine Learning}.
\newblock Information Science and Statistics. Springer, New York, NY, 2016.

\bibitem[Boys et~al.(2024)Boys, Girolami, Pidstrigach, Reich, Mosca, and Akyildiz]{boys2024}
Boys, B., Girolami, M., Pidstrigach, J., Reich, S., Mosca, A., and Akyildiz, O.~D.
\newblock Tweedie {M}oment {P}rojected {D}iffusions {F}or {I}nverse {P}roblems, 2024.
\newblock \textit{ar{X}iv preprint ar{X}iv:2310.06721}.

\bibitem[Chung et~al.(2023)Chung, Kim, Mccann, Klasky, and Ye]{chung2023}
Chung, H., Kim, J., Mccann, M.~T., Klasky, M.~L., and Ye, J.~C.
\newblock Diffusion {P}osterior {S}ampling for {G}eneral {N}oisy {I}nverse {P}roblems.
\newblock In \emph{The Eleventh International Conference on Learning Representations}, 2023.

\bibitem[Cox(1955)]{cox_process}
Cox, D.~R.
\newblock Some {S}tatistical {M}ethods {C}onnected with {S}eries of {E}vents.
\newblock \emph{Journal of the Royal Statistical Society: Series B (Methodological)}, 17\penalty0 (2):\penalty0 129--157, 1955.

\bibitem[Daras et~al.(2024)Daras, Chung, Lai, Mitsufuji, Ye, Milanfar, Dimakis, and Delbracio]{Daras2024}
Daras, G., Chung, H., Lai, C.-H., Mitsufuji, Y., Ye, J.~C., Milanfar, P., Dimakis, A.~G., and Delbracio, M.
\newblock A {S}urvey on {D}iffusion {M}odels for {I}nverse {P}roblems, 2024.
\newblock \textit{ar{X}iv preprint ar{X}iv:2410.00083}.

\bibitem[De~Bortoli et~al.(2024)De~Bortoli, Hutchinson, Wirnsberger, and Doucet]{target_score_matching}
De~Bortoli, V., Hutchinson, M., Wirnsberger, P., and Doucet, A.
\newblock Target {S}core {M}atching, 2024.
\newblock \textit{ar{X}iv preprint ar{X}iv:2402.08667}.

\bibitem[Efron(2011)]{Efron2011}
Efron, B.
\newblock Tweedie’s {F}ormula and {S}election {B}ias.
\newblock \emph{Journal of the American Statistical Association}, 106\penalty0 (496):\penalty0 1602–1614, 2011.

\bibitem[Geffner et~al.(2023)Geffner, Papamakarios, and Mnih]{geffner23a}
Geffner, T., Papamakarios, G., and Mnih, A.
\newblock Compositional {S}core {M}odeling for {S}imulation-{B}ased {I}nference.
\newblock In \emph{Proceedings of the 40th International Conference on Machine Learning}, volume 202 of \emph{Proceedings of Machine Learning Research}, pp.\  11098--11116. PMLR, 2023.

\bibitem[{Greater London Authority (GLA), Mayor of London}(2025)]{LBSM}
{Greater London Authority (GLA), Mayor of London}.
\newblock {London Building Stock Model (LBSM)}.
\newblock Available at \sloppy\url{https://data.london.gov.uk/dataset/london-building-stock-model}, 2025.

\bibitem[Heaton et~al.(2018)Heaton, Datta, Finley, Furrer, Guinness, Guhaniyogi, Gerber, Gramacy, Hammerling, Katzfuss, Lindgren, Nychka, Sun, and Zammit-Mangion]{Heaton2017-vl}
Heaton, M.~J., Datta, A., Finley, A.~O., Furrer, R., Guinness, J., Guhaniyogi, R., Gerber, F., Gramacy, R.~B., Hammerling, D., Katzfuss, M., Lindgren, F., Nychka, D.~W., Sun, F., and Zammit-Mangion, A.
\newblock A {C}ase {S}tudy {C}ompetition {A}mong {M}ethods for {A}nalyzing {L}arge {S}patial {D}ata.
\newblock \emph{Journal of Agricultural, Biological and Environmental Statistics}, 24\penalty0 (3):\penalty0 398–425, 2018.

\bibitem[Ho et~al.(2020)Ho, Jain, and Abbeel]{ho_denoising}
Ho, J., Jain, A., and Abbeel, P.
\newblock Denoising {D}iffusion {P}robabilistic {M}odels.
\newblock In \emph{Advances in Neural Information Processing Systems}, volume~33, pp.\  6840--6851, 2020.

\bibitem[Johnson et~al.(2016)Johnson, Duvenaud, Wiltschko, Adams, and Datta]{Johnson2016}
Johnson, M.~J., Duvenaud, D.~K., Wiltschko, A., Adams, R.~P., and Datta, S.~R.
\newblock Composing graphical models with neural networks for structured representations and fast inference.
\newblock In \emph{Advances in Neural Information Processing Systems}, volume~29, 2016.

\bibitem[Kadkhodaie \& Simoncelli(2021)Kadkhodaie and Simoncelli]{kadkhodaie2021}
Kadkhodaie, Z. and Simoncelli, E.~P.
\newblock Stochastic solutions for linear inverse problems using the prior implicit in a denoiser.
\newblock In Beygelzimer, A., Dauphin, Y., Liang, P., and Vaughan, J.~W. (eds.), \emph{Advances in Neural Information Processing Systems}, 2021.

\bibitem[Kawar et~al.(2021)Kawar, Vaksman, and Elad]{kawar2021}
Kawar, B., Vaksman, G., and Elad, M.
\newblock {SNIPS}: {S}olving {N}oisy {I}nverse {P}roblems {S}tochastically.
\newblock In \emph{Advances in Neural Information Processing Systems}, 2021.

\bibitem[Kawar et~al.(2022)Kawar, Elad, Ermon, and Song]{kawar2022}
Kawar, B., Elad, M., Ermon, S., and Song, J.
\newblock {D}enoising {D}iffusion {R}estoration {M}odels.
\newblock In \emph{Advances in Neural Information Processing Systems}, 2022.

\bibitem[Lindgren et~al.(2011)Lindgren, Rue, and Lindström]{Lindgren2011-fv}
Lindgren, F., Rue, H., and Lindström, J.
\newblock An explicit link between {G}aussian fields and {G}aussian {M}arkov random fields: the stochastic partial differential equation approach.
\newblock \emph{Journal of the Royal Statistical Society Series B: Statistical Methodology}, 73\penalty0 (4):\penalty0 423--498, 2011.

\bibitem[Pfeffer et~al.(2018)Pfeffer, Lucas, May, Harris, Rozier, Twohig, Dalrymple, Guerra, Moyes, Thorn, Nguyen, Bhatt, Cameron, Weiss, Howes, Battle, Gibson, and Gething]{Pfeffer2018-cm}
Pfeffer, D.~A., Lucas, T. C.~D., May, D., Harris, J., Rozier, J., Twohig, K.~A., Dalrymple, U., Guerra, C.~A., Moyes, C.~L., Thorn, M., Nguyen, M., Bhatt, S., Cameron, E., Weiss, D.~J., Howes, R.~E., Battle, K.~E., Gibson, H.~S., and Gething, P.~W.
\newblock {malariaAtlas}: an {R} interface to global malariometric data hosted by the {M}alaria {A}tlas {P}roject.
\newblock \emph{Malaria Journal}, 17\penalty0 (1):\penalty0 352, 2018.

\bibitem[Rasmussen \& Williams(2005)Rasmussen and Williams]{Rasmussen_GP}
Rasmussen, C.~E. and Williams, C. K.~I.
\newblock \emph{Gaussian Processes for Machine Learning}.
\newblock The MIT Press, 2005.
\newblock ISBN 026218253X.

\bibitem[Royden \& Fitzpatrick(2010)Royden and Fitzpatrick]{Royden_2010}
Royden, H. and Fitzpatrick, P.
\newblock \emph{Real Analysis}.
\newblock Prentice Hall, 4th edition, 2010.
\newblock ISBN 978-0-13-143747-0.

\bibitem[Rozet et~al.(2024)Rozet, Andry, Lanusse, and Louppe]{rozet2024}
Rozet, F., Andry, G., Lanusse, F., and Louppe, G.
\newblock Learning {D}iffusion {P}riors from {O}bservations by {E}xpectation {M}aximization.
\newblock In \emph{The Thirty-eighth Annual Conference on Neural Information Processing Systems}, 2024.

\bibitem[Rue \& Held(2005)Rue and Held]{Rue2005}
Rue, H. and Held, L.
\newblock \emph{Gaussian Markov Random Fields: Theory and Applications}, volume 104 of \emph{Monographs on Statistics and Applied Probability}.
\newblock Chapman and Hall/CRC, 2005.

\bibitem[Rue et~al.(2009)Rue, Martino, and Chopin]{Rue2009-ty}
Rue, H., Martino, S., and Chopin, N.
\newblock Approximate {B}ayesian inference for latent {G}aussian models by using integrated nested {L}aplace approximations.
\newblock \emph{Journal of the Royal Statistical Society Series B: Statistical Methodology}, 71\penalty0 (2):\penalty0 319--392, 2009.

\bibitem[Simons et~al.(2023)Simons, Sharrock, Liu, and Beaumont]{simons2023}
Simons, J., Sharrock, L., Liu, S., and Beaumont, M.
\newblock Neural {S}core {E}stimation: {L}ikelihood-{F}ree {I}nference with {C}onditional {S}core {B}ased {D}iffusion models.
\newblock In \emph{Fifth Symposium on Advances in Approximate Bayesian Inference}, 2023.

\bibitem[Song et~al.(2023)Song, Vahdat, Mardani, and Kautz]{song2023pseudoinverseguided}
Song, J., Vahdat, A., Mardani, M., and Kautz, J.
\newblock Pseudoinverse-{G}uided {D}iffusion {M}odels for {I}nverse {P}roblems.
\newblock In \emph{International Conference on Learning Representations}, 2023.

\bibitem[Song et~al.(2021)Song, Sohl-Dickstein, Kingma, Kumar, Ermon, and Poole]{song2021scorebased}
Song, Y., Sohl-Dickstein, J., Kingma, D.~P., Kumar, A., Ermon, S., and Poole, B.
\newblock Score-{B}ased {G}enerative {M}odeling through {S}tochastic {D}ifferential {E}quations.
\newblock In \emph{International Conference on Learning Representations}, 2021.

\bibitem[{Stan Development Team}(2025)]{pystan}
{Stan Development Team}.
\newblock {P}y{S}tan, a {P}ython interface to {S}tan, a package for {B}ayesian inference.
\newblock Available at \url{https://pystan.readthedocs.io/en/latest/}, 2025.

\bibitem[Stein(2014)]{Stein2014-hc}
Stein, M.~L.
\newblock Limitations on low rank approximations for covariance matrices of spatial data.
\newblock \emph{Spatial Statistics}, 8:\penalty0 1--19, 2014.

\bibitem[Vaswani et~al.(2017)Vaswani, Shazeer, Parmar, Uszkoreit, Jones, Gomez, Kaiser, and Polosukhin]{Vaswani2017}
Vaswani, A., Shazeer, N., Parmar, N., Uszkoreit, J., Jones, L., Gomez, A.~N., Kaiser, L.~u., and Polosukhin, I.
\newblock Attention {I}s {A}ll {Y}ou {N}eed.
\newblock In \emph{Advances in Neural Information Processing Systems}, volume~30, 2017.

\bibitem[Vincent(2011)]{vincent2011}
Vincent, P.
\newblock A {C}onnection {B}etween {S}core {M}atching and {D}enoising {A}utoencoders.
\newblock \emph{Neural Computation}, 23\penalty0 (7):\penalty0 1661--1674, 2011.

\bibitem[Weiss et~al.(2019)Weiss, Lucas, Nguyen, Nandi, Bisanzio, Battle, Cameron, Twohig, Pfeffer, Rozier, Gibson, Rao, Casey, Bertozzi-Villa, Collins, Dalrymple, Gray, Harris, Howes, Kang, Keddie, May, Rumisha, Thorn, Barber, Fullman, Huynh, Kulikoff, Kutz, Lopez, Mokdad, Naghavi, Nguyen, Shackelford, Vos, Wang, Smith, Lim, Murray, Bhatt, Hay, and Gething]{Weiss2019-au}
Weiss, D.~J., Lucas, T. C.~D., Nguyen, M., Nandi, A.~K., Bisanzio, D., Battle, K.~E., Cameron, E., Twohig, K.~A., Pfeffer, D.~A., Rozier, J.~A., Gibson, H.~S., Rao, P.~C., Casey, D., Bertozzi-Villa, A., Collins, E.~L., Dalrymple, U., Gray, N., Harris, J.~R., Howes, R.~E., Kang, S.~Y., Keddie, S.~H., May, D., Rumisha, S., Thorn, M.~P., Barber, R., Fullman, N., Huynh, C.~K., Kulikoff, X., Kutz, M.~J., Lopez, A.~D., Mokdad, A.~H., Naghavi, M., Nguyen, G., Shackelford, K.~A., Vos, T., Wang, H., Smith, D.~L., Lim, S.~S., Murray, C. J.~L., Bhatt, S., Hay, S.~I., and Gething, P.~W.
\newblock Mapping the global prevalence, incidence, and mortality of \textit{Plasmodium falciparum}, 2000-17: a spatial and temporal modelling study.
\newblock \emph{Lancet}, 394\penalty0 (10195):\penalty0 322--331, 2019.

\end{thebibliography}
\bibliographystyle{style/icml2025}

\appendix
\onecolumn
\renewcommand{\thefigure}{A\arabic{figure}}  
\renewcommand{\thetable}{A\arabic{table}}    

\newpage
\clearpage
\section{Exponential Family}
\label{app-exponential-family}
In this appendix, we introduce the notation for exponential families and provide a summary of key results, drawing inspiration from~\citet[Appendix B]{Johnson2016}. Throughout this manuscript we take all densities to be absolutely continuous with respect to the appropriate Lebesgue measure (when the underlying set $\mathcal{Y}$ is Euclidean space) or counting measure (when $\mathcal{Y}$ is discrete), and denote the Borel $\sigma$-algebra of a set $\mathcal{Y}$ as $\mathcal{B}(\mathcal{Y})$ (generated by Euclidean and discrete topologies, respectively). We assume measurability of all functions as necessary.

Given a statistic function $\mathbf{T}_{\mathbf{y}} : \mathcal{Y} \to \mathbb{R}^k$ and a base measure $h_{\mathbf{y}}(\mathbf{y})$, we can define an exponential family of probability densities on $\mathcal{Y}$ and indexed by natural parameter $\boldsymbol{\eta} \in \mathbb{R}^k$ by
\begin{equation*}
p_{\mathbf{y}\vert \boldsymbol{\eta}}(\mathbf{y} \vert \boldsymbol{\eta}) \propto  h_{\mathbf{y}}(\mathbf{y}) \exp\left( \boldsymbol{\eta}^{\top} \mathbf{T}_{\mathbf{y}}(\mathbf{y})  \right), \quad \forall \boldsymbol{\eta} \in \mathbb{R}^k.
\end{equation*}
 We define the partition function  as
\begin{equation*}
Z_{\mathbf{y}}(\boldsymbol{\eta}) := \int h_{\mathbf{y}}(\mathbf{y}) \exp\left( \boldsymbol{\eta}^{\top} \mathbf{T}_{\mathbf{y}}(\mathbf{y})  \right) \mathrm{d}\mathbf{y}
\end{equation*}
and the log-partition function as
\begin{equation*}
A_{\mathbf{y}}(\boldsymbol{\eta}) := \log Z_{\mathbf{y}}(\boldsymbol{\eta}).
\end{equation*}
Lastly, we define $\mathcal{H} \subseteq \mathbb{R}^k$ to be the set of all normalizable natural parameters,
\begin{equation*}
\mathcal{H} := \left\{ \boldsymbol{\eta} \in \mathbb{R}^k : A_{\mathbf{y}}(\boldsymbol{\eta}) < \infty \right\}.
\end{equation*}
We can write the normalized probability density as
\begin{equation*}
p_{\mathbf{y}\vert \boldsymbol{\eta}}(\mathbf{y} \vert \boldsymbol{\eta}) =h_{\mathbf{y}}(\mathbf{y}) \exp \left(\boldsymbol{\eta}^{\top}\mathbf{T}_{\mathbf{y}}(\mathbf{y}) -A_{\mathbf{y}}(\boldsymbol{\eta}) \right).
\end{equation*}
We say that an exponential family is \emph{regular} if $\mathcal{H}$ is open, and \emph{minimal} if there is no $\boldsymbol{\eta} \in \mathbb{R}^k \setminus \{ 0 \}$ such that $\boldsymbol{\eta}^{\top} \mathbf{T}_{\mathbf{y}}(\mathbf{y}) = 0$. We assume all families are regular and minimal.

We parameterize the family by the parameters $\boldsymbol{\theta}$ instead of the natural parameters. We write the natural parameter as a continuous function of the parameters, $\boldsymbol{\eta}(\boldsymbol{\theta})$ and take $\Theta = \boldsymbol{\eta}^{-1}(\mathcal{H})$ to be the open set of parameters that correspond to normalizable densities. We summarize this notation in the following definition.

\begin{definition}[Exponential family of densities]
\label{def-exponential-family}
Given a measure space $(\mathcal{Y}, \mathcal{B}(\mathcal{Y}))$, a statistic function $\mathbf{T}_{\mathbf{y}} : \mathcal{Y} \to \mathbb{R}^k$, and a natural parameter function $\boldsymbol{\eta} : \boldsymbol{\theta} \to \mathbb{R}^k$, the corresponding exponential family of densities is
\begin{equation*}
p_{\mathbf{y}\vert\boldsymbol{\theta}}(\mathbf{y} \vert \boldsymbol{\theta}) = h_{\mathbf{y}}(\mathbf{y})   \exp \left( \boldsymbol{\eta}(\boldsymbol{\theta})^{\top} \mathbf{T}_{\mathbf{y}}(\mathbf{y}) - A_{\mathbf{y}}(\boldsymbol{\eta}(\boldsymbol{\theta})) \right).
\end{equation*}
\end{definition}
When we write exponential families of densities for different random variables, we change the subscripts on the statistic function, natural parameter function, and log partition function to correspond to the symbol used for the random variable. 

The next proposition shows that the log partition function of an exponential family generates cumulants of the statistic.

\begin{proposition}[Gradients of log-partition function and expected sufficient statistics]\label{prop:gradient_A_suff_stats}
The gradient of the log partition function $A_{\mathbf{y}}$ of an exponential family distributed random variable $\mathbf{y}$ equals its expected sufficient statistic:
\begin{equation*}
\nabla_{\boldsymbol{\eta}} A_{\mathbf{y}}(\boldsymbol{\eta}) = \mathbb{E}_{p_{\mathbf{y} \vert \boldsymbol{\eta}}} \left[ \mathbf{T}_{\mathbf{y}}(\mathbf{y}) \right],
\end{equation*}
where the expectation is taken with respect to the density $p_{\mathbf{y} \vert \boldsymbol{\eta}}(\mathbf{y} \vert \boldsymbol{\eta})$. 
\end{proposition}
It is convenient to introduce the following notation for the Jacobian matrix $\mathbf{J}_{\mathbf{f}}(\mathbf{x}) \in \mathbb{R}^{n \times m}$ of a vector-value function $\mathbf{f}:\mathbb{R}^{m}\to\mathbb{R}^{n}$ with $\mathbf{f}(\mathbf{x}) = (f_{1}(\mathbf{x}),\ldots,f_{n}(\mathbf{x}))$ and where
\begin{equation*}
[\mathbf{J}_{\mathbf{f}}(\mathbf{x})]_{ij} = \frac{\partial f_{i}}{\partial x_{i}}\Big\vert_{\mathbf{x}}
\end{equation*}
\begin{corollary} \label{corollary:derivative_log_partition_function}
From Proposition~\ref{prop:gradient_A_suff_stats}, applying the chain rule yields
\begin{equation*}
 \mathbf{J}_{\boldsymbol{\theta}}(\boldsymbol{\eta})\Big|_{\boldsymbol{\eta} = \boldsymbol{\eta}(\theta)} \nabla_{\boldsymbol{\theta}} A_{\mathbf{y}}(\boldsymbol{\theta}) =  \mathbb{E}_{p_{\mathbf{y} \vert \boldsymbol{\theta}}} \left[ \mathbf{T}_{\mathbf{y}}(\mathbf{y}) \right],
\end{equation*}
where the expectation is over the random variable $\mathbf{y}$ with density $p_{\mathbf{y} \vert \boldsymbol{\theta}}(\mathbf{y} \vert \boldsymbol{\theta})$. 
\end{corollary}


Given an exponential family of densities on $\mathcal{Y}$ as in Definition~\ref{def-exponential-family}, we can define a related exponential family of densities on $\boldsymbol{\theta}$ in terms of the functions $\boldsymbol{\eta}(\boldsymbol{\theta})$ and $A_{\mathbf{y}}(\boldsymbol{\eta}(\boldsymbol{\theta}))$ and by defining hyperparameters $\boldsymbol{\zeta} = (\boldsymbol{\nu}, \tau)$ and a base function $h_{\boldsymbol{\theta}}(\boldsymbol{\nu}, \tau)$.

\begin{definition}[Natural exponential family conjugate prior]
\label{def-exponential-family-conjugate}
Given the exponential family $p_{\mathbf{y}\vert\boldsymbol{\theta}}(\mathbf{y} \vert \boldsymbol{\theta})$ of Definition~\ref{def-exponential-family}, 
the natural exponential family conjugate prior to the density $p_{\mathbf{y}\vert\boldsymbol{\theta}}(\mathbf{y} \vert \boldsymbol{\theta})$ is
\begin{equation*}
\label{eq:exponential-family-conjugate}
p_{\boldsymbol{\theta}|\boldsymbol{\zeta}}(\boldsymbol{\theta}|\boldsymbol{\zeta}) = h_{\boldsymbol{\theta}}(\boldsymbol{\theta}) \exp \left(\boldsymbol{\zeta}^\top  \mathbf{T}_{\boldsymbol{\theta}}(\boldsymbol{\theta}) -  A_{\boldsymbol{\theta}}(\boldsymbol{\nu}, \tau) \right),
\end{equation*}
where $\boldsymbol{\zeta} = (\boldsymbol{\nu}, \tau)$, $\boldsymbol{\nu} \in \mathbb{R}^{k}$ and  $\tau \in \mathbb{R}$ are hyperparameters, $ \mathbf{T}_{\boldsymbol{\theta}}(\boldsymbol{\theta}) = (\boldsymbol{\eta}(\boldsymbol{\theta}), -A_{\mathbf{y}}(\boldsymbol{\eta}(\boldsymbol{\theta})))$,  and the density is taken on $(\boldsymbol{\theta}, \mathcal{B}(\boldsymbol{\theta}))$.
\end{definition}
The next proposition demonstrates that the joint density of independent variables, each following the same univariate exponential family distribution but with distinct parameters, also belongs to the exponential family. 
\begin{proposition}[Exponential Family Form of Independent Univariate Variables]
\label{prop:expfam_form_independent_variables}
    Let $y_j \in \mathcal{Y} \subseteq \mathbb{R}$ be part of the one-parameter univariate exponential family with parameter $\theta_j \in \Theta \subseteq \mathbb{R}$, natural parameter $\eta(\theta_j)$, base measure $h_{y}(y_j)$, sufficient statistics $T_{y}(y_j)$ and log-partition function $A_y(\eta(\theta_j))$ for $j = 1, \ldots, d$. Further, let $y_j|\theta_j$ be independent of $y_k|\theta_k$ for all $k \neq j$. Then, the joint density of $\boldsymbol{y} = (y_1, \ldots, y_d)$ conditional on $\boldsymbol{\theta} = (\theta_1, \ldots, \theta_d)$, $p_{\boldsymbol{y}\vert\boldsymbol{\theta}}(\boldsymbol{y} \vert \boldsymbol{\theta})$, is given by 
    \begin{equation*}
        \begin{aligned}
    &p_{\boldsymbol{y}\vert\boldsymbol{\theta}}(\boldsymbol{y} \vert \boldsymbol{\theta}) = h_{\boldsymbol{y}}(\boldsymbol{y})   \exp \left( \boldsymbol{\eta}(\boldsymbol{\theta})^{\top} \mathbf{T}_{\boldsymbol{y}}(\boldsymbol{y}) - \mathbf{1}_d^\top \mathbf{A}_{\boldsymbol{y}}(\boldsymbol{\eta}(\boldsymbol{\theta})) \right)\\
        &h_{\boldsymbol{y}}(\boldsymbol{y}) = \prod_{j = 1}^d h_{y}(y_j),\quad
        \boldsymbol{\eta}(\boldsymbol{\theta}) = (\eta(\theta_1), \dots, \eta(\theta_d)), \quad
        \mathbf{T}_{\boldsymbol{y}}(\boldsymbol{y}) = (T_y(y_1), \ldots, T_y(y_d)) \\
        &\mathbf{A}_{\boldsymbol{y}}(\boldsymbol{\eta}(\boldsymbol{\theta}))= (A_y(\eta(\theta_1)), \ldots, A_y(\eta(\theta_d)))
            \end{aligned}
    \end{equation*}
    where $\mathbf{1}_d$ is a vector of ones of dimension $d$.
\end{proposition}
The next proposition demonstrates that the joint density of $N$ multivariate independent variables following the same exponential family distribution also belongs to the exponential family. 
\begin{proposition}[Exponential Family Form of Independent Multivariate Variables]
\label{prop:expfam_form_independent_multivariate_variables}
    Let $\boldsymbol{y}_i \in \mathcal{Y}^d \subseteq \mathbb{R}^d$ be part of the multivariate exponential family with parameter $\boldsymbol{\theta} \in \Theta^d \subseteq \mathbb{R}^d$, natural parameter $\boldsymbol{\eta}(\boldsymbol{\theta})$, base measure $h_{\boldsymbol{y}}(\boldsymbol{y}_i)$, sufficient statistics $\mathbf{T}_{\boldsymbol{y}}(\boldsymbol{y}_i)$ and log-partition function $\mathbf{1}_d^\top \mathbf{A}_{\boldsymbol{y}}(\boldsymbol{\eta}(\boldsymbol{\theta}))$ for $i, \dots, N$. Further, let $\boldsymbol{y}_i|\boldsymbol{\theta}$ be independent of $\boldsymbol{y}_l|\boldsymbol{\theta}$ for all $i \neq l$. Then, the joint density of $\mathbf{y} = (\boldsymbol{y}_1, \ldots, \boldsymbol{y}_N)$ conditional on $\boldsymbol{\theta}$, $p_{\mathbf{y}\vert\boldsymbol{\theta}}(\mathbf{y} \vert \boldsymbol{\theta})$, is given by 
    \begin{equation*}
        \begin{aligned}
    &p_{\mathbf{y}\vert\boldsymbol{\theta}}(\mathbf{y} \vert \boldsymbol{\theta}) = h_{\mathbf{y}}(\mathbf{y})   \exp \left( \boldsymbol{\eta}(\boldsymbol{\theta})^{\top} \mathbf{T}_{\mathbf{y}}(\mathbf{y}) - N \mathbf{1}_d^\top \mathbf{A}_{\boldsymbol{y}}(\boldsymbol{\eta}(\boldsymbol{\theta})) \right)\\
        &h_{\mathbf{y}}(\mathbf{y}) = \prod_{i = 1}^N h_{\boldsymbol{y}}(\boldsymbol{y}_i),\quad
        \boldsymbol{\eta}(\boldsymbol{\theta}) = (\eta(\theta_1), \dots, \eta(\theta_d)), \quad
        \mathbf{T}_{\mathbf{y}}(\mathbf{y}) = \sum_{i=1}^N \mathbf{T}_{\boldsymbol{y}}(\boldsymbol{y}_i) 
            \end{aligned}
    \end{equation*}
    where $\mathbf{1}_d$ is a vector of ones of dimension $d$.
\end{proposition}
The next proposition shows that the conjugate prior for the parameters of independent variables, each following the same univariate exponential family distribution but with distinct parameters, is of the natural exponential family form.
\begin{proposition}[Natural Exponential Family Conjugate Prior for Independent Univariate Parameters]
\label{prop:expfam_form_independent_parameters}
    Let $y_j \in \mathcal{Y} \subseteq \mathbb{R}$ be part of the one-parameter univariate exponential family with parameter $\theta_j \in \Theta \subseteq \mathbb{R}$, natural parameter $\eta(\theta_j)$, base measure $h_{y}(y_j)$, sufficient statistics $T_{y}(y_j)$ and log-partition function $A_y(\eta(\theta_j))$ for $j = 1, \ldots, d$. 
    Let the natural exponential family conjugate prior of $\theta_j$ with hyperparameters $\boldsymbol{\zeta}_j = (\nu_j, \tau_j), \nu_j, \tau_j \in \mathbb{R}$ be given by, 
    \begin{align*} 
        p_{{\theta}|\boldsymbol{\zeta}}({\theta}_j|\boldsymbol{\zeta}_j) = h_{{\theta}}(\theta_j) \exp \left(\boldsymbol{\zeta}_j^\top \mathbf{T}_{\theta}(\theta_j) - A_{{\theta}}({\nu}_j,  {\tau}_j)) \right).
    \end{align*}
    where $ \mathbf{T}_{\theta}(\theta_j)= ({\eta}({\theta}_j), -{A}_{{y}}({\eta}({\theta}_j))$.
    Further, let $\theta_j|\zeta_j$ be independent of $\theta_k|\zeta_k$ for all $k \neq j$. Then the joint natural exponential family conjugate prior of $\boldsymbol{\theta}= (\theta_1, \ldots, \theta_d)$ conditional on $\boldsymbol{\zeta} = (\boldsymbol{\nu}, \boldsymbol{\tau})$, where $\boldsymbol{\nu} = (\nu_1, \ldots,\nu_d)$ and $\boldsymbol{\tau} = (\tau_1, \ldots,\tau_d)$, is given by, 
    \begin{equation*}
        \begin{aligned}
        &p_{\boldsymbol{\theta}|\boldsymbol{\zeta}}(\boldsymbol{\theta}|\boldsymbol{\zeta}) = h_{\boldsymbol{\theta}}(\boldsymbol{\theta}) \exp \left(
        \boldsymbol{\zeta}^T \mathbf{T}_{\boldsymbol{\theta}}(\boldsymbol{\theta}) - A_{\boldsymbol{\theta}}(\boldsymbol{\nu},  \boldsymbol{\tau})  \right) \\
        &h_{\boldsymbol{\theta}}(\boldsymbol{\theta}) = \prod_{j = 1}^d h_{{\theta}}({\theta}_j), \quad
        \mathbf{T}_{\boldsymbol{\theta}}(\boldsymbol{\theta}) = (\boldsymbol{\eta}(\boldsymbol{\theta}), -\mathbf{A}_{\boldsymbol{y}}(\boldsymbol{\eta}(\boldsymbol{\theta}))), \\
        &\boldsymbol{\eta}(\boldsymbol{\theta}) = (\eta(\theta_1), \dots, \eta(\theta_d)), \quad \mathbf{A}_{\boldsymbol{y}}(\boldsymbol{\eta}(\boldsymbol{\theta}))= (A_y(\eta(\theta_1)), \ldots, A_y(\eta(\theta_d))), \\
        &A_{\boldsymbol{\theta}}(\boldsymbol{\nu},  \boldsymbol{\tau}) = \sum_{j = 1}^d A_{{\theta}}({\nu}_j,  {\tau}_j),
            \end{aligned}
    \end{equation*} 
\end{proposition}
When the exponential family $p_{\mathbf{y}\vert\boldsymbol{\theta}}(\mathbf{y} \vert \boldsymbol{\theta})$ is a likelihood function and the family $p_{\boldsymbol{\theta}|\boldsymbol{\zeta}}(\boldsymbol{\theta}|\boldsymbol{\zeta})$ is used as a prior, the pair exhibits a convenient conjugacy property, formalized in the following proposition.
\begin{proposition}[Conjugacy]  \label{prop:conjugacy}
Let the densities $p_{\mathbf{y}\vert\boldsymbol{\theta}}(\mathbf{y} \vert \boldsymbol{\theta})$ and $p_{\boldsymbol{\theta}|\boldsymbol{\zeta}}(\boldsymbol{\theta}|\boldsymbol{\zeta})$ be defined as in Proposition~\ref{prop:expfam_form_independent_multivariate_variables} and~\ref{prop:expfam_form_independent_parameters}, respectively. We have the relations
\begin{equation*}
\begin{aligned}
p_{\boldsymbol{\theta}, \mathbf{y}\vert\boldsymbol{\zeta}}(\boldsymbol{\theta}, \mathbf{y}\vert\boldsymbol{\zeta}) &= h_{\mathbf{y}}(\mathbf{y})h_{\boldsymbol{\theta}}(\boldsymbol{\theta})  \exp\left(- A_{\boldsymbol{\theta}}(\boldsymbol{\nu}, \boldsymbol{\tau}) \right) \exp \left( \boldsymbol{\eta}(\boldsymbol{\theta})^{\top} \left(\boldsymbol{\nu}+\mathbf{T}_{\mathbf{y}}(\mathbf{y}) \right) - \left(\boldsymbol{\tau} 
 + N \mathbf{1}_d\right)^\top A_{\boldsymbol{y}}(\boldsymbol{\eta}(\boldsymbol{\theta})) \right) \\
p_{\boldsymbol{\theta} \vert \mathbf{y}, \boldsymbol{\zeta}}(\boldsymbol{\theta} \vert \mathbf{y}, \boldsymbol{\zeta}) &= 
h_{\boldsymbol{\theta}}(\boldsymbol{\theta})  \exp\left(-A_{\boldsymbol{\theta}}(\boldsymbol{\nu}+\mathbf{T}_{\mathbf{y}}(\mathbf{y}), \boldsymbol{\tau} + N \mathbf{1}_d) \right) \exp \left( \boldsymbol{\eta}(\boldsymbol{\theta})^{\top} \left(\boldsymbol{\nu}+\mathbf{T}_{\mathbf{y}}(\mathbf{y})\right) - \left(\boldsymbol{\tau} 
 + N\mathbf{1}_d\right)^\top A_{\boldsymbol{y}}(\boldsymbol{\eta}(\boldsymbol{\theta})) \right) \\
p_{\mathbf{y}\vert\boldsymbol{\zeta}}(\mathbf{y}\vert\boldsymbol{\zeta}) &= h_{\mathbf{y}}(\mathbf{y}) \frac{\exp(-A_{\boldsymbol{\theta}}(\boldsymbol{\nu}, \boldsymbol{\tau}))}{\exp(-A_{\boldsymbol{\theta}}(\boldsymbol{\nu}+\mathbf{T}_{\mathbf{y}}(\mathbf{y}) , \boldsymbol{\tau} + N \mathbf{1}_d))}
\end{aligned}
\end{equation*}
and hence in particular the posterior $p_{\boldsymbol{\theta} \vert \mathbf{y}}(\boldsymbol{\theta} \vert \mathbf{y})$ is in the same exponential family as $p_{\boldsymbol{\theta}|\boldsymbol{\zeta}}(\boldsymbol{\theta}|\boldsymbol{\zeta})$ with the parameters $\boldsymbol{\nu}+\mathbf{T}_{\mathbf{y}}(\mathbf{y}) , \boldsymbol{\tau} + N \mathbf{1}_d$.
\end{proposition}



\clearpage
\newpage
\section{Dealing with an Observation Operator $\mathbf{H}$}
\label{app-observation-operator-H}
Our methodology can be extended to accommodate the presence of an observation matrix $ \mathbf{H} \in \mathbb{R}^{d_y \times d_x} $ with $ d_y \leq d_x $. When $ \mathbf{H} $ is not full-rank, a standard inverse does not exist. To address this, we employ the Moore-Penrose pseudoinverse, defined as:
\begin{equation*}
\mathbf{H}^\dagger := (\mathbf{H}^\top \mathbf{H})^{-1} \mathbf{H}^\top,
\end{equation*}
which is well-defined for any $ \mathbf{H} $ with $ d_y \leq d_x $.  Given the original link function $ g(\cdot) $, as defined in Section~\ref{sec-sampling-link-function}, we introduce a modified link function $ g_{\mathbf{H}}: \mathbb{R}^{d_y} \to \mathbb{R}^{d_x} $, defined as:
\begin{equation*}
g_{\mathbf{H}}(\boldsymbol{\theta}) := \mathbf{H}^{\dagger} g(\boldsymbol{\theta}),
\end{equation*}
where $ \mathbf{H}^\dagger $ ensures a consistent mapping even when $ \mathbf{H} $ is not full-rank. In this framework, the inverse of $ g_{\mathbf{H}} $ is naturally given by:
\begin{equation*}
g_{\mathbf{H}}^{-1}(\mathbf{x}) = g^{-1}(\mathbf{H} \mathbf{x}),
\end{equation*}
where $ g_{\mathbf{H}}^{-1} $ maps $ \mathbf{x}$ into the observation space defined by $ \mathbf{H}$.
All results presented in Section~\ref{sec:method} remain valid under this modification, provided the link function $ g $ is replaced with the modified link function $ g_{\mathbf{H}} $.

Finally, we note that the use of the Moore-Penrose pseudoinverse is analogous to the standard approach employed in many overdetermined statistical problems (see, for example, \citet[Section 3.1.1]{bishop_pattern_recognition}).

\clearpage
\newpage
\section{Likelihood and Prior Formulations} \label{app-table_distributions}
In this appendix, for each likelihood in the one-parameter exponential family distributions we obtain the corresponding natural conjugate prior. The results of this appendix are also summarized in Table~\ref{tab:data-distributions-details}.

Let the likelihood of ${y}_{i,j}$ conditioned on parameters $\theta_j$ be given by
\begin{equation*}
p_{{y}\vert{\theta}}({y}_{i,j} \vert {\theta}_j) = h_{{y}}({y}_{i,j})   \exp \left( {\eta}({\theta}_j)^{\top} {T}_{{y}}({y}_{i,j}) - {A}_{{y}}({\eta}({\theta}_j)) \right),
\end{equation*}
for all $i = 1, \ldots, N$ and $j = 1, \dots, d$. Notice that the likelihood of $\mathbf{y} = \{y_{i,j}\}_{i = 1, \ldots, N;\: j = 1, \ldots, d}$ conditioned on the parameters $\boldsymbol{\theta} = (\theta_1, \ldots, \theta_d)$ can be found using Proposition~\ref{prop:expfam_form_independent_variables} and~\ref{prop:expfam_form_independent_multivariate_variables}. 

Further, let the natural conjugate prior of $\theta_j$ conditional on hyperparameters $\boldsymbol{\zeta}_j = (\nu_j, \tau_j)$ be given by, 
\begin{align*} 
    p_{{\theta}|\boldsymbol{\zeta}}({\theta}_j|\boldsymbol{\zeta}_j) = h_{{\theta}}(\theta_j) \exp \left(\boldsymbol{\zeta}_j^\top \mathbf{T}_{\theta}(\theta_j) - A_{{\theta}}({\nu}_j,  {\tau}_j)) \right).
\end{align*}
where $ \mathbf{T}_{\theta}(\theta_j)= ({\eta}({\theta}_j), -{A}_{{y}}({\eta}({\theta}_j))$. The form of the natural conjugate prior of  $\boldsymbol{\theta} = (\theta_1, \ldots, \theta_d)$ conditioned on hyperparameters $\boldsymbol{\zeta} = (\boldsymbol{\nu}, \boldsymbol{\tau})$  for $\boldsymbol{\nu} = (\nu_1, \ldots,\nu_d)$ and $\boldsymbol{\tau} = (\tau_1, \ldots,\tau_d)$ can be found using Proposition~\ref{prop:expfam_form_independent_parameters}.

%
%

\subsection{Likelihood Distribution with Gaussian Conjugate Prior}

\paragraph{Observations following a Normal distribution with fixed variance.}
For $p_{{y}\vert{\theta}}({y}_{i,j} \vert {\theta}_j) = \mathcal{N}({y}_{i,j};{\theta}_j, \sigma^2)$ with known variance $\sigma^2$ and with mean $\theta_j$, the likelihood can be expressed as an exponential family distribution with
\begin{equation*}
h_{y}({y}_{i,j} ) = \frac{1}{\sqrt{2\pi}\sigma} \exp\left(-\frac{{y}_{i,j} ^{2}}{2\sigma^{2}}\right),\quad
\eta(\theta_j) = \frac{\theta_j}{\sigma^{2}},\quad
T_{y}({y}_{i,j} ) = {y}_{i,j}  ,\quad
A_{y}(\eta(\theta_j)) = \frac{\theta_j^2}{2\sigma^{2}}
\end{equation*}
Let $\boldsymbol{\Sigma} = \mathbf{I}_{d} \sigma^2$, where $\mathbf{I}_{d}$ is the identity matrix of dimension $d\times d$. For the discussion that follows it is convenient to consider the joint distribution $p_{\boldsymbol{y}|\boldsymbol{\theta}}(\boldsymbol{y}_i|\boldsymbol{\theta}) = \mathcal{N}_d(\boldsymbol{\theta}, \boldsymbol{\Sigma})$ for all $i = 1, \ldots, N$. The multivariate prior $p_{\boldsymbol{\theta}}(\boldsymbol{\theta}\vert \boldsymbol{\zeta})$ conditional on $\boldsymbol{\zeta} = (\boldsymbol{\nu}, \boldsymbol{\tau})$, where $\boldsymbol{\nu} \in \mathbb{R}^d$ and $\boldsymbol{\tau} \in \mathbb{R}^{d}$, which is conjugate to $p_{\boldsymbol{y}|\boldsymbol{\theta}}(\boldsymbol{y}_i|\boldsymbol{\theta})$ is of the form
\begin{equation} \label{eq:conjugate_prior_normal_fixed_variance}
        \begin{aligned}
        &p_{\boldsymbol{\theta}|\boldsymbol{\zeta}}(\boldsymbol{\theta}|\boldsymbol{\zeta}) = h_{\boldsymbol{\theta}}(\boldsymbol{\theta}) \exp \left(
        \boldsymbol{\nu}^\top \boldsymbol{\eta}(\boldsymbol{\theta}) - \frac{1}{2} \left(\boldsymbol{\tau} \boldsymbol{\theta}\right)^\top  \boldsymbol{\Sigma}^{-1} \boldsymbol{\theta} - A_{\boldsymbol{\theta}}(\boldsymbol{\nu}, \boldsymbol{\tau})  \right) \\
        &\boldsymbol{\eta}(\boldsymbol{\theta}) 
    = \boldsymbol{\Sigma}^{-1}\boldsymbol{\theta},\quad
        h_{\boldsymbol{\theta}}(\boldsymbol{\theta}) = (2\pi)^{-\frac{d}{2}}, \\
        &A_{\boldsymbol{\theta}}(\boldsymbol{\nu}, \boldsymbol{\tau}) = \frac{1}{2} \left(\boldsymbol{\tau}^{-1} \boldsymbol{\nu}\right)^\top \left( \boldsymbol{\Sigma}^{-1} \boldsymbol{\tau}\right) \left(\boldsymbol{\tau}^{-1} \boldsymbol{\nu}\right) - \frac{1}{2} \log \text{det}\left(\boldsymbol{\Sigma}^{-1} \boldsymbol{\tau} \right).
            \end{aligned}
\end{equation}
Notice that $p_{\boldsymbol{\theta}|\boldsymbol{\zeta}}(\boldsymbol{\theta}|\boldsymbol{\zeta}) = \mathcal{N}_d(\boldsymbol{\theta}; \boldsymbol{\mu}_0, \boldsymbol{\Sigma}_0)$ where $(\boldsymbol{\nu},\boldsymbol{\tau})$ and $(\boldsymbol{\mu}_0,  \boldsymbol{\Sigma}_0)$ are related through $\boldsymbol{\nu} = \boldsymbol{\Sigma} \boldsymbol{\Sigma}_0^{-1} \boldsymbol{\mu}_0$ and $\boldsymbol{\tau} = \boldsymbol{\Sigma} \boldsymbol{\Sigma}_0^{-1} $.  It follows from the fact that $\boldsymbol{\Sigma}$ is a diagonal matrix, that also  $\boldsymbol{\Sigma}_{0}$ is diagonal and the components of $\boldsymbol{\theta}$ are independent. We remark that it is possible to construct a prior for $p_{\boldsymbol{y}|\boldsymbol{\theta}}(\boldsymbol{y}_i|\boldsymbol{\theta})$ whose covariance matrix is not diagonal. This can be accomplished by utilizing the same form as in \eqref{eq:conjugate_prior_normal_fixed_variance}, but with $\boldsymbol{\tau}\in\mathbb{R}^{d\times d}$.  We remark the similitude from the natural conjugate prior by noticing that $\mathbf{A}_{\boldsymbol{y}}( \boldsymbol{\eta}(\boldsymbol{\theta}) ) = - \frac{1}{2} \boldsymbol{\theta}^\top \boldsymbol{\Sigma}^{-1} \boldsymbol{\theta}$. 

\paragraph{Observations following a Log-Normal distribution with fixed variance.}
For $p_{{y}\vert{\theta}}({y}_{i,j} \vert {\theta}_j) = \text{Log-Normal}({y}_{i,j};{\theta}_j, \sigma^2)$ with known logarithm of scale $\sigma$ and with logarithm of location $\theta_j$, the likelihood can be expressed as an exponential family distribution with
\begin{equation*}
h_{y}({y}_{i,j} ) = \frac{1}{\sqrt{2\pi}\sigma{y}_{i,j}} \exp\left(-\frac{\log({y}_{i,j})^{2}}{2\sigma^{2}}\right),\quad
\eta(\theta_j) = \frac{\theta_j}{\sigma^{2}},\quad
T_{y}({y}_{i,j} ) = \log {y}_{i,j}  ,\quad
A_{y}(\eta(\theta_j)) = \frac{\theta_j^2}{2\sigma^{2}}
\end{equation*}
The multivariate conjugate prior $p_{\boldsymbol{\theta}}(\boldsymbol{\theta}\vert \boldsymbol{\zeta})$ conditional on $\boldsymbol{\zeta} = (\boldsymbol{\nu}, \boldsymbol{\tau})$ where $\boldsymbol{\nu} \in \mathbb{R}^d$ and $\boldsymbol{\tau} \in \mathbb{R}^{d\times d}$ is the same as that of the Normal distribution case presented in~\eqref{eq:conjugate_prior_normal_fixed_variance}.

%
%

\subsection{Likelihood Distribution with Gamma Conjugate Prior}
\paragraph{Observations following a Poisson distribution.} 
For $p_{{y}\vert{\theta}}({y}_{i,j} \vert {\theta}_j) = \text{Poisson}({y}_{i,j};{\theta}_j)$ with rate $\theta_j$, the likelihood can be expressed as an exponential family distribution with
\begin{align*}
&h_{{y}}({y}_{i,j}) =  \frac{1}{y_{i,j}!},\quad
{\eta}({\theta}_j) = \log(\theta_j),\quad {T}_{{y}}({y}_{i,j}) = y_{i,j}, \quad
{A}_{{y}}({\eta}({\theta}_j)) = \theta_j.
\end{align*}
The natural conjugate prior for ${\theta}_j$ conditioned on the hyperparameters $\boldsymbol{\zeta}_j = (\nu_j, \tau_j)$ is an exponential family distribution with
\begin{equation*} 
h_{{\theta}}({\theta}_j) = 1, \quad A_{\theta}({\nu}_j, {\tau}_j) = \log \Gamma(\nu_j + 1) - (\nu_j+1)\log(\tau_j).
\end{equation*}
We note that $p_{\theta\vert \boldsymbol{\zeta}}(\theta_j\vert \boldsymbol{\zeta}_j)= \text{Gamma}(\theta_j;\alpha_j,\beta_j)$ for $\nu_j=\alpha_j-1$ and $\tau_j=\beta_j$.

\paragraph{Observations following an Exponential distribution.}
For $p_{{y}\vert{\theta}}({y}_{i,j} \vert {\theta}_j) = \text{Exponential}({y}_{i,j};{\theta}_j)$ with rate $\theta_j$, the likelihood can be expressed as exponential family distribution with
\begin{equation*}
h_y({y}_{i,j}) = 1,\quad \eta(\theta_j) = -\theta_j,\quad T_y({y}_{i,j}) = {y}_{i,j},\quad {A}_{{y}}({\eta}({\theta}_j)) = - \log \theta_j.
\end{equation*}
The natural conjugate prior for ${\theta}_j$ conditioned on the hyperparameters $\boldsymbol{\zeta}_j = (\nu_j, \tau_j)$ is an exponential family distribution with 
\begin{equation*}
h_{{\theta}}({\theta}_j) = 1, \quad A_{\theta}({\nu}_j, {\tau}_j) = \log \Gamma(\tau_j + 1) - (\tau_j+1)\log(\nu_j).
\end{equation*}
We note that $p_{\theta\vert \boldsymbol{\zeta}}(\theta_j\vert \boldsymbol{\zeta}_j)= \text{Gamma}(\theta_j;\alpha_j,\beta_j)$ for $\nu_j=\beta_j$ and $\tau_j= \alpha_j-1$.

\paragraph{Observations following a Gamma distribution with fixed shape.}
For $p_{{y}\vert{\theta}}({y}_{i,j} \vert {\theta}_j) = \text{Gamma}({y}_{i,j};a, {\theta}_j)$ with known shape $a$ and with rate ${\theta}_j$, the likelihood can be expressed as an exponential family distribution with
\begin{equation*}
h_y({y}_{i,j}) = \frac{1}{\Gamma(a)}{y}_{i,j}^{a-1} ,\quad \eta(\theta_j) = -\theta_j ,\quad T_y({y}_{i,j}) = {y}_{i,j},\quad {A}_{{y}}({\eta}({\theta}_j)) =  - a \log(\theta_j).
\end{equation*}
The natural conjugate prior for ${\theta}_j$ conditioned on hyperparameters $\boldsymbol{\zeta}_j = (\nu_j, \tau_j)$ is an exponential family distribution with 
\begin{equation*} 
h_{{\theta}}({\theta}_j) = 1, \quad A_{\theta}({\nu}_j, {\tau}_j) = \log \Gamma(\tau_j \: a + 1) - (\tau_j \: a+1)\log(\nu_j).
\end{equation*}
We note that $p_{\theta\vert \boldsymbol{\zeta}}(\theta_j\vert \boldsymbol{\zeta}_j)= \text{Gamma}(\theta_j;\alpha_j,\beta_j)$ for $\nu_j=\beta_j$ and $\tau_j= (\alpha_j-1) / a$.

\paragraph{Observations following a Pareto distribution with fixed scale.}
For $p_{{y}\vert{\theta}}({y}_{i,j} \vert {\theta}_j) = \text{Pareto}({y}_{i,j};x_m, {\theta}_j)$ with known scale $x_m$ and with shape ${\theta}_j$, the likelihood can be expressed as an exponential family distribution with
\begin{equation*}
h_y({y}_{i,j}) = 1 ,\quad \eta(\theta_j) = -\theta_j -1,\quad T_y({y}_{i,j}) = \log({y}_{i,j}),\quad {A}_{{y}}({\eta}({\theta}_j)) =  - \log(\theta_j) - \theta_j \log(x_m) .
\end{equation*}
The natural conjugate prior for ${\theta}_j$ conditioned on the hyperparameters $\boldsymbol{\zeta}_j = (\nu_j, \tau_j)$ is an exponential family distribution with 
\begin{equation*} 
h_{{\theta}}({\theta}_j) = 1, \quad A_{\theta}({\nu}_j, {\tau}_j) = \log \Gamma(\tau_j + 1) - \nu_j  - (\tau_j +1)\log(\nu_j - \tau_j \log(x_m)).
\end{equation*}
We note that $p_{\theta\vert \boldsymbol{\zeta}}(\theta_j\vert \boldsymbol{\zeta}_j)= \text{Gamma}(\theta_j;\alpha_j,\beta_j)$ for $\nu_j=(\alpha_j-1)\log(x_m) + \beta_j$ and $\tau_j=  \alpha_j-1$.

%
%

\subsection{Likelihood Distribution with Beta Conjugate Prior}
\paragraph{Observations following a Binomial or Bernoulli distribution.} 
For $p_{{y}\vert{\theta}}({y}_{i,j} \vert {\theta}_j) = \text{Binomial}({y}_{i,j};n, {\theta}_j)$ with known number of trials $n$ and with success probability $\theta_j$, the likelihood can be expressed as an exponential family distribution with
\begin{equation*}
h_y({y}_{i,j}) =  \binom{n}{{y}_{i,j}}  ,\quad \eta(\theta_j) = \log \frac{\theta_j}{1-\theta_j} ,\quad T_y({y}_{i,j}) = {y}_{i,j},\quad {A}_{{y}}({\eta}({\theta}_j)) =  - n \log(1-\theta_j).
\end{equation*}
The Bernoulli distribution has the same components with $n=1$.
The natural conjugate prior for ${\theta}_j$ conditioned on the hyperparameters $\boldsymbol{\zeta}_j = (\nu_j, \tau_j)$ is an exponential family distribution with 
\begin{equation*} 
h_{{\theta}}({\theta}_j) = 1, \quad A_{\theta}({\nu}_j, {\tau}_j) = \log \Gamma(\nu_j  + 1)+ \log \Gamma(\tau_j \: n - \nu_j + 1) - \log \Gamma(\tau_j \: n + 2) .
\end{equation*}
We note that $p_{\theta\vert \boldsymbol{\zeta}}(\theta_j\vert \boldsymbol{\zeta}_j)= \text{Beta}(\theta_j;\alpha_j,\beta_j)$ for $\nu_j=\alpha_j-1$ and $\tau_j= (\alpha_j + \beta_j -2) / n$.

\paragraph{Observations following a Negative-Binomial distribution.} 
For $p_{{y}\vert{\theta}}({y}_{i,j} \vert {\theta}_j) = \text{Negative Binomial}({y}_{i,j};r, {\theta}_j)$ with known number of successes $r$ and with success probability $\theta_j$, the likelihood of ${y}_{i,j}$ failures can be expressed as an exponential family distribution with
\begin{equation*}
h_y({y}_{i,j}) =  \binom{{y}_{i,j} + r -1}{{y}_{i,j}}  ,\quad \eta(\theta_j) = \log (1-\theta_j) ,\quad T_y({y}_{i,j}) = {y}_{i,j},\quad {A}_{{y}}({\eta}({\theta}_j)) =  - r \log(\theta_j).
\end{equation*}
The natural conjugate prior for ${\theta}_j$ conditioned on the hyperparameters $\boldsymbol{\zeta}_j = (\nu_j, \tau_j)$ is an exponential family distribution with
\begin{equation*} 
h_{{\theta}}({\theta}_j) = 1, \quad A_{\theta}({\nu}_j, {\tau}_j) = \log \Gamma(\tau_j \: r + 1) + \log \Gamma(\nu_j  + 1) - \log \Gamma(\tau_j \: r + \nu_j + 2) .
\end{equation*}
We note that $p_{\theta\vert \boldsymbol{\zeta}}(\theta_j\vert \boldsymbol{\zeta}_j)= \text{Beta}(\theta_j;\alpha_j,\beta_j)$ for $\nu_j=\beta_j-1$ and $\tau_j= (\alpha_j -1) / r$.

\paragraph{Observations following a Geometric distribution.} 
For $p_{{y}\vert{\theta}}({y}_{i,j} \vert {\theta}_j) = \text{Geometric}({y}_{i,j};{\theta}_j)$ with success probability $\theta_j$, the likelihood of ${y}_{i,j}$ failures can be expressed as an exponential family distribution with
\begin{equation*}
h_y({y}_{i,j}) = 1  ,\quad \eta(\theta_j) = \log (1-\theta_j) ,\quad T_y({y}_{i,j}) = {y}_{i,j},\quad {A}_{{y}}({\eta}({\theta}_j)) =  -  \log(\theta_j).
\end{equation*}
The natural conjugate prior for ${\theta}_j$ conditioned on the hyperparameters $\boldsymbol{\zeta}_j = (\nu_j, \tau_j)$ is an exponential family distribution with
\begin{equation*} 
h_{{\theta}}({\theta}_j) = 1, \quad A_{\theta}({\nu}_j, {\tau}_j) = \log \Gamma(\tau_j  + 1) + \log \Gamma(\nu_j  + 1) - \log \Gamma(\tau_j  + \nu_j + 2) .
\end{equation*}
We note that $p_{\theta\vert \boldsymbol{\zeta}}(\theta_j\vert \boldsymbol{\zeta}_j)= \text{Beta}(\theta_j;\alpha_j,\beta_j)$ for $\nu_j=\beta_j-1$ and $\tau_j= \alpha_j -1$.

%
%

\subsection{Likelihood Distribution with Inverse Gamma Conjugate Prior}
\paragraph{Observations following a Normal distribution with fixed mean.} 
For $p_{{y}\vert{\theta}}({y}_{i,j} \vert {\theta}_j) = \mathcal{N}({y}_{i,j};\mu, {\theta}_j)$ with known mean $\mu$ and with variance ${\theta}_j$, the likelihood can be expressed as an exponential family distribution with
\begin{align*}
&h_{{y}}({y}_{i,j}) =  \frac{1}{\sqrt{2\pi}},\quad
{\eta}({\theta}_j) = \frac{1}{\theta_j},\quad {T}_{{y}}({y}_{i,j}) = \left(-\frac{y_{i,j}^2}{2} + \mu \: y_{i,j}\right), \quad
{A}_{{y}}({\eta}({\theta}_j)) = \frac{\mu^2}{2\theta_j} - \frac{1}{2}\log\left(\frac{1}{\theta_j}\right).
\end{align*}
The natural conjugate prior for ${\theta}_j$ conditioned on the hyperparameters $\boldsymbol{\zeta}_j = (\nu_j, \tau_j)$ is an exponential family distribution with
\begin{equation}
\label{eq:conjugate_prior_normal_fixed_mean}
h_{{\theta}}({\theta}_j) = 1, \quad A_{\theta}({\nu}_j, {\tau}_j) = \log \Gamma\left(\frac{\tau_j}{2} - 1\right) - \left(\frac{\tau_j}{2} - 1\right)\log\left(\frac{\tau_j\:\mu^2}{2} - \nu_j\right).
\end{equation}
We note that $p_{\theta\vert \boldsymbol{\zeta}}(\theta_j\vert \boldsymbol{\zeta}_j)= \text{Inverse-Gamma}(\theta_j;\alpha_j,\beta_j)$ for $\nu_j=\mu^2(\alpha_j+1) - \beta_j$ and $\tau_j=2(\alpha_j+1)$.

\paragraph{Observations following a Log-Normal distribution with fixed mean.} 
For $p_{{y}\vert{\theta_j}}({y}_{i,j} \vert {\theta}_j) = \text{Log-Normal}({y}_{i,j};\mu, {\theta}_j)$ with known logarithm of location $\mu$ and with logarithm of scale $\sqrt{\theta}_j$, the likelihood can be expressed as an exponential family distribution with
\begin{align*}
&h_{{y}}({y}_{i,j}) =  \frac{1}{\sqrt{2\pi} {y}_{i,j}},\quad
{\eta}({\theta}_j) = \frac{1}{\theta}_j,\quad {T}_{{y}}({y}_{i,j}) = \left(-\frac{\log(y_{i,j})^2}{2} + \mu \log(y_{i,j})\right), \quad
{A}_{{y}}({\eta}({\theta}_j)) = \frac{\mu^2}{2\theta_j} - \frac{1}{2}\log\left(\frac{1}{\theta_j}\right).
\end{align*}
The natural conjugate prior of ${\theta}_j$ conditioned on hyperparameters $\boldsymbol{\zeta}_j = (\nu_j, \tau_j)$ has the same form as that of the Normal distribution presented in~\eqref{eq:conjugate_prior_normal_fixed_mean}.

\paragraph{Observations following a Weibull distribution with fixed shape.} 
For $p_{{y}\vert{\theta}}({y}_{i,j} \vert {\theta}_j) = \text{Weibull}({y}_{i,j};{\theta}_j^{1/k}, k)$ with known shape $k$ and with scale parameter ${\theta}_j^{1/k}$, the likelihood can be expressed as an exponential family distribution with
\begin{align*}
&h_{{y}}({y}_{i,j}) = k \: y_{i,j}^{k-1} \quad
{\eta}({\theta}_j) = -\frac{1}{\theta_j},\quad {T}_{{y}}({y}_{i,j}) = y_{i,j}^k, \quad
{A}_{{y}}({\eta}({\theta}_j)) = \log(\theta_j).
\end{align*}
The natural conjugate prior for ${\theta}_j$ conditioned on the hyperparameters $\boldsymbol{\zeta}_j = (\nu_j, \tau_j)$ is an exponential family distirbution with
\begin{equation*}
h_{{\theta}}({\theta}_j) = 1, \quad A_{\theta}({\nu}_j, {\tau}_j) = \log \Gamma\left(\tau_j - 1\right) - \left(\tau_j - 1\right)\log\left(\nu_j  \right).
\end{equation*}
We note that $p_{\theta\vert \boldsymbol{\zeta}}(\theta_j\vert \boldsymbol{\zeta}_j)= \text{Inverse-Gamma}(\theta_j;\alpha_j,\beta_j)$ for $\nu_j=\beta_j$ and $\tau_j=\alpha_j+1$.

\clearpage
\newpage
\begin{landscape}
\begin{table}[!hp]
\caption{\textbf{Summary of Statistical Components.} The Likelihood and Prior forms are detailed in Appendix~\ref{app-table_distributions}, while the formulations of the proposed inverse link functions are described in Appendix~\ref{app-proposed_link_distributions}.}
\begin{threeparttable}
    \resizebox*{!}{0.9\textwidth}{%
    \centering
    \begin{tabular}{l|lll}
      \toprule
      \midrule
        & \textbf{Likelihood Distribution} 
        & \textbf{Conjugate prior distribution}
        & \textbf{Inverse link function} \\ 
        & $p_{\boldsymbol{y}|\boldsymbol{\theta}}(\boldsymbol{y}_i|\boldsymbol{\theta}) $ 
        & $q_{\boldsymbol{\theta}|\boldsymbol{\zeta}(\mathbf{x}_{t})}(\boldsymbol{\theta}|\boldsymbol{\nu}(\mathbf{x}_{t}), \boldsymbol{\tau}(\mathbf{x}_{t}))$ 
        & $\boldsymbol{\theta} = g^{-1}(\mathbf{x}_{0}) $ \\
        \midrule
        Normal distribution 
        & $\mathcal{N}_d( \boldsymbol{y}_i; \boldsymbol{\theta}, \boldsymbol{\Sigma})$, 
        & $\mathcal{N}_d(\boldsymbol{\theta};\boldsymbol{\mu}_0(\mathbf{x}_{t}), \boldsymbol{\Sigma}_0(\mathbf{x}_{t}))$  
        & $g^{-1}(\mathbf{x}_0) = \mathbf{x}_0$ \\ 
        with known variance $\sigma^2$
        & $\boldsymbol{\Sigma} = \sigma^2 \mathbf{I}_d$ 
        & $\boldsymbol{\mu}_0(\mathbf{x}_{t}) = \boldsymbol{\tau}(\mathbf{x}_{t})^{-1} \boldsymbol{\nu}(\mathbf{x}_{t}) $, $ \boldsymbol{\Sigma}_0= \boldsymbol{\tau}(\mathbf{x}_{t})^{-1} \boldsymbol{\Sigma}  $ 
        &   
        \\
        \midrule
        Log-Normal distribution 
        & $\text{Log-Normal}_d(\boldsymbol{y}_i; \boldsymbol{\theta}, \boldsymbol{\Sigma})$, 
        & Same as Normal distribution.  
        & Same as Normal distribution. \\ 
        with known $\sigma^2$
        & $\boldsymbol{\Sigma} = \sigma^2 \mathbf{I}_d$ 
        & 
        &  \\
        \midrule
        \midrule
        Poisson distribution 
        & $ \prod_{j = 1}^d \text{Poisson}(y_{i,j};\theta_j)$ 
        & $\prod_{j = 1}^d \text{Gamma}(\theta_j;\alpha_j(\mathbf{x}_{t}), \beta_j(\mathbf{x}_{t})) $
        & $g^{-1}(\mathbf{x}_{0}) = \exp(\mathbf{x}_{0})$
        \\
        & 
        & $\alpha_j(\mathbf{x}_{t}) = \nu_j(\mathbf{x}_{t}) + 1, \beta_j(\mathbf{x}_{t}) = \tau_j(\mathbf{x}_{t})$
        & 
        \\
        \midrule
        Exponential distribution 
        & $ \prod_{j = 1}^d \text{Exponential}(y_{i,j};\theta_j)$ 
        & $\prod_{j = 1}^d \text{Gamma}(\theta_j;\alpha_j(\mathbf{x}_{t}), \beta_j(\mathbf{x}_{t})) $
        & $g^{-1}(\mathbf{x}_{0}) = 1/\exp(\mathbf{x}_{0})$\\
        &
        & $\alpha_j(\mathbf{x}_{t}) = \tau_j(\mathbf{x}_{t}) + 1, \beta_j(\mathbf{x}_{t}) = \nu_j(\mathbf{x}_{t})$
        & 
        \\
        \midrule
        Gamma distribution 
        & $ \prod_{j = 1}^d \text{Gamma}(y_{i,j};a, \theta_j)$ 
        & $\prod_{j = 1}^d \text{Gamma}(\theta_j;\alpha_j(\mathbf{x}_{t}), \beta_j(\mathbf{x}_{t})) $
        & $g^{-1}(\mathbf{x}_{0}) = a/\exp(\mathbf{x}_{0})$\\
        with known shape $a$
        &
        & $\alpha_j(\mathbf{x}_{t}) = \tau_j(\mathbf{x}_{t}) a + 1, \beta_j(\mathbf{x}_{t}) = \nu_j(\mathbf{x}_{t})$
        & 
        \\
        \midrule
        Pareto distribution 
        & $ \prod_{j = 1}^d \text{Pareto}(y_{i,j};x_m, \theta_j)$ 
        & $\prod_{j = 1}^d \text{Gamma}(\theta_j;\alpha_j(\mathbf{x}_{t}), \beta_j(\mathbf{x}_{t})) $
        & $g^{-1}(\mathbf{x}_{0}) = 1/(\exp(\mathbf{x}_{0}) -\log(x_m))$\\
        with known scale $x_m$
        & 
        & $\alpha_j(\mathbf{x}_{t}) = \tau_j(\mathbf{x}_{t}) + 1, $
        & 
        \\
         
        & 
        & $\beta_j(\mathbf{x}_{t}) = \nu_j(\mathbf{x}_{t}) - \tau_j(\mathbf{x}_{t}) \log(x_m)$
        & 
        \\
        \midrule
        \midrule
        Binomial distribution 
        & $ \prod_{j = 1}^d \text{Binomial}(y_{i,j};n, \theta_j)$ 
        & $\prod_{j = 1}^d \text{Beta}(\theta_j;\alpha_j(\mathbf{x}_{t}), \beta_j(\mathbf{x}_{t})) $
        & $g^{-1}(\mathbf{x}_{0}) = \text{sigmoid}(\mathbf{x}_{0}) $\\
        with known \# trials $n$
        & 
        & $\alpha_j(\mathbf{x}_{t}) = \nu_j(\mathbf{x}_{t}) + 1,$
        & 
        \\
        & 
        & $\beta_j(\mathbf{x}_{t}) = \tau_j(\mathbf{x}_{t})n - \nu_j(\mathbf{x}_{t})  + 1$
        & 
        \\
        \midrule
        Negative Binomial distribution 
        & $ \prod_{j = 1}^d \text{Neg-Binomial}(y_{i,j};r, \theta_j)$ 
        & $\prod_{j = 1}^d \text{Beta}(\theta_j;\alpha_j(\mathbf{x}_{t}), \beta_j(\mathbf{x}_{t})) $
        & $g^{-1}(\mathbf{x}_{0}) = \text{sigmoid}(\mathbf{x}_{0}) $\\
        with known \# successes $r$
        & 
        & $\alpha_j(\mathbf{x}_{t}) = \tau_j(\mathbf{x}_{t})r + 1, \beta_j(\mathbf{x}_{t}) = \nu_j(\mathbf{x}_{t}) + 1$
        & 
        \\
        \midrule
        Geometric distribution 
        & $ \prod_{j = 1}^d \text{Geometric}(y_{i,j}; \theta_j)$ 
        & $\prod_{j = 1}^d \text{Beta}(\theta_j;\alpha_j(\mathbf{x}_{t}), \beta_j(\mathbf{x}_{t})) $
        & $g^{-1}(\mathbf{x}_{0}) = \text{sigmoid}(\mathbf{x}_{0}) $\\
        & 
        & $\alpha_j(\mathbf{x}_{t}) = \tau_j(\mathbf{x}_{t}) + 1, \beta_j(\mathbf{x}_{t}) = \nu_j(\mathbf{x}_{t}) + 1$
        & 
        \\
        \midrule 
        \midrule
        Normal distribution 
        & $\mathcal{N}_d( \boldsymbol{y}_i; \boldsymbol{\mu}, \boldsymbol{\Sigma})$, 
        & $\prod_{j = 1}^d \text{Inverse-Gamma}(\theta_j;\alpha_j(\mathbf{x}_{t}), \beta_j(\mathbf{x}_{t}))$
        & $g^{-1}(\mathbf{x}_0) = \exp(\mathbf{x}_0)$ \\ 
        with known mean $\boldsymbol{\mu}$
        & $\boldsymbol{\Sigma} = \text{diag}(\theta_1, \ldots, \theta_d)$ 
        & $\alpha_j(\mathbf{x}_{t}) = \nu_j(\mathbf{x}_{t})/2 - 1,$
        &   
        \\
        &  
        & $\beta_j(\mathbf{x}_{t}) = \mu_j \tau_j(\mathbf{x}_{t}) /2 -\nu_j(\mathbf{x}_{t}) $
        &   
        \\
        \midrule
        Log-Normal distribution 
        & $\text{Log-Normal}_d( \boldsymbol{y}_i; \boldsymbol{\mu}, \boldsymbol{\Sigma})$, 
        & Same as Normal distribution.
        & Same as Normal distribution. \\ 
        with known mean $\boldsymbol{\mu}$
        & $\boldsymbol{\Sigma} = \text{diag}(\theta_1, \ldots, \theta_d)$ 
        & 
        &   
        \\
        \midrule
        Weibull distribution 
        & $\prod_{j = 1}^d \text{Weibull}(y_{i,j}; \theta_j^{1/k}, k)$  
        & $\prod_{j = 1}^d \text{Inverse-Gamma}(\theta_j;\alpha_j(\mathbf{x}_{t}), \beta_j(\mathbf{x}_{t}))$
        & $g^{-1}(\mathbf{x}_0) = \exp(\mathbf{x}_0)$ \\ 
        with known shape $k$
        &  
        & $\alpha_j(\mathbf{x}_{t}) = \tau_j(\mathbf{x}_{t}) - 1,\beta_j(\mathbf{x}_{t}) = \nu_j(\mathbf{x}_{t}) $
        &   
        \\
        \midrule
        \hline
      \bottomrule
    \end{tabular}
    }
    \end{threeparttable}
    \label{tab:data-distributions-details}
\end{table}
\end{landscape}

\clearpage
\newpage
\section{Proposed Approach to Define Link Function}\label{app-proposed_link_distributions}
In this appendix, we discuss our proposed approach to select the link function $g(\cdot)$. It is important to note that this is a suggested method and that any link function satisfying Assumption~\ref{ass-link-function} can be utilized. The proposed link functions, categorized by likelihood distribution, are summarized in Table~\ref{tab:data-distributions-details}.

Let us recall the problem formulation outline in Section~\ref{sec:method}. For notational convenience, we omit the subscripts on $y$ and $\theta$. We assume partial measurements $y\in \mathcal{Y} \subseteq \mathbb{R}$ dependent on a parameter $\theta \in \Theta \subseteq \mathbb{R}$ following a one-parameter univariate exponential family distribution $p(y|\theta)$ of the form
\begin{equation*}
p(y \vert {\theta}) = h_y(y) \exp\left( {\eta}({\theta}) {T}_y(y) - A_y(\eta({\theta})) \right),
\end{equation*}
where ${\eta}({\theta})$ are the natural parameters, $A_y(y)$ is the log-partition function, ${T}_y(y)$ are the sufficient statistics and $h_y(y)$ is the base measure. 
Further, the parameter $\theta$ is related to a latent variable $x_0$ through the relationship: 
\begin{align*}
    \theta = g^{-1}(x_0).
\end{align*}

\subsection{Direct Derivation of the Link Function}

For a parameter that represents a probability, it is common to directly model the relationship between the latent variable and the parameter using a sigmoid function. 
Similarly, when the parameter represents a variance, it is typical to model the relationship between the latent variable and the parameter with an exponential function.

\paragraph{Link function for the Binomial, Negative Binomial, and Geometric distributions.} We define the link function as $g^{-1}(\mathbf{x}_0) = \text{sigmoid}(\mathbf{x}_0)$.

\paragraph{Link function for the Normal and Log-Normal distributions, with fixed mean.} We define the link function as $g^{-1}(\mathbf{x}_0) = \exp(\mathbf{x}_0)$.

\subsection{Derivation of the Link Function via the Expectation of the Sufficient Statistic}

For the discussion that follows, we define the function $g_2 : \Theta \to \mathcal{M}$ as:
\begin{equation*}
    g_2(\theta) := \frac{\mathrm{d}\theta}{\mathrm{d}\eta}\Big|_{\eta = \eta(\theta)} \frac{\mathrm{d}}{\mathrm{d}\theta} A(\eta(\theta)),
\end{equation*}
where $\mathcal{M} = \left\{g_2(\theta): \theta \in \Theta \right\}$.
From Corollary~\ref{corollary:derivative_log_partition_function}, it follows that the expectation of the sufficient statistic $T_y(y)$ is given by:
\begin{align*}
    \mathbb{E}_{y|\theta}[T_y(y)] = g_2(\theta).
\end{align*}

In the context of our hierarchical probabilistic model (Figure~\ref{fig-graphical-model}), when the parameter of interest $\theta$ is a shape, a rate or a scale, it is usually more intuitive to model the relationship between the latent variable $x_0$ and the expectation of the sufficient statistics $E_{y|\theta}[T_y(y)]$, rather than directly modeling the relationship between a latent variable $x_0$ and the parameter $\theta$ itself.
For example, in modeling with the exponential distribution, it is common to set the parameter as the inverse of a function of the latent variable. This arises because the expectation of a random variable following the exponential distribution with rate $\theta$ is the inverse of its parameter, $1/\theta$.
Therefore, we propose that the link function $g(\cdot)$ should be constructed as a composition of two distinct link functions. The first link function relates the latent variable to the expectation of the sufficient statistics, while the second maps the expectation of the sufficient statistics to the parameter. Specifically, the inverse link function can be expressed as:
\begin{align*}
    &g^{-1}(\cdot) = g_2^{-1}(g_1^{-1}(\cdot)),  \\
    &g_1^{-1}(\cdot) : \mathbb{R} \to \mathcal{M}, \quad g_2^{-1}(\cdot) : \mathcal{M} \to \Theta, 
\end{align*}
such that
\begin{align*}
    \mathbb{E}_{y|\theta}[T_y(y)] &= g_1^{-1}(x_0)\\
    \theta &= g_2^{-1}\left(\mathbb{E}_{y|\theta}[T_y(y)]\right).
\end{align*}
The following assumption is imposed on the link function $g_1$
\begin{assumption}
\label{ass-link-function-appendix}
    We assume that $g_1: \mathcal{M} \to \mathbb{R}$ is a continuously differentiable, one-to-one function and with $\frac{\mathrm{d}g}{\mathrm{d}x}\neq~0$ for all $x\in\mathcal{M}$.
\end{assumption}
Assumption~\ref{ass-link-function-appendix} is standard and is consistent with those typically used in the context of the change-of-variable technique in probability and statistics (see Theorem 17.2 in \cite{billingsley_prob}).
Note that this approach of deriving the link function is related to that of a Generalized Linear Model (GLM), which considers distributions from the exponential family for which $T_y(y) = y$. In GLMs, the linear predictor $x_0 :=\mathbf{X} \boldsymbol{\beta} \in \Lambda$ is associated with the first moment, denoted by $\mu := \mathbb{E}_{y|\theta}[y] \in  \mathcal{M} $, through a link function $\tilde{g}(\cdot)$. This relationship is expressed as $\mu = \tilde{g}^{-1}(x_0)$. 
GLMs assume that the linear predictor belongs to the same space as the natural parameters $\eta \in \Lambda$, which is often not the case in practice.
A commonly used link function in GLMs, known as the canonical link function, assumes that $\tilde{g}^{-1}(\cdot): \Lambda \to \mathcal{M} =  \frac{\mathrm{d}}{\mathrm{d}\eta }A(\eta)$.

\paragraph{Link function for the Normal distribution with variance $\sigma^2$.} We find that $g_2(\theta) = \theta$, further we decide to use $g^{-1}_1(\mathbf{x}_0) = \mathbf{x}_0$. Therefore, $g^{-1}(\mathbf{x}_0) = \mathbf{x}_0$.

\paragraph{Link function for the Log-Normal distribution with variance $\sigma^2$.} We find that $g_2(\theta) = \theta$, further we decide to use $g^{-1}_1(\mathbf{x}_0) = \mathbf{x}_0$. Therefore, $g^{-1}(\mathbf{x}_0) = \mathbf{x}_0$.

\paragraph{Link function for the Poisson distribution.} We find that $g_2(\theta) = \theta$, further we decide to use $g^{-1}_1(\mathbf{x}_0) = \exp(\mathbf{x}_0)$. Therefore, $g^{-1}(\mathbf{x}_0) = \exp(\mathbf{x}_0)$.

\paragraph{Link function for the Exponential distribution.} We find that $g_2(\theta) = 1/\theta$, further we decide to use $g^{-1}_1(\mathbf{x}_0) = \exp(\mathbf{x}_0)$. Therefore, $g^{-1}(\mathbf{x}_0) = 1/\exp(\mathbf{x}_0)$.

\paragraph{Link function for the Gamma distribution with fixed shape $a$.} We find that $g_2(\theta) = a/\theta$, further we decide to use $g^{-1}_1(\mathbf{x}_0) = \exp(\mathbf{x}_0)$. Therefore, $g^{-1}(\mathbf{x}_0) = a/\exp(\mathbf{x}_0)$.

\paragraph{Link function for the Pareto distribution with fixed scale $x_m$.} We find that $g_2(\theta) = 1/\theta + \log(x_m)$, further we decide to use $g^{-1}_1(\mathbf{x}_0) = \exp(\mathbf{x}_0)$. Therefore, $g^{-1}(\mathbf{x}_0) = 1/(\exp(\mathbf{x}_0) - \log(x_m))$.

\paragraph{Link function for the Weibull distribution with fixed shape $k$.} We find that $g_2(\theta) = \theta$, further we decide to use $g^{-1}_1(\mathbf{x}_0) = \exp(\mathbf{x}_0)$. Therefore, $g^{-1}(\mathbf{x}_0) = \exp(\mathbf{x}_0)$.

\subsection{Ensuring Adequate Prior Coverage}
To ensure the stability of the neural network, it is often desirable for $\mathbf{x}_0$ --- the latent variable upon which $\mathbf{x}_t$ depends for $t <1$ --- to lie within a scalable range, such as $[0, 1]$ or $[-1, 1]$. In this context, it is crucial to consider whether the range of the prior distribution can encompass the desired values of the parameter $\boldsymbol{\theta}$. 
For instance, if $\mathbf{x}_0$ is constrained to the range $[-1, 1]$ and the data are binomial, the value of the parameter $\boldsymbol{\theta}$ that the prior covers are limited to the range $\text{sigmoid}(-1) \approx 0.27$ and $\text{sigmoid}(1) \approx 0.73$. 
A practical solution is to introduce a scaling factor $s$ to the link function. For binomial data, this adjustment modifies the link function to become $g^{-1}(\mathbf{x}_0) = \text{sigmoid}(s\,\mathbf{x}_0)$. In the example above, if $s=5$, the range of the parameters covered become $\text{sigmoid}(-5) \approx 0.001$ and $\text{sigmoid}(5) \approx 0.99$. 

\clearpage
\newpage
\section{Implementation Details}
\label{app-implementation}
\paragraph{Architecture for experiment in Section~\ref{sec-1d-synthetic-data}.}
We employed a feedforward neural network with six hidden layers, each containing 96 neurons and using SiLU activations. Prior to being input into the network, the diffusion timestep was transformed using sinusoidal positional embeddings of length 64, as described in~\citet{Vaswani2017}.
The network’s final layer was designed to reduce the output dimensionality to $d$ for the score network and $2\times d$ for the inference network. Training was conducted with a batch size of $1,000$, using a learning rate of 
$1\text{e}{-4}$ for the score network and $1\text{e}{-3}$ for the inference network. The training process spanned $100,000$ epochs.

\paragraph{Architecture for experiments in Sections~\ref{sec-experiment-cox-process}-\ref{sec-experiment-malaria}.}
We employed the same U-Net architecture as~\citet{ho_denoising} for both the score network~\eqref{eq:phi_star} and inference network~\eqref{eq-amortized-objective}, with a key modification: the number of convolutional channels is reduced from 128 to 64 to align with our single-channel input data (one vs. three-channel in~\citet{ho_denoising}), optimizing efficiency while preserving performance. The network features a symmetric encoder-decoder structure with six residual layers, where the encoder progressively downsamples spatial resolution via strided convolutions and the decoder upsamples via transposed convolutions, interconnected by skip channels between matching resolution stages. Each residual layer contains two convolutional blocks with SiLU activations, batch normalization, and diffusion timestep conditioning --- achieved by injecting Transformer-style sinusoidal positional embeddings~\citep{Vaswani2017} of length 64 into each block. A single-head global attention layer is incorporated at the $16\times16$ resolution. 
The final layer is a $1\times1$ convolution which fuses the concatenated features via channel-wise linear combination, reducing the output dimensionality to one channel for the score network and two channels for the inference network. 

Note that we were unable to leverage pre-trained weights for the score network, as done in \citet{chung2023} or \citet{boys2024}, because the pre-trained U-Nets used in these studies were trained on three-channel RGB images.

Our score network model had $33,766,529$ parameters, and inference network had $33,943,234$ parameters. To train the networks, we used a batch size of $128$ images, a learning rate of $1\text{e}{-4}$ for both networks. We ran the training of the score network for $100,000$ epochs and the training of the inference network for $50,000$ epochs.

\paragraph{SDE.}
We adopted the standard VP-SDE setup with $\beta(t) = \beta_0 + t (\beta_1 - \beta_0)$, $\beta_1 = 20$ and $\beta_0 = 0.001$.

\paragraph{Sampling algorithm.}
To sample from the posterior, we employed the \textit{Predictor-Corrector} (PC) sampling algorithm proposed by~\citet[Algorithm 1, 3, 5]{song2021scorebased}. The predictor step was implemented using the Euler-Maruyama solver, while the corrector step used the annealed Langevin dynamics. The gradients, used by the Predictor and Corrector are defined by
\begin{align*}
&\mathbf{g} = \mathbf{g}_{\text{prior}} + \mathbf{g}_{\text{likelihood}} \\
&\mathbf{g}_{\text{prior}} \simeq \nabla_{\mathbf{x}_t}  \log p_{\mathbf{x}_t}(\mathbf{x}_t), \quad \mathbf{g}_{\text{likelihood}} \simeq \nabla_{\mathbf{x}_t} \log p_{\mathbf{y}|\mathbf{x}_t}(\mathbf{y}|\mathbf{x}_{t})
\end{align*}
where 
\begin{align*}
&\mathbf{g}_{\text{prior}} = \mathbf{s}_{\boldsymbol{\phi}^\star}(\mathbf{x}_t, t) \label{eq:gradient_sampling} \\
& \mathbf{g}_{\text{likelihood}}  = \nabla_{\mathbf{x}_t} \left(h_{\mathbf{y}}(\mathbf{y}) - A_{\boldsymbol{\theta}}(\boldsymbol{\nu}^\star(\mathbf{x}_t), \boldsymbol{\tau}^\star(\mathbf{x}_t)) + A_{\boldsymbol{\theta}}(\mathbf{T}_{\mathbf{y}}(\mathbf{y}) + \boldsymbol{\nu}^\star(\mathbf{x}_t), \boldsymbol{\tau}^\star(\mathbf{x}_t) + N \mathbf{1}_d)\right)
\end{align*}
with $(\boldsymbol{\nu}^\star, \boldsymbol{\tau}^\star) = \boldsymbol{\zeta}_{\boldsymbol{\rho}^\star}(\mathbf{x}_t, t)$ and where $\mathbf{s}_{\boldsymbol{\phi}^\star}(\mathbf{x}_t, t)$ and $\boldsymbol{\zeta}_{\boldsymbol{\rho}^\star}(\mathbf{x}_t, t)$ are the trained score and inference networks, respectively. To prevent numerical instabilities, we clip the gradients $\mathbf{g}$ to the range $[-10,10]$.
Following~\citet{song2021scorebased}, we used $1,000$ noise scales for the sampling process. Through experimentation, we found that the best results were achieved with a \textit{signal-to-noise ratio} of $r = 0.1$. 
Lastly, we generated $500$ posterior samples to compute summary statistics such as posterior medians and quantiles.

\paragraph{Compute time.} 
The experiments in Section~\ref{sec-1d-synthetic-data} required 40 minutes to run the score and inference networks in a multi-CPU environment.
The experiments in Sections~\ref{sec-experiment-cox-process}-\ref{sec-experiment-malaria} were performed on a single NVIDIA A100-SXM4-80GB GPU. The score and inference networks took 30 hours to train. Sampling 128 posterior samples concurrently on a single GPU took 30 minutes per image. This process was repeated until the required 500 posterior samples were obtained.

\paragraph{Code availability.} The code and data are available on the GitHub repository \url{https://github.com/MLGlobalHealth/score-sde-expfam/}.

\clearpage
\newpage
\section{Experiments Set-up}
\label{app-experiment-set-up}
\subsection{One-dimensional Benchmark Analysis}

\paragraph{Likelihood.} We considered multiple likelihood distributions to evaluate our approach in comparison to MCMC and DPS. Specifically, we included the Gaussian and Poisson likelihoods, defined as follows:
\begin{align*} 
&p_{\mathbf{y}^{\text{normal}} \vert \boldsymbol{\theta}}(\mathbf{y}^{\text{normal}}|\boldsymbol{\theta}) = \prod_{i=1}^{N}\prod_{j=1}^{d} \mathcal{N}\left(y_{i,j};\theta_{j}, 1\right),\\ &p_{\mathbf{y}^{\text{poisson}} \vert \boldsymbol{\theta}}(\mathbf{y}^{\text{poisson}}|\boldsymbol{\theta}) = \prod_{i=1}^{N}\prod_{j=1}^{d} \text{Poisson}\left(y_{i,j};\theta_{j}\right), \end{align*}
where we set $N = 1$ and $d = 30$.
Furthermore, to compare our approach to MCMC for additional likelihood distributions, we considered two additional distributions: the Exponential and Log-Normal likelihoods. These are given by
\begin{align*} 
&p_{\mathbf{y}^{\text{log-normal}} \vert \boldsymbol{\theta}}(\mathbf{y}^{\text{log-normal}}|\boldsymbol{\theta}) = \prod_{i=1}^{N}\prod_{j=1}^{d} \text{Log-Normal}\left(y_{i,j};\theta_{j}, 1\right),\\
&p_{\mathbf{y}^{\text{exponential}} \vert \boldsymbol{\theta}}(\mathbf{y}^{\text{exponential}}|\boldsymbol{\theta}) = \prod_{i=1}^{N}\prod_{j=1}^{d} \text{Exponential}\left(y_{i,j};\theta_{j}\right), 
\end{align*}
again with $N = 1$ and $d = 30$.

\paragraph{Prior.} The prior was a zero-mean Gaussian Process:
\begin{equation*}
    \mathbf{x}_0 \sim \mathcal{GP}(\boldsymbol{0}, \mathbf{K})
\end{equation*}
where $\mathbf{K}$ is a covariance matrix defined by an RBF kernel with variance 1 and length-scale $0.1$. 

\paragraph{Inverse link function.} 
The inverse link functions for the Gaussian, Log-Normal, and Exponential likelihood were specified as follows:
\begin{equation*}
\begin{aligned}
     \boldsymbol{\theta} = g^{-1}(\mathbf{x}_0) &= \mathbf{x}_0, &&  (\text{Normal likelihood})\\
    \boldsymbol{\theta} = g^{-1}(\mathbf{x}_0) &= \mathbf{x}_0, &&  (\text{Log-Normal likelihood})\\
    \boldsymbol{\theta} = g^{-1}(\mathbf{x}_0) &= 1 / \exp(\mathbf{x}_0)  && (\text{Exponential likelihood})
\end{aligned}
\end{equation*}
For the Poisson likelihood, we used two different link functions to explore both high-rate and low-rate regimes,
\begin{equation*}
\begin{aligned}
    \boldsymbol{\theta} = g^{-1}(\mathbf{x}_0) &= \exp(\mathbf{x}_0)  && (\text{Poisson likelihood with low rate})\\
    \boldsymbol{\theta} = g^{-1}(\mathbf{x}_0) &= \exp(5 +\mathbf{x}_0) && (\text{Poisson likelihood with high rate})
\end{aligned}
\end{equation*}

\paragraph{True Intensity.} The true intensity, denoted as $\boldsymbol{\theta}^{\text{true}}$, was generated by sampling from the prior distribution.

\paragraph{Observations.} For each likelihood (Normal, Poisson, Log-Normal, Exponential), we generated a single sample ($N = 1$) with dimension $d = 30$, using the corresponding true intensity. All the observations were allocated to the training set.

\subsection{Score-Based Cox Process.}
\label{app-score-based-cox-process-experiment}
\paragraph{Cox Process definition.}
We consider the inhomogeneous Poisson process on a domain $\mathcal{S}\subset\mathbb{R}^{D}$. This process is characterized by a stochastic intensity function $\lambda(\boldsymbol{s}):\mathcal{S}\to\mathbb{R}_{+}$, which specifies the rate at which events occur at each location $\boldsymbol{s}\in\mathcal{S}$. For any subregion $\mathcal{T}\subset\mathcal{S}$, the number of events $N(\mathcal{T})$ within $\mathcal{T}$ is Poisson-distributed with parameter $\lambda_{\mathcal{T}} = \int_{\mathcal{T}} \lambda(\boldsymbol{s})d\boldsymbol{s}$. Additionally, for any disjoint subsets $\mathcal{S}_{i}$ of $\mathcal{S}$, the event counts $N(\mathcal{S}_{i})$ are independent. This process is frequently referred to as a Cox process, as it was introduced by \citet{cox_process}.

Let $\mathbf{s} =(\boldsymbol{s}_{1},\ldots,\boldsymbol{s}_{M})$ with $\boldsymbol{s}_{m}\in \mathcal{S}$ for $m=1,\ldots,M$ be  a set of $M$ observed points. The probability density of $\mathbf{s}$ conditioned on the intensity function $\lambda$ is 
\begin{equation*}
    p_{\mathbf{s}\vert \lambda}(\mathbf{s}| \lambda) = \exp\left(-\int_{\mathcal{S}} \lambda(\boldsymbol{s})\mathrm{d}\boldsymbol{s}\right)\prod_{m=1}^{M}\lambda(\boldsymbol{s}_{m})
\end{equation*}
To make inference on $\lambda$, a prior distribution over the intensity function, $p_{\lambda}(\lambda)$, is introduced. Applying Bayes’ rule, the posterior distribution of $\lambda$ conditioned on the observations is
\begin{equation*}
    p_{\lambda\vert\mathbf{s}}(\lambda | \mathbf{s}) = \frac{p_{\lambda}(\lambda)p_{\mathbf{s}\vert \lambda}(\mathbf{s}| \lambda)}{\int p_{\lambda}(\lambda)p_{\mathbf{s}\vert \lambda}(\mathbf{s}| \lambda) d\lambda}
\end{equation*}
This posterior is \textit{doubly intractable} due to the two challenging integrals: one over $\mathcal{S}$ in the numerator and another over $\lambda$ in the denominator.

\paragraph{Likelihood.}
For computational simplicity, we assume that the observations $\mathbf{s}$ can be grouped into $d$ distinct sets, represented by a finite partition $\mathcal{P} = \{\mathcal{S}_{j}\}_{j=1}^{d}$ of $\mathcal{S}$. The count of observations over the partition is denoted by $\boldsymbol{y} = (y_{1},\ldots, y_{d})$, where $y_{j}$ is the number of observations in the set $\mathcal{S}_{j}$ and it is defined as $y_{j} = \sum_{m=1}^{M} \mathbbm{1}_{\{\boldsymbol{s}_{m}\in\mathcal{S}_{j}\}}$. Under this partition, we approximate the function $\lambda$ with a finite dimensional approximation $\boldsymbol{\theta}= (\theta_{1},\ldots,\theta_{d})\in\mathbb{R}^{d}$, where each $\theta_{j}$ approximates the aggregated function $\lambda$ over the subset $\mathcal{S}_{j}$. The discrete Cox process has the finite-dimensional distribution 
\begin{equation}
\label{eq-discrete-nhpp}
   p_{\boldsymbol{y}\vert \boldsymbol{\theta}}(\boldsymbol{y}|\boldsymbol{\theta})=\prod_{j=1}^{d} \text{Poisson}\left(y_{j};\theta_{j}|\mathcal{S}_{j}|\right)
\end{equation}
where $|\mathcal{S}_{j}|>0$ denotes the measure of $\mathcal{S}_{j}$. 
Henceforth, we assume that the sets $\{\mathcal{S}_j\}_{j=1}^d$ are of equal measure, which without loss of generality we set to $|\mathcal{S}_1|$. Furthermore, we assume that we can observe multiple samples, denoted as $\mathbf{y} = (\boldsymbol{y}_1, \ldots, \boldsymbol{y}_N)$, where each 
$\boldsymbol{y}_i$ is independently and identically distributed (i.i.d.) according to the distribution in~\eqref{eq-discrete-nhpp}. The likelihood of the discrete Cox process becomes, 
\begin{equation*}
   p_{\mathbf{y}\vert \boldsymbol{\theta}}(\mathbf{y}|\boldsymbol{\theta})=\prod_{i=1}^{N}\prod_{j=1}^{d} \text{Poisson}\left(y_{i,j};\theta_{j}|\mathcal{S}_{1}|\right).
\end{equation*}

\paragraph{Evidence.}
As shown in~\eqref{eq-discrete-nhpp}, the intensity of the Poisson depends on the measure $|\mathcal{S}_1|$. Using the evidence trick, the log-density $\log p_{\mathbf{y}|\mathbf{x}_t}(\mathbf{y}|\mathbf{x}_t)$ is still tractable. A straightforward computation shows that $\log p_{\mathbf{y}|\mathbf{x}_t}(\mathbf{y}|\mathbf{x}_t)$ can be approximated as  
\begin{multline*}
    \log p_{\mathbf{y}|\mathbf{x}_t}(\mathbf{y}|\mathbf{x}_{t})
     \approx \log h_{\mathbf{y}}(\mathbf{y}) + \vert \mathcal{S}_1\vert\, \mathbf{T}_{\mathbf{y}}(\mathbf{y})  - A_{\boldsymbol{\theta}}\left(\boldsymbol{\nu}(\mathbf{x}_t), \boldsymbol{\tau}(\mathbf{x}_t)\right) + A_{\boldsymbol{\theta}}\left(\mathbf{T}_{\mathbf{y}}(\mathbf{y}) + \boldsymbol{\nu}(\mathbf{x}_t), \boldsymbol{\tau}(\mathbf{x}_t) + (N  \vert \mathcal{S}_1\vert) \mathbf{1}_d\right) .
\end{multline*}
where the likelihood's base measure $h_{\mathbf{y}}$ and sufficient statistics $\mathbf{T}_{\mathbf{y}}$, as well as the natural conjugate prior's log-partition function $A_{\boldsymbol{\theta}}$ are the same as in standard Poisson Likelihood case, presented in Appendix~\ref{app-table_distributions}.

\paragraph{Prior.}
The prior samples of $\mathbf{x}_0$ are images from the ImageNet train dataset, which are first converted to grayscale, then resized such that their smaller
edge measured 256 pixels, center-cropped to the size $256\times 256$ and finally scaled to the range $[-1, 1]$ using the transformation
\begin{align*}
\mathbf{x}_0 = 2 \frac{\mathbf{x}_0^{\text{orginal}} - \min(\mathbf{x}_0^{\text{orginal}})} { \max(\mathbf{x}_0^{\text{orginal}}) -  \min(\mathbf{x}_0^{\text{orginal}})} - 1
\end{align*}
Note that we converted the images to grayscale, reducing them to a single channel, as this aligns with the latent parameter space typically encountered in real-world problems.

\paragraph{Inverse link function.}
We used the following inverse link function:
\begin{equation*}
\boldsymbol{\theta} = g^{-1}(\mathbf{x}_0) = \exp(\mathbf{x}_0).
\end{equation*}

\paragraph{True intensity.}
The true intensity $\boldsymbol{\theta}^{\text{true}}$ is obtained by applying the inverse link function on the true latent variable  $\mathbf{x}_0^{\text{true}}$. We selected the following true latent variables:
\begin{itemize}
    \item Two images from the ImageNet validation set: A picture of a corn and a house were used in the experiments. The images underwent the same preprocessing steps as the prior samples.
    \item Building Height in London: We used the average building height in 2023 (variable \texttt{MEAN\_OBJECT\_HEIGHT\_M}) in Greater London sourced from~\citet{LBSM}. The map was cropped to a region defined by latitudes between $51.3274$ and $51.6874$ and longitudes between $-0.4178$ and $0.1622$. The data were then rescaled to a $256 \times 256$ grid, where each grid cell contained the average building height within its boundaries. Finally, the building height values were normalized to the range $[-1, 1]$ using the following transformations:  
\begin{align*}
\mathbf{x}_0^{\text{true}} = 2  \frac{\mathbf{x}_0^{\text{true, norm}} - \min(\mathbf{x}_0^{\text{true, norm}})}{\max(\mathbf{x}_0^{\text{true, norm}}) - \min(\mathbf{x}_0^{\text{true, norm}})} - 1, \quad \mathbf{x}_0^{\text{true, norm}} = \log(\mathbf{x}_0^{\text{true, original}} + 1).
\end{align*}
\item Satellite image: The satellite image was obtained using Google Earth Engine with Sentinel-2 Image Collection. The data was filtered to cover the region of New York City Manhattan area defined by latitudes between $-74.02$ and $-73.97$ and longitudes between $40.7$ and $40.75$, at a resolution of $10$ meters, and in the time range from January 1, 2021, to December 31, 2021. The RGB image was converted to grayscale.
A $256\times256$ region was cropped starting at coordinates 
$(100,200)$. The least cloudy image was selected based on the \texttt{CLOUDY\_PIXEL\_PERCENTAGE} property.  Finally, entries were normalized following transformations:  
\begin{align*}
\mathbf{x}_0^{\text{true}} = -\left(2  \frac{\mathbf{x}_0^{\text{true, norm}} - \min(\mathbf{x}_0^{\text{true, norm}})}{\max(\mathbf{x}_0^{\text{true, norm}}) - \min(\mathbf{x}_0^{\text{true, norm}})}\right)^4 + 1, \quad \mathbf{x}_0^{\text{true, norm}} = \log(\mathbf{x}_0^{\text{true, original}} + 0.01).
\end{align*}
The purpose of this transformation was to assign high values to the streets and low values to the buildings.
\end{itemize}

\paragraph{Observations.}
We generated $N$ i.i.d samples of the Cox process, $\mathbf{y} = \{\boldsymbol{y}_i\}_{i=1, \ldots, N}$, where the set of variables for one sample is $\boldsymbol{y}_i = \{y_{i,j}\}_{j=1, \ldots, d}$. Given the true intensity $\boldsymbol{\theta}^{\text{true}}$, we generated the variables $y_{i,j}$ with
\begin{equation*}
    y_{i,j} \sim \text{Poisson}(\theta_j^{\text{true}} |\mathcal{S}_1|),
\end{equation*}
where $|\mathcal{S}_1| = 1$ and for $j = 1, \ldots, d$ and $i = 1, \ldots, N$. This resulted in $N \times 256 \times 256 = N \times 65,536$ observations. We randomly allocated 80\% of the observations ($N \times 52,428$) to the training set and 20\% to the test set ($N \times 13,108$).

\subsection{Prevalence of Malaria in Sub-Saharan Africa.}

\paragraph{Likelihood.} 
In each location $j = 1, \ldots, d$, the observations consist of the number of positive cases, denoted as $y_j$, out of the total number of individuals examined, $n_j$. 
Accordingly, the number of positive cases follows a Binomial distribution
\begin{align*}
p_{\mathbf{y}|\boldsymbol{\theta}}(\mathbf{y}|\boldsymbol{\theta}) = \prod_{j = 1}^d \text{Binomial}(y_j; n_j, \theta_j)
\end{align*}
We note that in this experiment $N =1$, so we omit the indexing on $i$.

\paragraph{Prior.} The same prior as in the Score-Based Cox Process experiment was used.

\paragraph{Inverse link function.}
We used the following inverse link function:
\begin{equation*}
\boldsymbol{\theta} = g^{-1}(\mathbf{x}_0) = \text{sigmoid}(s\, \mathbf{x}_0),
\end{equation*}
with $s=5$. Such that the range covered by the prior on parameters $\boldsymbol{\theta}$ is $\text{sigmoid}(-5) \approx 0.001$ and $\text{sigmoid}(5) \approx 0.99$.

\paragraph{Observations.}
We used a real-world dataset on PfPR in Sub-Saharan Africa from the Malaria Atlas Project~\citep{Bhatt2015-uk,Pfeffer2018-cm,Weiss2019-au}. 
We ignore temporal aspects, and only aim to interpolate spatial data across all of Sub-Saharan Africa. We use a grid resolution of $256 \times 256$, equivalent to a $\sim 111 \text{ km}^2$ resolution, and aggregate the count of the number of positive cases and the number of examined to this resolution. Out of the grid, $7,048$ ($10.75$\%) entries had non-missing observations. We randomly allocated 80\% of the observations ($5,621$) to the train set and 20\% to the test set ($1,427$).

\subsection{Benchmark Methods}
\label{app-benchmark}

\subsubsection{Diffusion Posterior Sampling}
\label{app-dps}
Recall from~\eqref{eq-likelihood-y-xt} the following factorization of the conditional density $p_{\mathbf{y}|\mathbf{x}_t}(\mathbf{y}|\mathbf{x}_t)$ that we repeat here for convenience:
\begin{equation*}
    p_{\mathbf{y}|\mathbf{x}_t}(\mathbf{y}|\mathbf{x}_t) 
    = \int p_{\mathbf{y}|\mathbf{x}_0}(\mathbf{y}|\mathbf{x}_0) \, p_{\mathbf{x}_0|\mathbf{x}_t}(\mathbf{x}_0|\mathbf{x}_t) \, \mathrm{d}\mathbf{x}_0.
\end{equation*}
for all $t\in[\epsilon,1]$.
When $p_{\mathbf{y}|\mathbf{x}_0}(\mathbf{y}|\mathbf{x}_0)$ is a Gaussian distribution with mean $\mathcal{H}(\mathbf{x}_0)$ and known covariance matrix $\boldsymbol{\Sigma}_{d} = \sigma^2 \mathbf{I}_{d}$, \citet{chung2023} proposed approximating $p_{\mathbf{y}|\mathbf{x}_t}(\mathbf{y}|\mathbf{x}_t)$ by  $p_{\mathbf{y}|\mathbf{x}_0}(\mathbf{y}|\hat{\mathbf{x}}_0)$ where $\hat{\mathbf{x}}_0$ denotes the mean of $p_{\mathbf{x}_0|\mathbf{x}_t}(\mathbf{x}_0|\mathbf{x}_t)$:
\begin{equation*}
\hat{\mathbf{x}}_0= \mathbb{E}_{\mathbf{x}_0 \sim p_{\mathbf{x}_0|\mathbf{x}_t}}[\mathbf{x}_0]. 
\end{equation*}
This approach approximates $p_{\mathbf{x}_0|\mathbf{x}_t}(\mathbf{x}_0|\mathbf{x}_t)$ with a Dirac delta function centered at $\hat{\mathbf{x}}_0$. Consequently, the integral in~\eqref{eq-likelihood-y-xt} simplifies to the following closed-form expression:
\begin{equation*}
p_{\mathbf{y}|\mathbf{x}_t}(\mathbf{y}|\mathbf{x}_t) \approx \frac{1}{\sqrt{(2\pi)^{n}\sigma^{2n}}}\exp\left[-\frac{\norm{\mathbf{y}-\mathcal{H}(\hat{\mathbf{x}}_{0} (\mathbf{x}_t) )}^{2}}{2\sigma^{2}}\right]
\end{equation*}
where the dependency of $\hat{\mathbf{x}}_{0}$ on $\mathbf{x}_t$ has been made explicit by the notation $\hat{\mathbf{x}}_{0} (\mathbf{x}_t)$.Here, $n$ represents the total number of observations, defined in our notation as $n=N \times d$. It is worth noting that~\citet{chung2023}, in their experiments, focused on the case where $N=1$, making $n=d$. Nevertheless, their methodologies naturally extend to cases where $N>1$. Moreover, note that under our notation, $\mathcal{H}(\cdot) = g^{-1}(\cdot)$.

In the context of VP-SDEs, $\hat{\mathbf{x}}_0$ has an analytically tractable form:
\begin{equation}
\label{eq:tweedieformula}
    \hat{\mathbf{x}}_0 = \frac{1}{\sqrt{\alpha_t} }\left(\mathbf{x}_t + (1-\alpha_t) \nabla_{\mathbf{x}_t} \log p_{\mathbf{x}_t}(\mathbf{x}_t)\right),
\end{equation}
where $\alpha_t = \exp\left(-\left(t \beta_0 + \left(\beta_1 - \beta_0\right) \frac{t^2}{2}\right)\right)$. The expression in~\eqref{eq:tweedieformula}
is often referred to as \textit{Tweedie's formula}. 
Here, the prior score function $\nabla_{\mathbf{x}_t} \log p_{\mathbf{x}_t}(\mathbf{x}_t)$ appearing in~\eqref{eq:tweedieformula} can be approximated using the score network $\mathbf{s}_{\boldsymbol{\phi}^{\star}}(\mathbf{x}_t, t)$. 

\textbf{Observations following a Normal distribution.} For inverse problems corrupted by Gaussian noise with unknown mean and known variance $\sigma^2$, \citet{chung2023} proposed to approximate the posterior score function with:
\begin{equation*}
\nabla_{\mathbf{x}_t}\log p_{\mathbf{x}_t \vert  \mathbf{y}}(\mathbf{x}_t \vert  \mathbf{y}) \simeq \mathbf{s}_{\boldsymbol{\phi}^{\star}}(\mathbf{x}_t, t) - \rho \sum_{i = 1}^N  \nabla_{\mathbf{x}_t} \norm{ \boldsymbol{y}_i - \mathcal{H}(\hat{\mathbf{x}}_{0}) }^2,
\end{equation*}
where $\rho = 1/ \sigma^2$.

\textbf{Observations following a Poisson distribution.} For inverse problems corrupted by Poisson noise, \citet{chung2023} proposed two methodologies. The first, termed \textit{Poisson-LS}, applies the same least squares (LS) method used for Gaussian noise. This approach is justified by the fact that Poisson noise approximates Gaussian noise in high signal-to-noise regimes. The posterior score function associated with this approximation is expressed as:
\begin{equation*}
\nabla_{\mathbf{x}_t}\log p_{\mathbf{x}_t \vert  \mathbf{y}}(\mathbf{x}_t \vert  \mathbf{y}) \simeq \mathbf{s}_{\boldsymbol{\phi}^{\star}}(\mathbf{x}_t, t) - \rho_t \sum_{i = 1}^N  \nabla_{\mathbf{x}_t} \norm{ \boldsymbol{y}_i - \mathcal{H}(\hat{\mathbf{x}}_{0}) }^2.
\end{equation*}
where $\rho_t$ is a weight that depends on diffusion timestep $t$.
The second approach, which we call \textit{Poisson-DPS}, adopts a shot noise approximation. The posterior score function associated with this approximation is expressed as:
\begin{equation} \label{eq:poisson_dps}
\nabla_{\mathbf{x}_t}\log p_{\mathbf{x}_t \vert  \mathbf{y}}(\mathbf{x}_t \vert  \mathbf{y}) \simeq \mathbf{s}_{\boldsymbol{\phi}^{\star}}(\mathbf{x}_t, t) - \rho_t \sum_{i = 1}^N  \nabla_{\mathbf{x}_t} 
\left\lVert \boldsymbol{y}_i - \mathcal{H}(\hat{\mathbf{x}}_{0}) \right\rVert_{\boldsymbol{\Lambda}_i}^2.
\end{equation}
where $\left\lVert \mathbf{a} \right\rVert_{\boldsymbol{\Lambda}}^2 = \mathbf{a}^\top \boldsymbol{\Lambda} \mathbf{a}$ and $\boldsymbol{\Lambda}_i$ is a diagonal matrix with entries $[\boldsymbol{\Lambda}_i]_{jj} = 1/(2 y_{i,j})$.
The step size $\rho_t$ in each method is set to 
\begin{equation*}
    \rho_t = \frac{\rho'}{\sqrt{\sum_{i = 1}^N \norm{\boldsymbol{y}_i - \mathcal{H}(\hat{\mathbf{x}}_{0}(\mathbf{x}_t))}^2}}
\end{equation*}
where $\rho'$ is a hyperparameter which is manually chosen and where the dependency on $\mathbf{x}_{t}$ has been made explicit. The recommended value for Poisson noise is $\rho' = 0.3$. 

The results of the experiments discussed in Section~\ref{sec-1d-synthetic-data} are presented in Appendix~\ref{app-results-synthetic-1d}. 
For the Gaussian likelihood (Figure~\ref{fig:synthetic-data-experiment-gaussian}), we observed that DPS produced point estimates within the correct range but underestimated uncertainty. We hypothesize that this discrepancy stems from DPS not fully capturing the uncertainty of the conditional distribution $p_{\mathbf{y} \vert \mathbf{x}_t}(\mathbf{y} \vert \mathbf{x}_t)$. In contrast, our method yielded results closer to the ground truth MCMC posterior in both point estimates and uncertainty quantification.  

For the Poisson likelihood, we first tested a low-rate regime (Figure~\ref{fig:synthetic-data-experiment-poisson-low-rate}), where both Poisson-LS and Poisson-DPS performed poorly. We suspected that the norm in~\eqref{eq:poisson_dps} introduced numerical instabilities due to division by Poisson observations, which struggle to handle zero-count events. To address this, we added a small noise term $(0.01)$ to zero observations, but this did not improve performance. We then tested a high-rate regime (Figure~\ref{fig:synthetic-data-experiment-poisson-high-rate}), yet the results remained unsatisfactory.  

We also experimented with different values of $\rho'$ but observed no noticeable impact on performance. The results presented use the recommended value $\rho' = 0.3 $.  
  
\subsubsection{Gaussian Process}
To obtain the ground-truth MCMC results of the experiment presented in Section~\ref{sec-1d-synthetic-data}, we used \texttt{PyStan} version 3.10.0~\citep{pystan}. Four Hamiltonian Monte Carlo (HMC) chains were run in parallel for 600 iterations, with the first 100 iterations considered as warm-up.

\subsubsection{Gaussian Markov Random Field}

In the experiment described in Section~\ref{sec-experiment-malaria}, we compared our results with the case where the prior $p_{\mathbf{x}_{0}}(\mathbf{x}_{0})$ was defined by a GMRF. 
We considered a discrete set of $d$ points $\mathcal{S}=\{\mathbf{s}_{1},\ldots,\mathbf{s}_{d}\}$, where $\mathbf{s}_{j}\in\mathbb{R}^{2}$ for $j=1,\ldots,d$.
We use the following model:
\begin{align*}
    &y_{j} \sim \text{Binomial}(n_j, \theta_j), \\
    &\theta_j = \text{Sigmoid}(\beta_0 + x_{0,j})\\ 
    &p(\beta_0) \propto 1, \\
    &\mathbf{x}_{0} \sim \mathcal{N}(\boldsymbol{0}, \mathbf{K}),
\end{align*}
for  $j = 1, \ldots, d$ and where $\mathbf{K}$ is a covariance matrix. Note that there are only one sample of the observations so we omit the indexing on $i$.

The entries of $\mathbf{K}$ are given by
\begin{equation*}
    \mathbf{K}_{i, j}  = k(\mathbf{s}_{i}, \mathbf{s}_{j}) \quad \text{for } \mathbf{s}_{i}, \mathbf{s}_{j} \in \mathcal{S},
\end{equation*}
where $k(\cdot, \cdot)$ is the kernel function. 
\citet{Lindgren2011-fv} introduced an explicit link between a certain stochastic partial differential equation and GMRFs. \citet{Lindgren2011-fv} considers linear stochastic partial differential equations of the form $\mathcal{L} u(\cdot) = \mathcal{W}(\cdot)$ to define random fields $u(\cdot)$ with differential operator $\mathcal{L}$ and $\mathcal{W}$ is a Gaussian white noise process on a general domain. Choosing $\mathcal{L}=\tau(\kappa^2-\triangle)^{\frac{\alpha}{2}}$, the resulting stochastic partial differential equation
\begin{equation*}
    \tau(\kappa^2-\triangle)^{\frac{\alpha}{2}} u = \mathcal{W}
\end{equation*}
have stationary solutions with a Matérn kernel function of the form 
\begin{equation*}
    k(\mathbf{s}, \mathbf{s}') = \frac{\sigma^2}{\Gamma(\nu)2^{\nu-1}}(k||\mathbf{s}-\mathbf{s}'||)^\nu K_\nu (k||\mathbf{s}-\mathbf{s}'||).
\end{equation*}
Where $\nu=\alpha = \frac{d}{2}$ and $\sigma^2 = \Gamma(\nu)\{\Gamma(\alpha)(4\pi)^{d/2}\kappa^{2\nu}\tau^2\}^{-1}$ and $K_\nu$ is the modified Bessel function of the second kind. $\nu>0$ is called the smoothness index (generally fixed to $\frac{1}{2},\frac{3}{2},\frac{5}{2}$), $\kappa>0$ is the spatial range and $\sigma^2$ is the marginal variance. The stochastic partial differential equation can be solved using the finite element method resulting in a GMRF. This approach therefore allows the construction of sparse precision matrices for GMRFs that are invariant to geometry of the spatial neighborhood and allow for extremely accurate low rank approximations with computational tractability. Inference using these precision matricies is performed using approximate Bayesian inference via the integrated nested Laplace approximation~\citep{Rue2009-ty} and constitutes the state-of-the-art in spatial statistics for approximate methods~\citep{Heaton2017-vl}. 

We used \texttt{INLA} version 24.12.11 to fit the GMRF. 
In the experiment, the Matérn stochastic partial differential equation with the parameter $\alpha = 1$ yielded the best results.

\clearpage
\newpage
\section{Further Experimental Results}
\label{app-further-experimental-results}
\subsection{Further Results for Section~\ref{sec-1d-synthetic-data}}
\label{app-results-synthetic-1d}
\subsubsection{Normal Observations}
\begin{figure*}[h!]
\centering
\includegraphics[width=\textwidth]{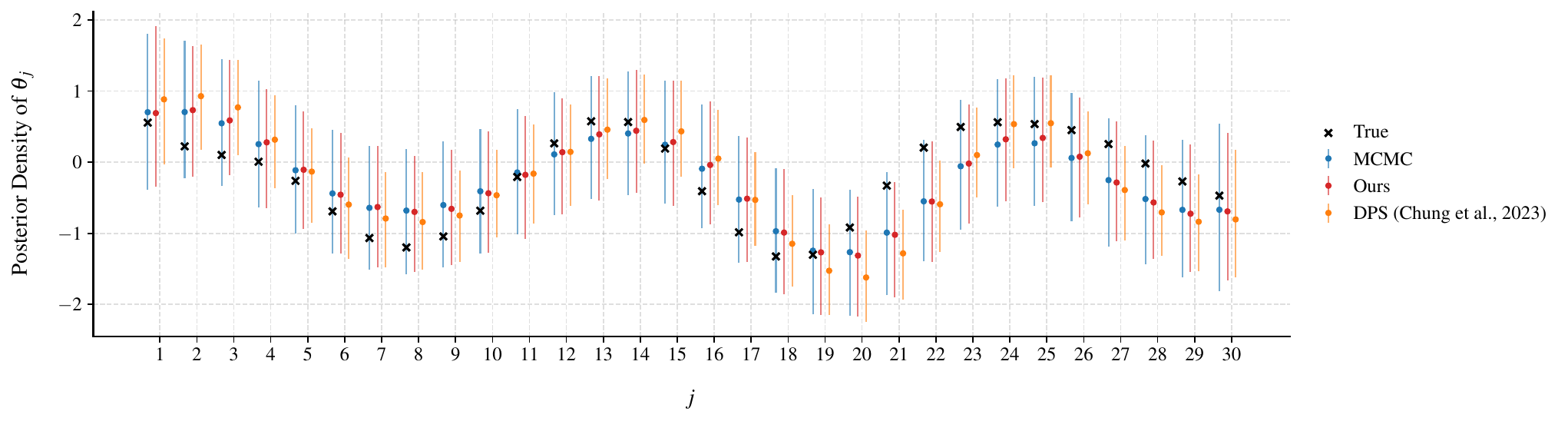}
\caption{\textbf{Posterior Density of $\boldsymbol{\theta}$ given observations following a Normal distribution.} Estimated posterior median (dot) and 95\% credible interval (error bars) by three methods (colors) along with the true value of $\boldsymbol{\theta}$ (cross). The inference was performed given $N=1$ observations following a Normal distribution for which the mean was equal to $\boldsymbol{\theta} = \mathbf{x}_0$ and the standard-deviation was fixed to $\sigma =1$.}
\label{fig:synthetic-data-experiment-gaussian}
\end{figure*}

\subsubsection{Poisson Observations}
\begin{figure*}[h!]
\centering
\includegraphics[width=\textwidth]{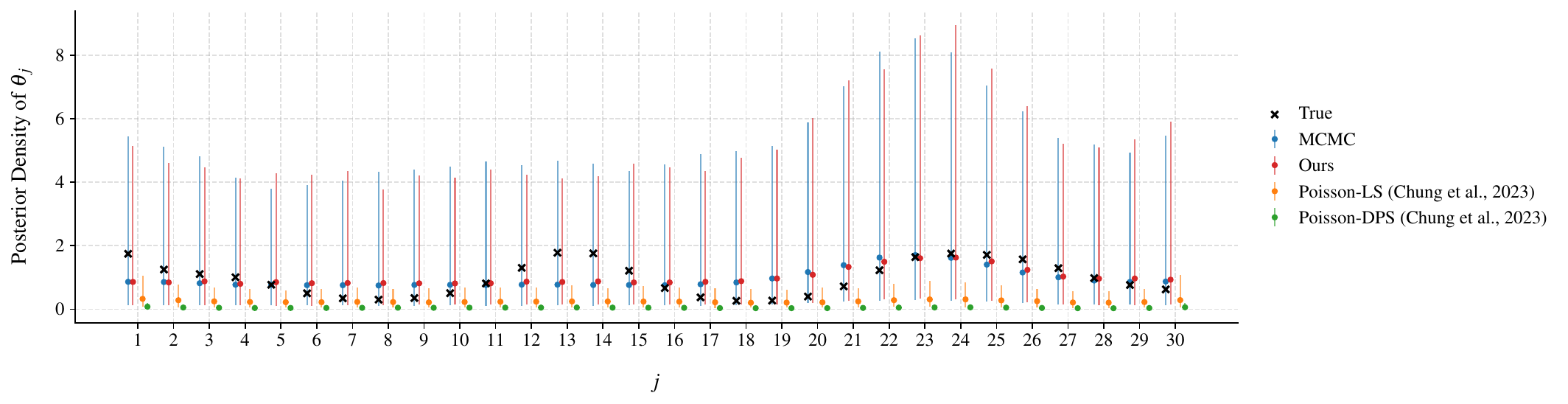}
\caption{\textbf{Posterior Density of $\boldsymbol{\theta}$ given observations following a Poisson distribution with low rate.} Estimated posterior median (dot) and 95\% credible interval (error bars) by four methods (colors) along with the true value of $\boldsymbol{\theta}$ (cross). The inference was performed given $N=1$ Poisson observations for which the rate was equal to $\boldsymbol{\theta} = \exp (\mathbf{x}_0)$.}
\label{fig:synthetic-data-experiment-poisson-low-rate}
\end{figure*}

\begin{figure*}[h!]
\centering
\includegraphics[width=\textwidth]{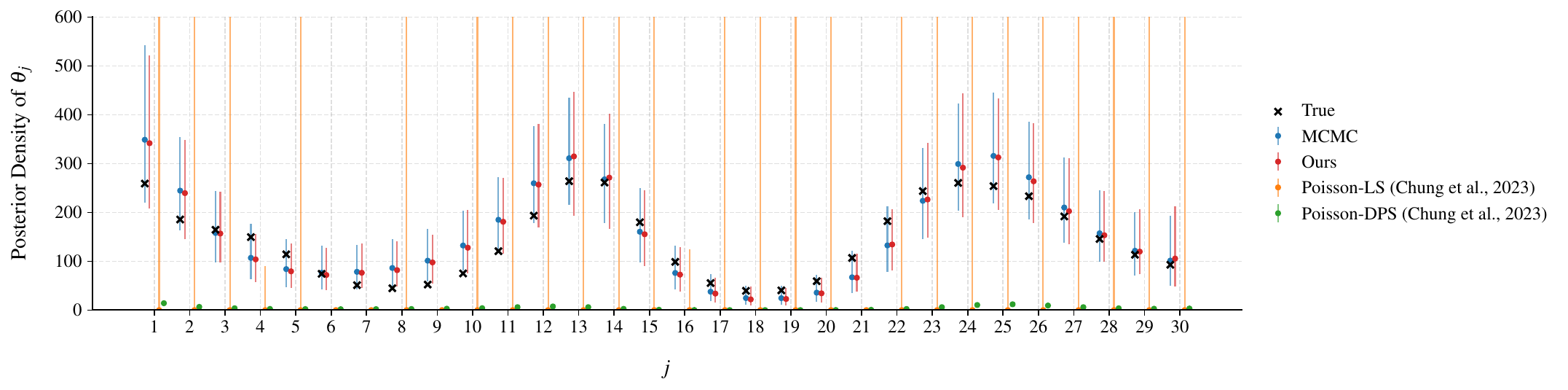}
\caption{\textbf{Posterior Density of $\boldsymbol{\theta}$ given observations following a Poisson distribution with high rate.} Estimated posterior median (dot) and 95\% credible interval (error bars) by four methods (colors) along with the true value of $\boldsymbol{\theta}$ (cross). The inference was performed given $N=1$ Poisson observations for which the rate was equal to $\boldsymbol{\theta} = \exp (5 + \mathbf{x}_0)$.}
\label{fig:synthetic-data-experiment-poisson-high-rate}
\end{figure*}

\clearpage
\newpage
\subsubsection{Log-Normal Observations}
\begin{figure*}[h!]
\centering
\includegraphics[width=\textwidth]{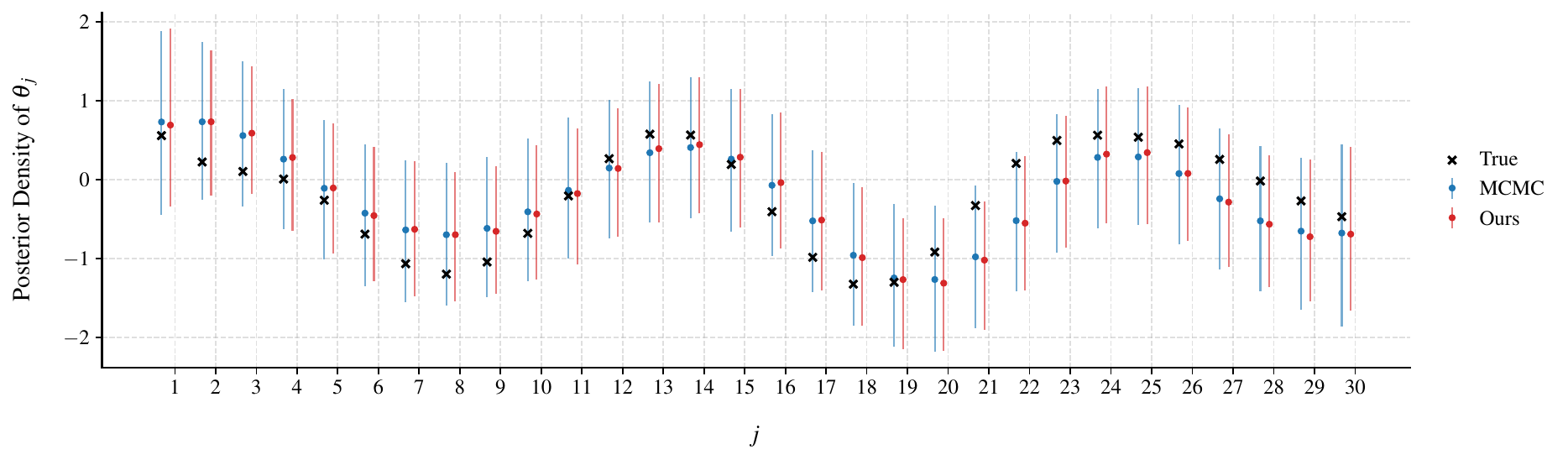}
\caption{\textbf{Posterior Density of $\boldsymbol{\theta}$ given observations following a Log-Normal distribution.} Estimated posterior median (dot) and 95\% credible interval (error bars) by three methods (colors) along with the true value of $\boldsymbol{\theta}$ (cross). The inference was performed given $N=1$ Log-Normal observations for which the logarithm of location was equal to $\boldsymbol{\theta} = \mathbf{x}_0$ and the logarithm of scale was fixed to $\sigma =1$.}
\label{fig:synthetic-data-experiment-lognormal-1}
\end{figure*}

\subsubsection{Exponential Observations}
\begin{figure*}[h!]
\centering
\includegraphics[width=\textwidth]{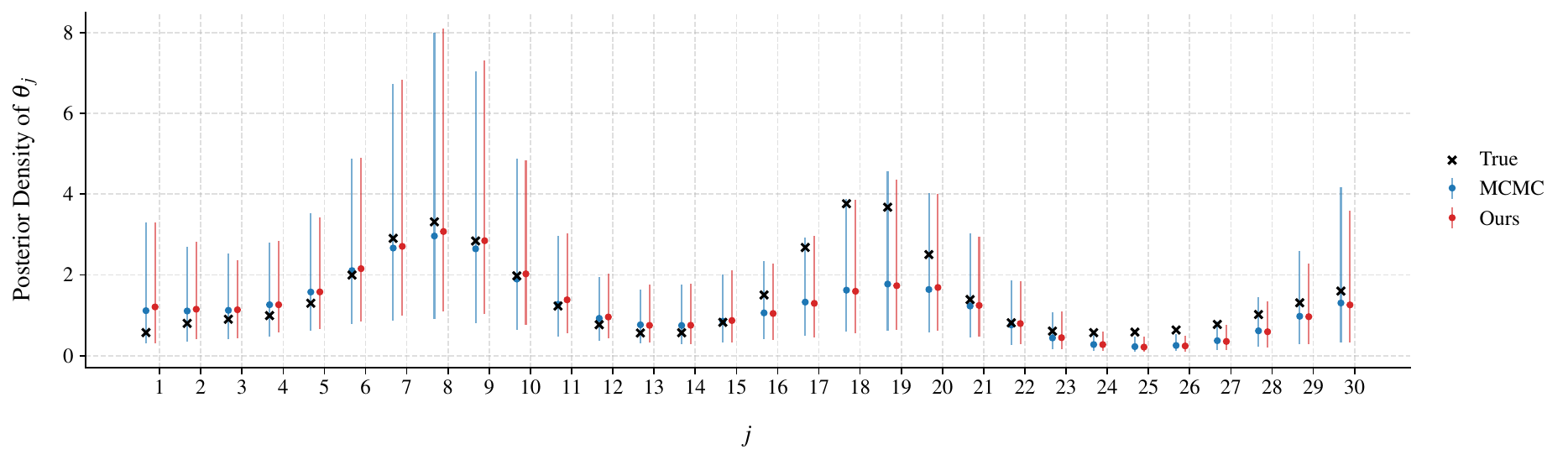}
\caption{\textbf{Posterior Density of $\boldsymbol{\theta}$ given observations following an Exponential distribution.} Estimated posterior median (dot) and 95\% credible interval (error bars) by three methods (colors) along with the true value of $\boldsymbol{\theta}$ (cross). The inference was performed given $N=1$ Exponential observations for which the rate was equal to $\boldsymbol{\theta} = 1 / \exp(\mathbf{x}_0)$.}
\label{fig:synthetic-data-experiment-exponential-1}
\end{figure*}

\clearpage
\newpage
\subsection{Further Results for Section~\ref{sec-experiment-cox-process}}
\label{app-further-experiment-cox-process}

\begin{figure*}[ht!]
\centering
\includegraphics[width=\textwidth]{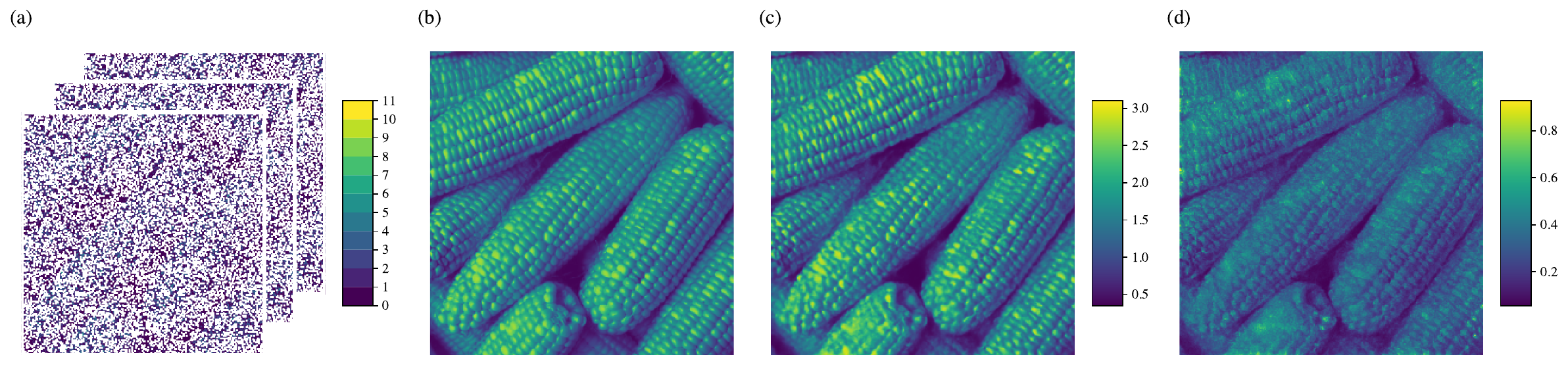}
\caption{\textbf{Further Results for Score-Based Cox Process Experiment with ImageNet Intensity (1).} \textbf{(a)} Observations in the train set (80\% of the grid points) sampled from a Cox Process with true intensity equal to the image in (b).  \textbf{(b)} True Cox Process intensity from the ImageNet validation set, transformed using an exponential link function. \textbf{(c)} Median of the estimated Cox Process intensity posterior distribution using the Score-Based Cox Process method.   \textbf{(d)} Interquartile range of the estimated Cox Process intensity posterior distribution using the Score-Based Cox Process method. }
\end{figure*}

\begin{figure*}[ht!]
\centering
\includegraphics[width=\textwidth]{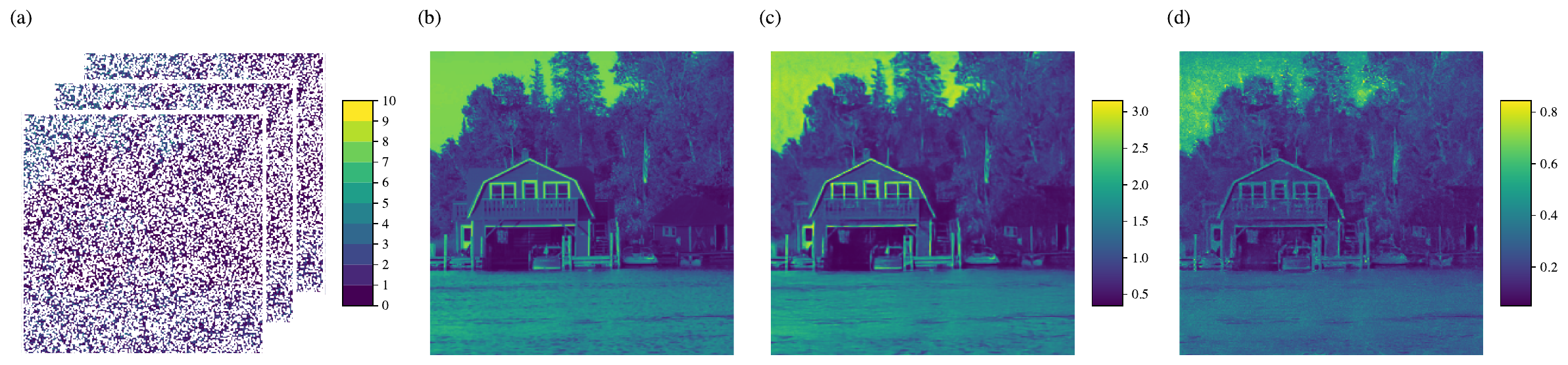}
\caption{\textbf{Further Results for Score-Based Cox Process Experiment with ImageNet Intensity (2).} \textbf{(a)} Observations in the train set (80\% of the grid points) sampled from a Cox Process with true intensity equal to the image in (b).  \textbf{(b)} True Cox Process intensity from the ImageNet validation set, transformed using an exponential link function. \textbf{(c)} Median of the estimated Cox Process intensity posterior distribution using the Score-Based Cox Process method.   \textbf{(d)} Interquartile range of the estimated Cox Process intensity posterior distribution using the Score-Based Cox Process method. }
\end{figure*}

\begin{figure*}[ht!]
\centering
\includegraphics[width=\textwidth]{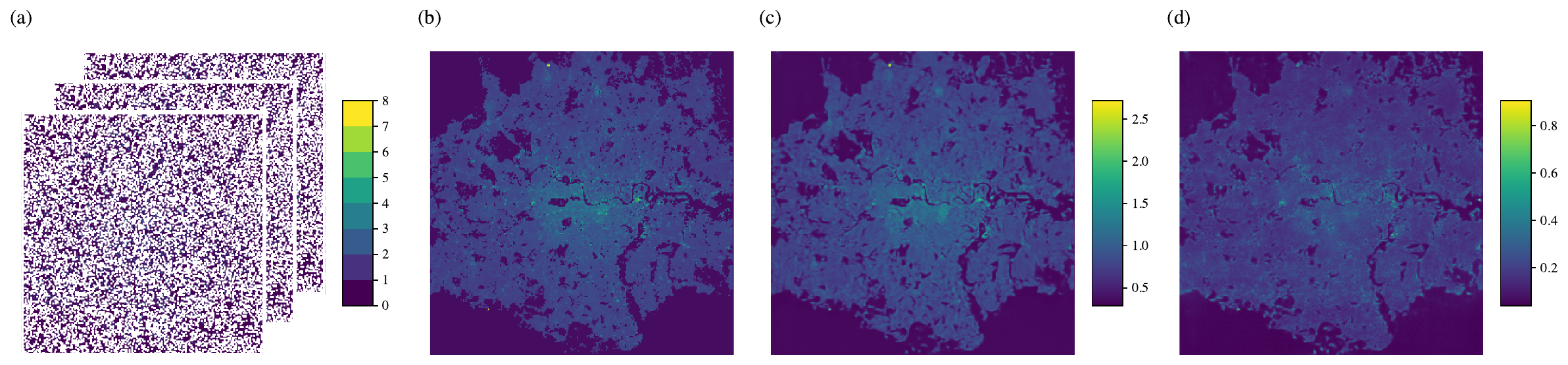}
\caption{\textbf{Further Results for Score-Based Cox Process Experiment with London Buildings Height Intensity.} \textbf{(a)} Observations in the train set (80\% of the grid points) sampled from a Cox Process with true intensity equal to the image in (b).  \textbf{(b)} True Cox Process intensity representing buildings height in London. \textbf{(c)} Median of the estimated Cox Process intensity posterior distribution using the Score-Based Cox Process method.   \textbf{(d)} Interquartile range of the estimated Cox Process intensity posterior distribution using the Score-Based Cox Process method. }
\end{figure*}

\begin{figure*}[ht!]
\centering
\includegraphics[width=\textwidth]{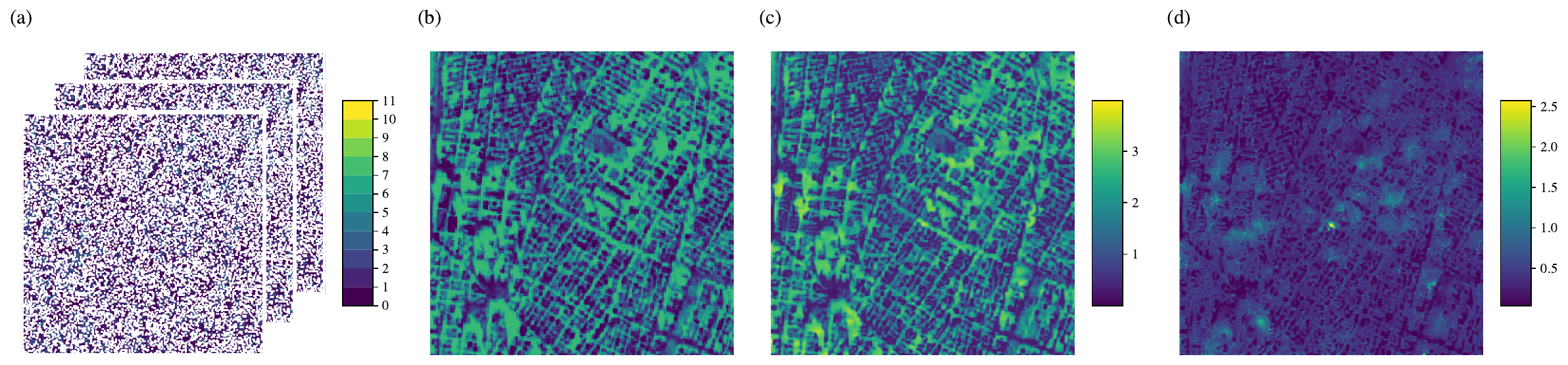}
\caption{\textbf{Further Results for Score-Based Cox Process Experiment with Satellite Image Intensity.} \textbf{(a)} Observations in the train set (80\% of the grid points) sampled from a Cox Process with true intensity equal to the image in (b).  \textbf{(b)} True Cox Process Intensity from Sentinel-2 Satellite Imagery of Manhattan, New York City. \textbf{(c)} Median of the estimated Cox Process intensity posterior distribution using the Score-Based Cox Process method.   \textbf{(d)} Interquartile range of the estimated Cox Process intensity posterior distribution using the Score-Based Cox Process method. }
\end{figure*}

\begin{figure*}[ht!]
\centering
\includegraphics[width=\textwidth]{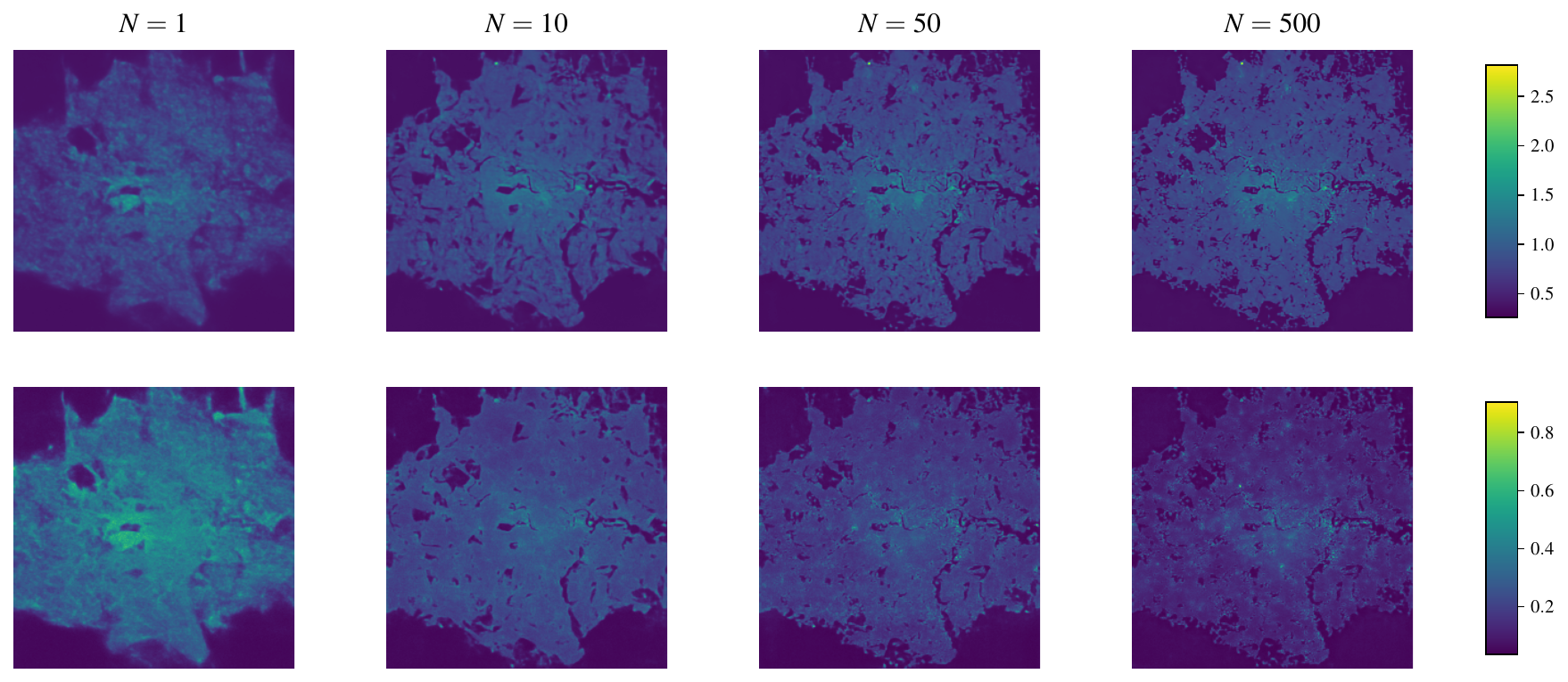}
\caption{\textbf{Further Results for Score-Based Cox Process Experiment with Varying Number of Samples $N$.} \textbf{(Top)} Median of the estimated Cox Process intensity posterior distribution using the Score-Based Cox Process method. \textbf{(Bottom)}
Interquartile range of the estimated Cox Process intensity posterior distribution using the Score-Based Cox Process method. }
\end{figure*}

\clearpage
\newpage
\subsection{Further Results for Section~\ref{sec-experiment-malaria}}
\label{app-further-experiment-malaria}

\begin{figure*}[ht!]
\centering
\includegraphics[width=\textwidth]{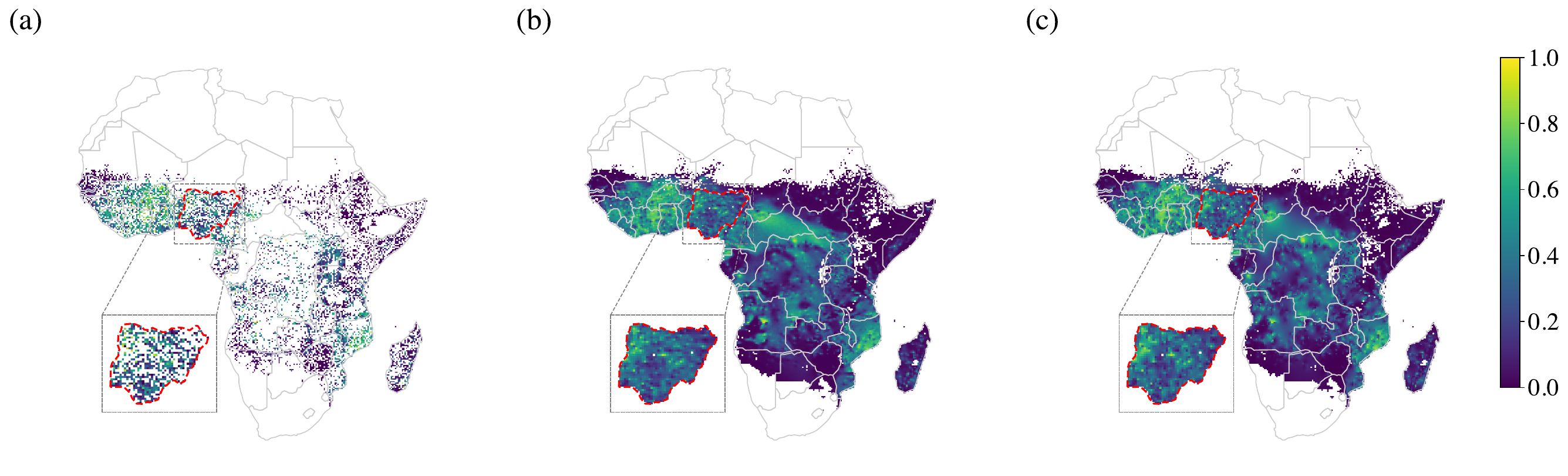}
\caption{
\textbf{Benchmark Comparison on the Prevalence of Malaria in Sub-Saharan Africa Experiment.}
\textbf{(a)} Empirical PfPR. \textbf{(b)} Median of the estimated PfPR posterior distribution using our approach (MAE = $0.1207$). \textbf{(c)} Mean of the estimated PfPR posterior distribution using a GMRF (MAE = $0.1225$).
The inset plots highlight Nigeria, one of the countries with the highest malaria burden worldwide.
The empty entries either correspond to locations outside Sub-Saharan Africa or are attributed to lakes or desert zones. }
\end{figure*}

\begin{figure*}[ht!]
\centering
\includegraphics[width=\textwidth]{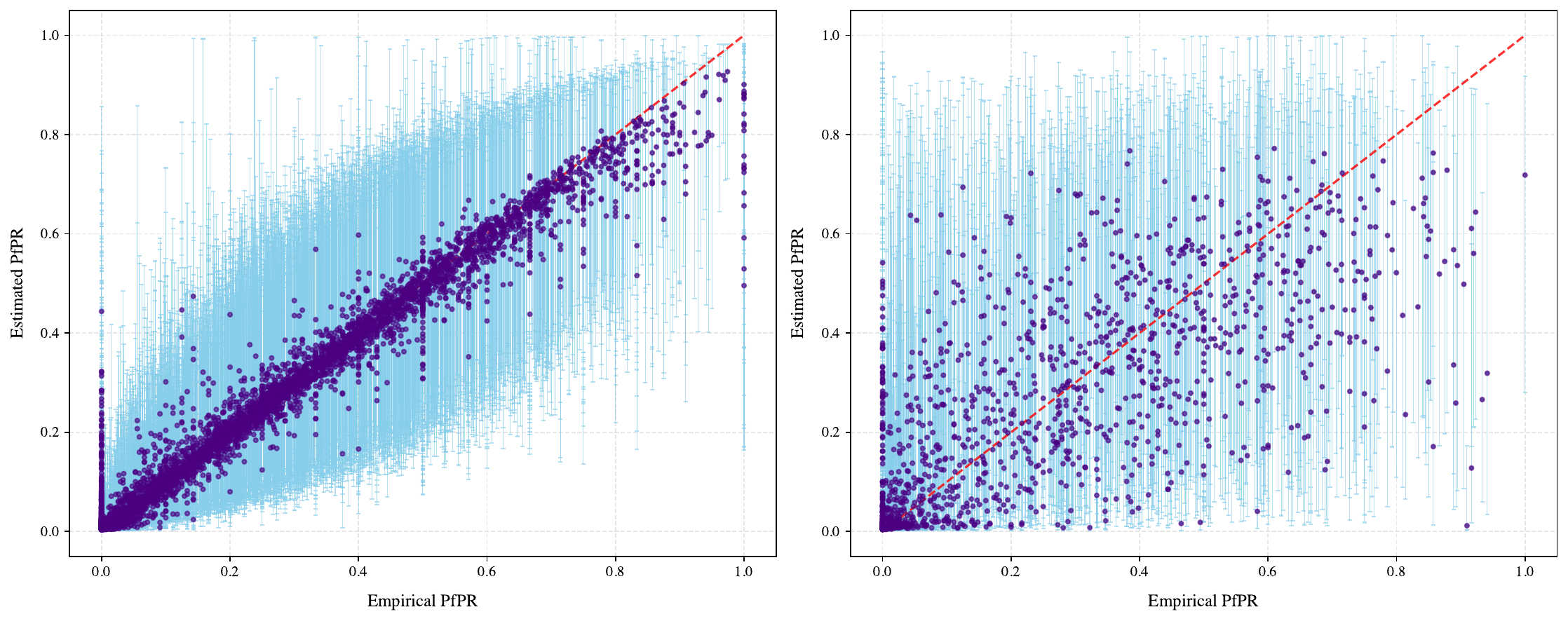}
\caption{\textbf{Posterior Predictive Checks of the Prevalence of Malaria in Sub-Saharan Africa Experiment.}  
\textbf{(Left)} Empirical PfPR in the train set against the median of the estimated PfPR posterior distribution (dot) along with the 95\% credible interval (error bars). The number of empirical PfPR in the train set lying inside the 95\% credible interval was $80$\%.
\textbf{(Right)} Empirical PfPR in the test set against the median of the estimated PfPR posterior distribution (dot) along with the 95\% credible interval (error bars). The number of empirical PfPR in the test set lying inside the 95\% credible interval was $78$\%. }
\end{figure*}

\begin{figure*}[ht!]
\centering
\includegraphics[width=\textwidth]{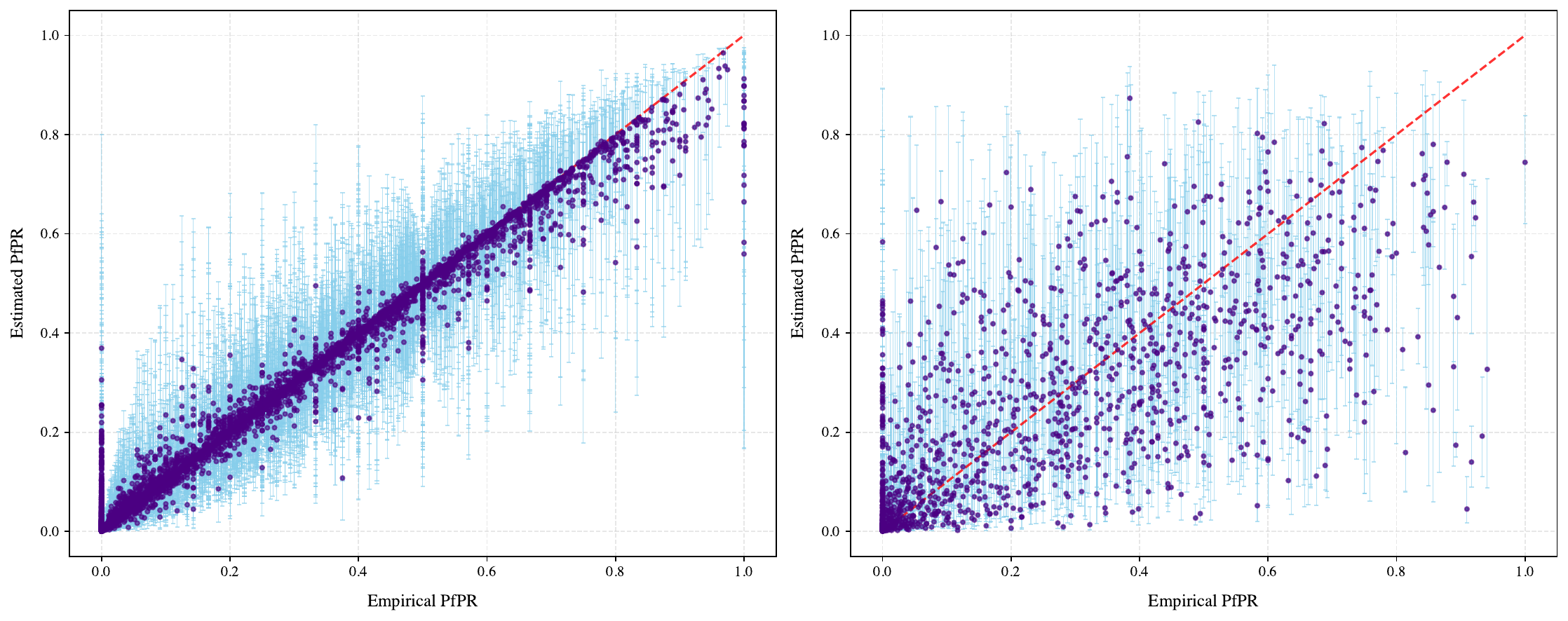}
\caption{\textbf{Posterior Predictive Checks for the Gaussian Markov Random Field model of the Prevalence of Malaria in Sub-Saharan Africa Experiment.}  
\textbf{(Left)} Empirical PfPR in the train set against the mean of the estimated PfPR posterior distribution by a GMRF (dot) along with the 95\% credible interval (error bars). The number of empirical PfPR in the train set lying inside the 95\% credible interval was $81$\%.
\textbf{(Right)} Empirical PfPR in the test set against the mean of the estimated PfPR posterior distribution by a GMRF (dot) along with the 95\% credible interval (error bars). The number of empirical PfPR in the test set lying inside the 95\% credible interval was $43$\%. }
\end{figure*}

\clearpage
\newpage
\section{Limitations}
\label{app-limitations}
Our methodology inherits the characteristics of diffusion model-based methods, which tend to be relatively slow in terms of sampling speed (see Appendix~\ref{app-implementation}). This limitation could potentially be mitigated by incorporating advanced samplers.
Due to the inherent stochasticity in the posterior sampling, as observed in~\citet{chung2023}, we encountered failures among the posterior samples when the signal-to-noise ratio was not properly tuned. To ensure numerical stability, we also had to apply gradient clipping on the approximated posterior gradients (see Appendix~\ref{app-implementation}). Developing more robust techniques to stabilize the sampling process presents an interesting avenue for future research.
To the best of our knowledge, among methods that explicitly approximate the measurement-matching term $\nabla_{\mathbf{x}_t}\log p_{\mathbf{y}\vert\mathbf{x}_t}(\mathbf{y}\vert\mathbf{x}_t)$, our approach is the only one capable of handling observations following non-Gaussian distributions (see \citet{Daras2024} for a comprehensive survey on the subject). However, compared to existing Gaussian-based methodologies, this flexibility comes at the additional cost of training a separate neural network. Notably, though, both networks can be trained in parallel, potentially mitigating computational overhead.

\newpage
\section{Proofs}
\label{app-proofs}
\subsection{Proof of Lemma \ref{lemma-KL-divergence}} \label{sec:proof_lemma-KL-divergence}
\begin{proof}[Proof of Lemma \ref{lemma-KL-divergence}.]
The KL divergence between $ p_{\boldsymbol{\theta}|\mathbf{x}_t} $ and $ q_{\boldsymbol{\theta}|\boldsymbol{\zeta}(\mathbf{x}_t)} $ can be expressed as:
\begin{equation}
\label{eq-kl-divergence-after-factorization}
   D_{\text{KL}}(p_{\boldsymbol{\theta}|\mathbf{x}_t} \vert\vert q_{\boldsymbol{\theta}|\boldsymbol{\zeta}(\mathbf{x}_t)})=  \int p_{\boldsymbol{\theta}|\mathbf{x}_t}(\boldsymbol{\theta}\vert \mathbf{x}_{t}) \log p_{\boldsymbol{\theta}|\mathbf{x}_t}(\boldsymbol{\theta}\vert \mathbf{x}_{t}) \, \mathrm{d}\boldsymbol{\theta} - \int p_{\boldsymbol{\theta}|\mathbf{x}_t}(\boldsymbol{\theta}\vert \mathbf{x}_{t}) \log q_{\boldsymbol{\theta}|\boldsymbol{\zeta}(\mathbf{x}_t)}(\boldsymbol{\theta}\vert \boldsymbol{\zeta}(\mathbf{x}_t)) \, \mathrm{d}\boldsymbol{\theta}
\end{equation}
Let the function $ C(\mathbf{x}_t) $ be defined as
\begin{equation*}
C(\mathbf{x}_t) := \int p_{\boldsymbol{\theta}|\mathbf{x}_t}(\boldsymbol{\theta}\vert \mathbf{x}_{t}) \log p_{\boldsymbol{\theta}|\mathbf{x}_t}(\boldsymbol{\theta}\vert \mathbf{x}_{t}) \, \mathrm{d}\boldsymbol{\theta} -  \int p_{\boldsymbol{\theta}|\mathbf{x}_t}(\boldsymbol{\theta}\vert \mathbf{x}_{t}) \log(h_{\boldsymbol{\theta}}(\boldsymbol{\theta})) \mathrm{d}\boldsymbol{\theta}.
\end{equation*}
Substituting~\eqref{eq-prior-q-theta} into the KL divergence in~\eqref{eq-kl-divergence-after-factorization}, we obtain
\begin{equation}
\label{eq-kl-divergence-before-cov}
\begin{aligned}
   D_{\text{KL}}(p_{\boldsymbol{\theta}|\mathbf{x}_t} \vert\vert q_{\boldsymbol{\theta}|\boldsymbol{\zeta}(\mathbf{x}_t)}) &= C(\mathbf{x}_t)  +  A_{\boldsymbol{\theta}}(\boldsymbol{\zeta}(\mathbf{x}_t)) \left( \int p_{\boldsymbol{\theta}|\mathbf{x}_t}(\boldsymbol{\theta}\vert \mathbf{x}_{t}) \mathrm{d}\boldsymbol{\theta} \right) - \boldsymbol{\zeta}(\mathbf{x}_t)^\top \left( \int p_{\boldsymbol{\theta}|\mathbf{x}_t}(\boldsymbol{\theta}\vert \mathbf{x}_{t}) \mathbf{T}_{\boldsymbol{\theta}}(\boldsymbol{\theta}) \, \mathrm{d}\boldsymbol{\theta} \right) \\
    &= C(\mathbf{x}_t) +  A_{\boldsymbol{\theta}}(\boldsymbol{\zeta}(\mathbf{x}_t)) -  \boldsymbol{\zeta}(\mathbf{x}_t)^\top \mathbb{E}_{p_{\boldsymbol{\theta}|\mathbf{x}_t}}[ \mathbf{T}_{\boldsymbol{\theta}}(\boldsymbol{\theta})].
\end{aligned}
\end{equation}
Notice that
\begin{equation*}
\mathbb{E}_{p_{\boldsymbol{\theta}|\mathbf{x}_t}}[ \mathbf{T}_{\boldsymbol{\theta}}(\boldsymbol{\theta})] = \mathbb{E}_{p_{\mathbf{x}_{0}|\mathbf{x}_t}}[ \mathbf{T}_{\boldsymbol{\theta}}(g^{-1}(\mathbf{x}_{0}))]
\end{equation*}
can be plugged in \eqref{eq-kl-divergence-before-cov} to obtain the desired result.
\end{proof}
\subsection{Proof of Theorem~\ref{prop-new-objective}} \label{proof-prop-new-objective}
Throughout this section, we denote by $p_{\boldsymbol{\theta}}(\boldsymbol{\theta})$ the marginal distribution of $\boldsymbol{\theta}$. Before proving Theorem~\ref{prop-new-objective} we need to show the following result.
\begin{lemma}
\label{lemma-bound-fubini}
Let $\boldsymbol{\zeta}(\mathbf{x}_t)$ be a Lipschitz continuous function. Suppose the following conditions hold:
\begin{equation}
\label{eq-lemma-finite-conditions}
\begin{aligned}
\mathbb{E}_{\boldsymbol{\theta}\sim  p_{\boldsymbol{\theta}}}\left[\norm{\mathbf{T}_{\boldsymbol{\theta}}(\boldsymbol{\theta})}\right] &< \infty, \\
\mathbb{E}_{\boldsymbol{\theta}\sim  p_{\boldsymbol{\theta}}}\left[\norm{g(\boldsymbol{\theta})}\norm{\mathbf{T}_{\boldsymbol{\theta}}(\boldsymbol{\theta})}\right] &< \infty.
\end{aligned}
\end{equation}
Then, the following integral is finite for all $ t \in [\epsilon, 1] $:
\begin{equation}
\label{eq-finite-integral-zeta-T}
\int p_{\mathbf{x}_0}(\mathbf{x}_0) 
\int p_{\mathbf{x}_t|\mathbf{x}_0}(\mathbf{x}_t)
\vert \boldsymbol{\zeta}(\mathbf{x}_t)^{\top} 
\mathbf{T}_{\boldsymbol{\theta}}(g^{-1}({\mathbf{x}}_{0}))\vert 
\mathrm{d}\mathbf{x}_t \mathrm{d}\mathbf{x}_0 < \infty.
\end{equation}
\end{lemma}


\begin{proof}
We express the conditions in~\eqref{eq-lemma-finite-conditions} in terms of $\mathbf{x}_{0}$. Recall that $\mathbf{x}_0 = g(\boldsymbol{\theta})$, therefore:
\begin{equation}
\label{eq-lemma-finite-conditions-x0}
\begin{aligned}
\mathbb{E}_{\mathbf{x}_{0}\sim  p_{\mathbf{x}_{0}}}\left[\norm{\mathbf{T}_{\boldsymbol{\theta}}(g^{-1}(\mathbf{x}_{0}))}\right] &< \infty \\
\mathbb{E}_{\mathbf{x}_{0}\sim  p_{\mathbf{x}_{0}}}\left[\norm{\mathbf{x}_{0}}\norm{\mathbf{T}_{\boldsymbol{\theta}}(g^{-1}(\mathbf{x}_{0}))}\right] &< \infty.
\end{aligned}
\end{equation}
Let $\boldsymbol{\zeta}(\mathbf{x}_t)$ be a $K$-Lipschitz continuous function with respect to the Euclidean norm $\norm{\cdot}$. Recall from~\eqref{eq-forward-transition-kernel} that $\mathbf{x}_t\vert\mathbf{x}_0$ follows a multivariate normal distribution\footnote{Note that in~\eqref{eq-forward-transition-kernel}, $\mathbf{x}_t$ is defined with dimension $d_x$. However, throughout Section~\ref{sec:method}, we assume $\mathbf{x}_t$ has the same dimension $d$ as $\mathbf{y}$ for consistency with the problem setup.} with mean $\sqrt{\alpha_t}\mathbf{x}_0$ and covariance matrix $v_{t}\mathbf{I}_{d}$. We use the triangle inequality and then the Cauchy-Schwarz inequality to obtain the following bound 
\begin{equation}
\label{eq-triangle-cauchy-bound}
\begin{aligned}
\vert\boldsymbol{\zeta}(\mathbf{x}_t)^{\top} \mathbf{T}_{\boldsymbol{\theta}}(g^{-1}(\mathbf{x}_0)) \vert &\leq  \vert(\boldsymbol{\zeta}(\mathbf{x}_t)- \boldsymbol{\zeta}(\sqrt{\alpha_t}\mathbf{x}_0))^{\top} \mathbf{T}_{\boldsymbol{\theta}}(g^{-1}(\mathbf{x}_0)) \vert + \vert\boldsymbol{\zeta}(\sqrt{\alpha_t}\mathbf{x}_0)^{\top} \mathbf{T}_{\boldsymbol{\theta}}(g^{-1}(\mathbf{x}_0))\vert \\
&\leq\norm{\boldsymbol{\zeta}(\mathbf{x}_t)- \boldsymbol{\zeta}(\sqrt{\alpha_t}\mathbf{x}_0)}\norm{\mathbf{T}_{\boldsymbol{\theta}}(g^{-1}(\mathbf{x}_0))} + \norm{\boldsymbol{\zeta}(\sqrt{\alpha_t}\mathbf{x}_0)}\norm{\mathbf{T}_{\boldsymbol{\theta}}(g^{-1}(\mathbf{x}_0))}
\end{aligned}
\end{equation}
for all $t\in[\epsilon,1]$. Additionally, the Lipschitz continuity of $\boldsymbol{\zeta}$ can be applied to the bound in~\eqref{eq-triangle-cauchy-bound}, resulting in the following bound
\begin{equation}
\label{eq-lipschitz-triangle-bound}
\begin{aligned}
\vert\boldsymbol{\zeta}(\mathbf{x}_t)^{\top} \mathbf{T}_{\boldsymbol{\theta}}(g^{-1}(\mathbf{x}_0)) \vert &\leq K \norm{\mathbf{x}_t-\sqrt{\alpha_t}\mathbf{x}_0}\norm{\mathbf{T}_{\boldsymbol{\theta}}(g^{-1}(\mathbf{x}_0))} + \norm{\boldsymbol{\zeta}(\sqrt{\alpha_t}\mathbf{x}_0)}\norm{\mathbf{T}_{\boldsymbol{\theta}}(g^{-1}(\mathbf{x}_0))} \\ 
&\leq K \norm{\mathbf{x}_t}\norm{\mathbf{T}_{\boldsymbol{\theta}}(g^{-1}(\mathbf{x}_0))} +  K\sqrt{\alpha_t}\norm{\mathbf{x}_0}\norm{\mathbf{T}_{\boldsymbol{\theta}}(g^{-1}(\mathbf{x}_0))} \\ &+\norm{\boldsymbol{\zeta}(\sqrt{\alpha_t}\mathbf{x}_0)}\norm{\mathbf{T}_{\boldsymbol{\theta}}(g^{-1}(\mathbf{x}_0))}
\end{aligned}
\end{equation}
where the final inequality follows from the triangle inequality. The inequality in~\eqref{eq-lipschitz-triangle-bound} represents a bound for the integrand in~\eqref{eq-finite-integral-zeta-T}.
The comparison test for Lebesgue integrability (see~\citep[Proposition 4.16]{Royden_2010}) states that, to show the integral in~\eqref{eq-finite-integral-zeta-T} is finite given the bound in~\eqref{eq-lipschitz-triangle-bound}, it is sufficient to verify that the following integrals are finite:
\begin{subequations}
\begin{align}
\mathcal{I}_{1} &:=  K\sqrt{\alpha_t}\int p_{\mathbf{x}_0}(\mathbf{x}_0) \norm{\mathbf{x}_0}\norm{\mathbf{T}_{\boldsymbol{\theta}}(g^{-1}(\mathbf{x}_0))} \mathrm{d}\mathbf{x}_0 \label{eq-I1}\\
\mathcal{I}_{2} &:= \int p_{\mathbf{x}_0}(\mathbf{x}_0)\norm{\boldsymbol{\zeta}(\sqrt{\alpha_t}\mathbf{x}_0)}\norm{\mathbf{T}_{\boldsymbol{\theta}}(g^{-1}(\mathbf{x}_0))} \mathrm{d}\mathbf{x}_0 \label{eq-I2}\\
\mathcal{I}_{3} &:= K  \int p_{\mathbf{x}_0}(\mathbf{x}_0)\int  p_{\mathbf{x}_t|\mathbf{x}_0}(\mathbf{x}_t) \norm{\mathbf{x}_t}\norm{\mathbf{T}_{\boldsymbol{\theta}}(g^{-1}(\mathbf{x}_0))}  \mathrm{d}\mathbf{x}_t \mathrm{d}\mathbf{x}_0 \label{eq-I3} 
\end{align}
\end{subequations}
We will now show that each $\mathcal{I}_{1}$, $\mathcal{I}_{2}$ and $\mathcal{I}_{3}$ is finite for any $t\in[\epsilon,1]$.
\paragraph{$\mathcal{I}_{1}$ is finite.} This follows immediately from~\eqref{eq-I1} and \eqref{eq-lemma-finite-conditions-x0}.
\paragraph{$\mathcal{I}_{2}$ is finite.}
Let $\mathbf{c}$ be an arbitrary point in the domain of $\boldsymbol{\zeta}$. By using the triangle inequality and the Lipschitz continuity of $\boldsymbol{\zeta}$, the integrand of~\eqref{eq-I2} can be bounded as follows
\begin{equation}
\label{eq-bound-integrand-I2}
\begin{aligned}
\norm{\boldsymbol{\zeta}(\sqrt{\alpha_t}\mathbf{x}_0)}\norm{\mathbf{T}_{\boldsymbol{\theta}}(g^{-1}(\mathbf{x}_0))} 
&\leq \norm{\boldsymbol{\zeta}(\sqrt{\alpha_t}\mathbf{x}_0)-\boldsymbol{\zeta}(\mathbf{c})}\norm{\mathbf{T}_{\boldsymbol{\theta}}(g^{-1}(\mathbf{x}_0))}  + \norm{\boldsymbol{\zeta}(\mathbf{c})}\norm{\mathbf{T}_{\boldsymbol{\theta}}(g^{-1}(\mathbf{x}_0))} \\ 
&\leq   K\norm{\sqrt{\alpha_t}\mathbf{x}_0-\mathbf{c}}\norm{\mathbf{T}_{\boldsymbol{\theta}}(g^{-1}(\mathbf{x}_0))}  +  \norm{\boldsymbol{\zeta}(\mathbf{c})}\norm{\mathbf{T}_{\boldsymbol{\theta}}(g^{-1}(\mathbf{x}_0))} \\ 
&\leq (K\sqrt{\alpha_t} +K\norm{\mathbf{c}} +\norm{\boldsymbol{\zeta}(\mathbf{c})})\norm{\mathbf{x}_0}\norm{\mathbf{T}_{\boldsymbol{\theta}}(g^{-1}(\mathbf{x}_0))}  
\end{aligned}
\end{equation}
where the last inequality follows again by the triangle inequality. Notice that it follows immediately from~\eqref{eq-bound-integrand-I2} and~\eqref{eq-lemma-finite-conditions-x0} that
\begin{equation}
\label{eq-bound-integral-I2}
(K\sqrt{\alpha_t} +K\norm{\mathbf{c}} +\norm{\boldsymbol{\zeta}(\mathbf{c})})\int p_{\mathbf{x}_{0}}(\mathbf{x}_{0})\norm{\mathbf{x}_0}\norm{\mathbf{T}_{\boldsymbol{\theta}}(g^{-1}(\mathbf{x}_0))} d\mathbf{x}_{0} < \infty
\end{equation}
for all $t\in[\epsilon,1]$. Hence, it follows immediately from~\eqref{eq-bound-integral-I2} and the comparison test for Lebesgue integrability that $\mathcal{I}_{2}$ is finite for all $t\in[\epsilon,1]$.
\paragraph{$\mathcal{I}_{3}$ is finite.} Notice that $\mathbf{x}_{t}\overset{d}{=}\sqrt{\alpha_t}\mathbf{x}_{0} + \sqrt{v_t} \mathbf{z}$ with $\mathbf{z}\sim\mathcal{N}_{d}(0,\mathbf{I}_{d})$. Therefore, it follows from~\eqref{eq-I3} that
\begin{equation}
\label{eq-expression-I3-normal}
\mathcal{I}_{3} = K  \int p_{\mathbf{x}_0}(\mathbf{x}_0)\int  p_{\mathbf{z}}(\mathbf{z}) \norm{\sqrt{\alpha_t}\mathbf{x}_{0} + \sqrt{v_t} \mathbf{z}}\norm{\mathbf{T}_{\boldsymbol{\theta}}(g^{-1}(\mathbf{x}_0))}  \mathrm{d}\mathbf{z} \mathrm{d}\mathbf{x}_0 
\end{equation}
where $p_{\mathbf{z}}(\mathbf{z})$ is the density of the standard multivariate normal random variable $\mathbf{z}$. We wish to show that $\mathcal{I}_{3}$ is finite by using the expression in \eqref{eq-expression-I3-normal}. It follows from the triangle inequality that the the integrand of \eqref{eq-expression-I3-normal} satisfies the following inequality
\begin{equation}
\label{bound-integrand-I3}
\norm{\sqrt{\alpha_t}\mathbf{x}_{0} + \sqrt{v_t} \mathbf{z}}\norm{\mathbf{T}_{\boldsymbol{\theta}}(g^{-1}(\mathbf{x}_0))}  \leq \sqrt{\alpha_t}\norm{\mathbf{x}_{0}}\norm{\mathbf{T}_{\boldsymbol{\theta}}(g^{-1}(\mathbf{x}_0))}  + \sqrt{v_t} \norm{\mathbf{z}}\norm{\mathbf{T}_{\boldsymbol{\theta}}(g^{-1}(\mathbf{x}_0))} 
\end{equation}
Given \eqref{bound-integrand-I3}, it is sufficient to show that
\begin{equation}
\label{eq-finite-integral-bound-I3}
\sqrt{\alpha_t}\int p_{\mathbf{x}_{0}}(\mathbf{x}_{0})\norm{\mathbf{x}_{0}}\norm{\mathbf{T}_{\boldsymbol{\theta}}(g^{-1}(\mathbf{x}_0))} \mathrm{d}\mathbf{x}_{0}  + \sqrt{v_t} \int p_{\mathbf{x}_{0}}(\mathbf{x}_{0})\int p_{\mathbf{z}}(\mathbf{z})\norm{\mathbf{z}}\norm{\mathbf{T}_{\boldsymbol{\theta}}(g^{-1}(\mathbf{x}_0))} \mathrm{d}\mathbf{z}\mathrm{d}\mathbf{x}_{0} <\infty 
\end{equation}
for all $t\in[\epsilon,1]$ to conclude, by the comparison test for Lebesgue integrability, that $\mathcal{I}_{3}$ is finite for all $t\in[\epsilon,1]$.

Notice that it follows from  \eqref{eq-lemma-finite-conditions-x0} that
\begin{equation}
\label{eq-first-bound-I3}
\sqrt{\alpha_t}\int p_{\mathbf{x}_{0}}(\mathbf{x}_{0})\norm{\mathbf{x}_{0}}\norm{\mathbf{T}_{\boldsymbol{\theta}}(g^{-1}(\mathbf{x}_0))} \mathrm{d}\mathbf{x}_{0} <\infty
\end{equation}
for all $t\in[\epsilon,1]$. Furthermore, recall that for a standard multivariate normal random variable $\mathbf{z}\sim\mathcal{N}_{d}(0,\mathbf{I}_{d})$ with $\mathbf{z} = (z_{1},\ldots,z_{d})$, the expectation of its $\ell^{2}$-norm satisfies the following bound:
\begin{equation}
\label{eq-bounded-norm-gaussian}
\mathbb{E}_{p_{\mathbf{z}}}[\norm{\mathbf{z}}] \leq \sqrt{\sum_{i=1}^{d}\text{Var}(z_{i}^{2})} =\sqrt{d}.
\end{equation}
where the first inequality follows from Jensen's inequality.
It follows from \eqref{eq-bounded-norm-gaussian}  that
\begin{equation}
\label{eq-bound-z-T-integral}
\begin{aligned}
\int p_{\mathbf{x}_0}(\mathbf{x}_0)\int  p_{\mathbf{z}}(\mathbf{z}) \norm{\sqrt{v_t} \mathbf{z}}\norm{\mathbf{T}_{\boldsymbol{\theta}}(g^{-1}(\mathbf{x}_0))}  \mathrm{d}\mathbf{z} \mathrm{d}\mathbf{x}_0  &= \sqrt{v_{t}}\left(\mathbb{E}_{p_{\mathbf{z}}}[\norm{\mathbf{z}}] \right) \left(\int p_{\mathbf{x}_0}(\mathbf{x}_0)\norm{\mathbf{T}_{\boldsymbol{\theta}}(g^{-1}(\mathbf{x}_0))}\mathrm{d}\mathbf{x}_0 \right) 
\\ &\leq \sqrt{v_{t}d}  \left(\int p_{\mathbf{x}_0}(\mathbf{x}_0)\norm{\mathbf{T}_{\boldsymbol{\theta}}(g^{-1}(\mathbf{x}_0))}\mathrm{d}\mathbf{x}_0 \right) 
\\ &<\infty
\end{aligned}
\end{equation}
for all $t\in[\epsilon,1]$, where the last inequality follows from \eqref{eq-lemma-finite-conditions-x0}. Then, it follows from \eqref{eq-first-bound-I3} and \eqref{eq-bound-z-T-integral} that the inequality in \eqref{eq-finite-integral-bound-I3} is satisfied for all $t\in[\epsilon,1]$. This shows that $\mathcal{I}_{3}$ is finite for all $t\in[\epsilon,1]$.

We have verified that $\mathcal{I}_{1}$, $\mathcal{I}_{2}$ and $\mathcal{I}_{3}$ are finite for any $t\in[\epsilon,1]$. This concludes the proof.
\end{proof}

\begin{proof}[Proof of Theorem~\ref{prop-new-objective}]
In order to show the desired result, it is sufficient to show that 
\begin{equation}
\label{eq-sufficient-expectation}
\mathbb{E}_{t\sim U(\epsilon, 1), \mathbf{x}_t \sim p_{\mathbf{x}_t}}\left[\boldsymbol{\zeta}(\mathbf{x}_t)^{\top}\mathbb{E}_{p_{\tilde{\mathbf{x}}_{0}\vert\mathbf{x}_t}}[\mathbf{T}_{\boldsymbol{\theta}}(g^{-1}(\tilde{\mathbf{x}}_{0}))]\right]  = \mathbb{E}_{t\sim U(\epsilon, 1),\mathbf{x}_0\sim  p_{\mathbf{x}_0}, \mathbf{x}_t \sim p_{\mathbf{x}_t|\mathbf{x}_0}, }\left[\boldsymbol{\zeta}(\mathbf{x}_t)^{\top}\mathbf{T}_{\boldsymbol{\theta}}(g^{-1}({\mathbf{x}}_{0}))\right].
\end{equation}
To begin, we express the LHS of~\eqref{eq-sufficient-expectation} explicitly as an integral:
\begin{equation}
\label{eq-explicit-integral-proof}
\begin{aligned}
&\mathbb{E}_{t\sim U(\epsilon, 1), \mathbf{x}_t \sim p_{\mathbf{x}_t}}\left[\boldsymbol{\zeta}(\mathbf{x}_t)^{\top}\mathbb{E}_{p_{\tilde{\mathbf{x}}_{0}\vert\mathbf{x}_t}}[\mathbf{T}_{\boldsymbol{\theta}}(g^{-1}(\tilde{\mathbf{x}}_{0}))]\right] \\
&=\int_{t}    p_{ U(\epsilon, 1)}(t)\int_{\mathbf{x}_t}  p_{\mathbf{x}_t}(\mathbf{x}_{t})  \int_{\tilde{\mathbf{x}}_0} p_{\tilde{\mathbf{x}}_0\vert \mathbf{x}_t}(\tilde{\mathbf{x}}_{0})\boldsymbol{\zeta}(\mathbf{x}_t)^{\top} \mathbf{T}_{\boldsymbol{\theta}}(g^{-1}(\tilde{\mathbf{x}}_{0})) \mathrm{d}\tilde{\mathbf{x}}_{0}\mathrm{d}\mathbf{x}_t\mathrm{d}t\\
&=\int_{t} p_{ U(\epsilon, 1)}(t)\int_{\mathbf{x}_t} \int_{\tilde{\mathbf{x}}_0} p_{\mathbf{x}_t}(\mathbf{x}_{t})       p_{\tilde{\mathbf{x}}_0\vert \mathbf{x}_t}(\tilde{\mathbf{x}}_{0}) \boldsymbol{\zeta}(\mathbf{x}_t)^{\top}\mathbf{T}_{\boldsymbol{\theta}}(g^{-1}(\tilde{\mathbf{x}}_{0})) \mathrm{d}\tilde{\mathbf{x}}_{0}\mathrm{d}\mathbf{x}_t\mathrm{d}t \\
&=\int_{t} p_{ U(\epsilon, 1)}(t)\int_{\mathbf{x}_t} \int_{\tilde{\mathbf{x}}_0} p_{\mathbf{x}_t|\tilde{\mathbf{x}}_{0}}(\mathbf{x}_{t}) p_{\tilde{\mathbf{x}}_{0}}(\tilde{\mathbf{x}}_{0}) \boldsymbol{\zeta}(\mathbf{x}_t)^{\top}\mathbf{T}_{\boldsymbol{\theta}}(g^{-1}(\tilde{\mathbf{x}}_{0})) \mathrm{d}\tilde{\mathbf{x}}_{0}\mathrm{d}\mathbf{x}_t\mathrm{d}t 
\end{aligned}
\end{equation}
where in the last equality we have used Bayes' theorem as follows
\begin{equation*}
 p_{\mathbf{x}_t}(\mathbf{x}_{t})  p_{\tilde{\mathbf{x}}_0\vert \mathbf{x}_t}(\tilde{\mathbf{x}}_{0})  = p_{\mathbf{x}_t|\tilde{\mathbf{x}}_{0}}(\mathbf{x}_{t}) p_{\tilde{\mathbf{x}}_{0}}(\tilde{\mathbf{x}}_{0}).
\end{equation*}
Lemma~\ref{lemma-bound-fubini} verifies, under the assumptions in the statement of the theorem, a sufficient condition to apply Fubini's theorem as follows
\begin{multline}
\label{eq-fubini-swap-integral}
\int_{\mathbf{x}_t} \int_{\tilde{\mathbf{x}}_0} p_{\mathbf{x}_t|\tilde{\mathbf{x}}_{0}}(\mathbf{x}_{t}) p_{\tilde{\mathbf{x}}_{0}}(\tilde{\mathbf{x}}_{0}) \boldsymbol{\zeta}(\mathbf{x}_t)^{\top}\mathbf{T}_{\boldsymbol{\theta}}(g^{-1}(\tilde{\mathbf{x}}_{0})) \mathrm{d}\tilde{\mathbf{x}}_{0}\mathrm{d}\mathbf{x}_t = \\ \int_{\tilde{\mathbf{x}}_0}   \int_{\mathbf{x}_t}  p_{\tilde{\mathbf{x}}_{0}}(\tilde{\mathbf{x}}_{0}) p_{\mathbf{x}_t|\tilde{\mathbf{x}}_{0}}(\mathbf{x}_{t}) \boldsymbol{\zeta}(\mathbf{x}_t)^{\top}\mathbf{T}_{\boldsymbol{\theta}}(g^{-1}(\tilde{\mathbf{x}}_{0}))\mathrm{d}\mathbf{x}_t\mathrm{d}\tilde{\mathbf{x}}_{0}
\end{multline}
for all $t\in[\epsilon,1]$.
We now plug~\eqref{eq-fubini-swap-integral} into~\eqref{eq-explicit-integral-proof} to obtain
\begin{equation*}
\begin{aligned}
&\mathbb{E}_{t\sim U(\epsilon, 1), \mathbf{x}_t \sim p_{\mathbf{x}_t}}\left[\boldsymbol{\zeta}(\mathbf{x}_t)^{\top}\mathbb{E}_{p_{\tilde{\mathbf{x}}_{0}\vert\mathbf{x}_t}}[\mathbf{T}_{\boldsymbol{\theta}}(g^{-1}(\tilde{\mathbf{x}}_{0}))]\right] \\ 
 &= \int_{t} p_{ U(\epsilon, 1)}(t)\int_{\tilde{\mathbf{x}}_0}  p_{\tilde{\mathbf{x}}_{0}}(\tilde{\mathbf{x}}_{0})  \int_{\mathbf{x}_t}   p_{\mathbf{x}_t|\tilde{\mathbf{x}}_{0}}(\mathbf{x}_{t}) \boldsymbol{\zeta}(\mathbf{x}_t)^{\top}\mathbf{T}_{\boldsymbol{\theta}}(g^{-1}(\tilde{\mathbf{x}}_{0}))\mathrm{d}\mathbf{x}_t\mathrm{d}\tilde{\mathbf{x}}_{0}\mathrm{d}t \\ 
&=\mathbb{E}_{t\sim U(\epsilon, 1),\mathbf{x}_0\sim  p_{\mathbf{x}_0}, \mathbf{x}_t \sim p_{\mathbf{x}_t|\mathbf{x}_0} }\left[\boldsymbol{\zeta}(\mathbf{x}_t)^{\top}\mathbf{T}_{\boldsymbol{\theta}}(g^{-1}({\mathbf{x}}_{0}))\right]
\end{aligned}
\end{equation*}
which shows the desired result.
\end{proof}

\end{document}